\documentclass[10pt]{article}
\input{macro_arxiv.tex}
\usepackage{mathrsfs}
\allowdisplaybreaks

\begin{document}

\title{On the Robustness of Distribution Support under Diffusion Guidance
}
\author{%
	Ruijia Cao\thanks{Center for Applied Mathematics,  Cornell University; \,\, Email: \href{mailto:rc948@cornell.edu}{\texttt{rc948@cornell.edu}}. } \and
Yuchen Wu\thanks{School of Operations Research and Information Engineering, Cornell University;  \,\, Email: {\href{mailto:yuchen.wu@cornell.edu}{\texttt{yuchen.wu@cornell.edu}}}. 
} 
   \and
Nisha Chandramoorthy\thanks{Department of Statistics, The University of Chicago; \,\, Email: \href{mailto:nishac@uchicago.edu}{\texttt{nishac@uchicago.edu}}. }
}

\date{\today}
\maketitle

\begin{abstract}

Diffusion guidance is a powerful technique that enables controllable and high-fidelity sample generation with diffusion models.
At a high level, it modifies the score function by incorporating a guidance term that steers the generative process toward a desired condition. 
Despite its empirical success, the theoretical properties of diffusion guidance remain largely unexplored, and it is not well understood why it consistently produces high-quality samples.

In this work, we explain the effectiveness of diffusion guidance by establishing a \emph{robustness of support} property. 
Specifically, we show that, given exact access to the score functions, guided diffusion processes almost always generate samples that remain close to the target support. 
This property is particularly desirable, as samples that lie off the support are often structurally implausible and may adversely affect downstream tasks.
Our analysis covers both Denoising Diffusion Implicit Models (DDIM) and Denoising Diffusion Probabilistic Models (DDPM), and applies to a wide range of discretization schemes induced by exponential integrators.
Our results provide a rigorous foundation for understanding why diffusion guidance produces physically meaningful and structurally plausible samples.

\end{abstract}

\tableofcontents

\section{Introduction}

\emph{Score-based diffusion models}, originally inspired by thermodynamic principles \citep{sohldickstein2015deep}, have emerged as a powerful class of generative models, achieving state-of-the-art performance across a wide range of applications \citep{song2019generative,ho2020denoising,song2020denoising,rombach2022highresolution, saharia2022photorealistic,watson2023novo,abramson2024accurate}.
At a high level, diffusion models generate samples by progressively denoising a sequence of intermediate distributions $\{P_i\}_{i = 1}^N$, thereby transforming a simple, easy-to-sample distribution $P_1$ (e.g., Gaussian noise) into a complex target distribution $P_N$.
This generative procedure is driven by a sequence of \emph{score functions} $\{\nabla \log P_i\}_{i = 1}^{N - 1}$, which are typically learned by training deep neural networks through a score matching procedure \citep{hyvarinen2005estimation}.
Given the estimated score functions, diffusion models generate samples through an $N$-step iterative process, which is illustrated by the following schema:
\begin{align*}
    P_1 \xrightarrow{\mathrm{sample}} Y_1 \xrightarrow{\nabla \log P_1 } Y_{2} \xrightarrow{\nabla \log P_{2} }  \cdots \xrightarrow{\nabla \log P_{N - 1} } Y_{N} \,\overset{d}{\approx}\, P_N. 
\end{align*}

In practice, many applications require controllable data generation to ensure that the generated content aligns with specific objectives or user-defined preferences \citep{dhariwal2021diffusion,ho2022classifierfree,zhang2023adding,cao2024controllable}. 
A common approach to achieve this goal with diffusion models is \emph{diffusion guidance} \citep{dhariwal2021diffusion,ho2022classifierfree}, 
a technique that steers the generation process toward samples with the desired attributes. 
Specifically, to generate samples associated with a target label $\eta_0$, diffusion guidance modifies the unconditional score $\nabla \log P_i(\cdot)$ by incorporating an additional guidance term $\gamma \nabla \log P_i(\eta_0 \mid \cdot)$ for all $i \in [N - 1]$, 
where $\gamma \geq 0$ controls the strength of the guidance.
In practice, this guidance term must be learned from data, and the two main approaches for doing so are \emph{classifier guidance} \citep{dhariwal2021diffusion} and \emph{classifier-free guidance} \citep{ho2022classifierfree}, depending on whether the term is approximated via classifier training or joint score estimation. 
Notably, diffusion guidance not only enables controllable generation, but also often improves the quality of the generated samples.
An informal illustration of the sample generation process under guided diffusion models is provided below, and a more detailed description is deferred to Section \ref{sec:guidance}. From a theoretical perspective, both guided and unguided diffusion models can be understood as discrete approximations of continuous-time processes \citep{song2021scorebased}.
\begin{align*}
    P_1 \xrightarrow{\mathrm{sample}} Y_1 \xrightarrow{\nabla \log P_1  + \textcolor{BrickRed}{\mathrm{guidance}}} Y_{2} \xrightarrow{\nabla \log P_{2}  + \textcolor{BrickRed}{\mathrm{guidance}}}  \cdots \xrightarrow{\nabla \log P_{N - 1}  + \textcolor{BrickRed}{\mathrm{guidance}}} Y_{N} \,\overset{d}{\approx} \,  \textcolor{BrickRed}{\mathrm{?}}\,. 
\end{align*}

Despite its empirical success, the theoretical properties of diffusion guidance remain largely unexplored. 
In particular, the distribution produced by guided diffusion sampling has not yet been fully characterized. 
To shed light on the effectiveness of diffusion guidance, a line of work analyzes simplified settings in which the target distribution admits tractable structure, for example, when the conditional distributions are Gaussian or supported on one-dimensional intervals \citep{wu2024theoretical,chidambaram2024does,ventura2026emergencedistortionshighdimensionalguided}. 
For more general distributions, \cite{bradley2024classifier} demonstrate that guidance does not, in general, sample from a powered distribution, and can be interpreted as a predictor-corrector method.
Recent work establishes an average-case result demonstrating that diffusion guidance reduces the expected reciprocal of the conditional label probability, 
and thus increases classification confidence in an average-case sense \citep{li2025provable,jiao2025towards}.
From a different perspective, \cite{sahu2026provably} show that under mild smoothness conditions, classifiers trained by minimizing the empirical cross-entropy loss lead to an effective guidance mechanism. 
For conditional sample generation with hard constraints, \cite{guo2026conditional} propose a guidance framework based on Doob's $h$-transform.

In this work, we adopt a different perspective, explaining the success of diffusion guidance through the \emph{robustness of support} property \citep{chandramoorthyand}.
In particular, we ask the following question:

\begin{center}
\begin{minipage}{0.8\linewidth}
\raggedright\itshape
\vspace{0.5em}
Does diffusion guidance generate samples that are close to the target support? 
\vspace{0.5em}
\end{minipage}
\end{center}
The ability to generate samples near the target support is a highly desirable property for generative models, since off-support samples are often structurally implausible and can adversely impact downstream tasks.
For example, a recent paper on video sampling empirically demonstrates that, in certain settings, generative models may fail to achieve the robustness of support property, leading to physically distorted videos \citep{jang2026self}.
In the context of diffusion guidance, this means that although guidance may not generate samples exactly from the conditional distribution associated with the label of interest (and indeed should not when $\gamma \neq 1$), 
we still expect it to produce physically meaningful and plausible samples that remain near the target support.
Support robustness of diffusion guidance has been established in \cite{chidambaram2024does} for one-dimensional mixture distributions. 
However, it is unclear whether their analysis extends to high-dimensional settings.

\subsection{Main contributions}

In this paper, we study the behavior of diffusion guidance when the target distribution is a mixture of general distributions supported on compact convex sets.
Our findings are summarized below:

\begin{itemize}
    \item \textbf{Support robustness with perfect score estimation: }
    We show that guided diffusion models with exact access to the score functions exhibit strong support robustness, in the sense that the generated sample paths almost surely converge to the correct support.
    Our theory applies broadly to all guidance strength $\gamma \geq 1$ and to a wide family of discretization schemes induced by an exponential integrator (including continuous-time processes with infinitesimal discretization),  and covers both the deterministic diffusion sampler Denoising Diffusion Implicit Models (DDIM) \citep{song2020denoising} 
and the stochastic diffusion sampler Denoising Diffusion Probabilistic Models (DDPM) \citep{ho2020denoising}.
Our theory implies that, in the absence of score estimation error and with sufficiently fine discretization, diffusion guidance with arbitrary guidance strength produces physically plausible samples that remain close to the support of the target component.
\item \textbf{Numerical evidence: }
We provide numerical evidence supporting our theoretical findings.  
Based on experimental outcomes, we further conjecture that, under mild regularity assumptions, convergence to the target support may continue to hold even in the absence of support convexity.
We leave a rigorous investigation of this point for future work.
\end{itemize}

In summary, our results provide insight into the qualitative behavior of the sample paths produced by diffusion guidance, revealing that the resulting dynamics inherently adapt to the intrinsic geometry of the target distribution.
In practice, however, diffusion guidance can still produce physically implausible samples, particularly when the guidance strength $\gamma$ is large. 
Our results suggest that this behavior does not originate from the diffusion guidance framework itself, but rather from either errors in score estimation or an insufficiently fine discretization scheme. 
Note that diffusion models are generally expected to fail under inaccurate score estimation or coarse discretization.

\subsection{Related works}
\paragraph{Theoretical explanations for the success of diffusion guidance.}  
A recent line of work seeks to theoretically explain the empirical success of diffusion guidance.
Among others,  \cite{wu2024theoretical} analyzes diffusion guidance when the target distribution is a Gaussian mixture model (GMM). 
Specifically, they show that applying guidance to both DDIM and DDPM increases classification confidence, defined as the posterior probability of the label of interest. 
Moreover, they prove that for guided DDIM, diffusion guidance decreases differential entropy, theoretically validating the empirical observation that guidance reduces sample diversity. 
\cite{chidambaram2024does} analyzes guided DDIM applied to GMM or mixture distributions supported on disjoint intervals.  
Their results show that strong guidance steers the generated samples toward the boundary of the target component.
More recently, several works have studied diffusion guidance for general distributions without an explicit parametric form \citep{bradley2024classifier,li2025provable,jiao2025towards,guo2026conditional}. Since we have discussed these works earlier, we omit further discussion here.

\paragraph{Reward-guided diffusion models. }
Another application of controllable generation with  diffusion models is reward-guided diffusion \citep{fan2023dpok,yuan2023reward,black2024training,wallace2024diffusion,gao2024reward,jiao2025towards}, where the objective is to generate samples that achieve higher reward values. These methods modify the reverse diffusion dynamics by incorporating gradients of a reward function, thereby steering the generation process toward outputs that receive higher reward.
Here, the reward function might represent various objectives, including alignment with human preferences or adherence to given constraints.
We refer the reader to \cite{uehara2025inference} for a review of recent developments in this area.

\paragraph{Other theoretical advances for diffusion models.}
Various results have been established to explain the overall success of diffusion models,
covering topics including convergence rates \citep{lee2023convergence,chen2023sampling,benton2024nearly,li2024accelerating,wu2024theoretical,gupta2025faster,li2025unified,li2025odt,huang2025convergence,huang2025fast,li2025faster}, score estimation \citep{oko2023diffusion,wibisono2024optimal,zhang2024minimax,dou2024optimal}, and applications in statistical or optimization settings \citep{montanari2023posterior,chung2023diffusion,li2024diffusion,chen2025diffusion}.
We refer interested readers to \cite{chen2024overview} for a summary of recent theoretical developments in the past few years. 

\subsection{Organization}
In \cref{sec:background}, we provide a brief overview of diffusion models and diffusion guidance. 
In Section \ref{sec:main_results_DDIM}, we present and prove the main results for DDIM, while Section \ref{sec:main_results_DDPM} presents the corresponding results for DDPM. 
We provide intuition for our theoretical results in the main text, while deferring the complete proofs to the appendix. 
Numerical experiments are reported in Section \ref{sec:experiments}.
We conclude the paper and discuss several directions for future work in Section \ref{sec:discussion}.

\subsection{Notation}

For $d \in \mathbb{N}_+$, we define the set $[d] = \{1, 2, \cdots, d\}$. 
For a probability measure $\mu$ over $\R^d$, we denote by $\supp(\mu)$ the support of $\mu$.
Unless otherwise specified, the symbol $\norm{\cdot}$ will denote the usual Euclidean $2$-norm.
For a set $A \subseteq \R^d$ and a vector $x \in \R^d$, we define $\dist(x, A) = \inf_{y \in A} \norm{x - y}$. 
The (set-valued) projection operator is defined by $\Proj_A(x) = \argmin_{y \in A} \norm{x - y}$.
If $\Proj_A(x)$ contains a single element, then we denote that element by $\Proj_A(x)$. 
For $d \ge 1$, we denote by $\mathbb{S}^{d - 1}$ the unit sphere in $\R^d$ centered at the origin. 
For a set $K \subseteq \R^d$, we denote by $\mathsf{cl}(K)$ the closure of $K$, $\mathsf{int}(K)$ the interior of $K$, $\vol(K)$ the volume of $K$, 
and define the boundary of $K$ as $\partial K = \mathsf{cl}(K) \backslash \mathsf{int}(K)$. 
For $x \in \R^d$ and $\delta > 0$, we define $\cB(x, \delta) = \{y \in \R^d: \|y - x\| < \delta\}$.

\section{Background}\label{sec:background}

\subsection{Score-based diffusion models}

This section provides a brief overview of diffusion models. 
At a high level, a diffusion model consists of a \emph{forward process} and a \emph{reverse process}.
In the forward process, the target distribution is progressively transformed into an easy-to-sample noise distribution (typically a standard Gaussian). 
The reverse process, in turn, maps this noise back to the target distribution, enabling the generation of new samples.
Note that only the reverse process is used during the sample generation stage. 

Diffusion models are often analyzed through their continuous-time limits, and are typically implemented as numerical approximations to these limits \citep{song2021scorebased}.
For better understanding, we first introduce the continuous-time formulation of diffusion models, and then describe the samplers implemented in practice, which can be viewed as time-discretized approximations of these continuous-time processes. 

\paragraph{Diffusion models through a continuous-time perspective.}
A common choice for the forward process is based on the Ornstein-Uhlenbeck (OU) process.
Specifically, let $p_0$ denote the target distribution to be sampled. We then consider the following stochastic differential equation (SDE) defined on $0 \leq t \leq T$: 
\begin{align}
\label{eq:forward_sde}
\dd y_t = -  y_t \D t + \sqrt{2}  \D W_t, \qquad
y_0 \sim p_{0},
\end{align}
where $(W_t)_{t \geq 0}$ denotes a standard $\R^d$-valued Wiener process that is independent of $y_0$. For $0 \leq t \leq T$, we denote by $p_t$ the distribution of $y_t$. 
Standard computation implies that 
\begin{align}
\label{eq:distribution-evolution}
\begin{split}
    & p_t \overset{d}{=} \lambda_t\, y_0 + \sigma_t\, w, \qquad (y_0, w) \sim p_0 \otimes \pi_d, \\
    & \lambda_t = e^{-t}, \qquad \sigma_t = \sqrt{ 1 - \lambda_t^2}, 
\end{split}
\end{align}
where $\pi_d$ denotes the distribution of a standard Gaussian random vector in $\R^d$. 
From \cref{eq:distribution-evolution}, we see that process \eqref{eq:forward_sde} begins with the target distribution, and progressively injects Gaussian noise until the resulting distribution approaches the standard Gaussian distribution.

For the reverse process, two commonly used approaches are DDIM and DDPM.
Their continuous-time formulations corresponding to the forward process \eqref{eq:forward_sde} are given below:
\begin{itemize}
    \item \emph{DDIM sampler. } 
    The continuous-time process underlying the deterministic DDIM sampler is formulated as an ordinary differential equation (ODE) on $0 \leq t < T$: 
    \begin{align} \label{eq:ODE-reverse}
    \dd x_t = \left(x_t + \nabla \log p_{T - t}(x_t) \right) \D t, \qquad x_0 \sim p_T.
    \end{align}
    The marginal distributions of process \eqref{eq:ODE-reverse} match that of the forward diffusion process \eqref{eq:forward_sde}, in the sense that $x_t \overset{d}{=} y_{T - t}$ for all $0 \leq t < T$. 
    \item \emph{DDPM sampler. } The continuous-time process underlying the stochastic DDPM sampler is on the other hand formulated as an SDE defined on $0 \leq t < T$:
    \begin{align} \label{eq:SDE-reverse}
    \dd x_t = \left(x_t + 2 \nabla \log p_{T - t}(x_t) \right) \D t + \sqrt{2} \D B_t, \qquad  x_0 \sim p_T,
\end{align}
where $(B_t)_{t \geq 0}$ is a standard $d$-dimensional Wiener process that is independent of $x_0$. 
Process \eqref{eq:SDE-reverse} also matches the marginal distributions of process \eqref{eq:forward_sde} in a time-reversed sense: $x_t \overset{d}{=} y_{T - t}$ for all $0 \leq t < T$ \citep{anderson1982reverse}. 
\end{itemize}

Processes \eqref{eq:ODE-reverse} and \eqref{eq:SDE-reverse} are both driven by a family of \emph{score functions} $\{\nabla \log p_{T - t}\}_{0 \leq t < T}$, which in practice are unknown and are typically learned from training data drawn from $p_0$ through a \emph{score matching} procedure \citep{hyvarinen2005estimation}.  
In addition, the correct initial distribution $p_T$ depends on $p_0$ hence is also unknown. 
On the positive side, from \cref{eq:distribution-evolution}, we know that $p_T$ can be well approximated by a standard Gaussian distribution for only moderately large $T$ (See Proposition 4 in \cite{benton2024nearly} for a rigorous justification of this point). 
To sample from $p_0$, rather than tracking the processes in \eqref{eq:ODE-reverse} and \eqref{eq:SDE-reverse} all the way to $T$, practitioners typically stop early at $T - \delta$ for a small positive $\delta$ to avoid numerical instability.

\paragraph{Diffusion models via time-discretization.}
Apart from score estimation and approximate initialization,
practical implementation of DDIM (DDPM) requires choosing a discretization scheme for the associated reverse-time ODE (SDE).
There are multiple ways to discretize the reverse dynamics, and the choice of discretization scheme can significantly affect the resulting diffusion model.
For instance, a popular approach is the exponential integrator, which has been empirically shown to produce high-fidelity samples \citep{zhang2023fast}. 
\cite{karras2022elucidating} further improve the  sampling efficiency of diffusion models by adaptively selecting discretization step sizes.
Another line of work investigates the use of high-order numerical solvers to accelerate sample generation with diffusion models \citep{lu2022dpm,lu2025dpm,wu2024stochastic,huang2025convergence,huang2025fast,li2025faster}.
For a summary of other common discretization strategies, we refer the reader to \cite[Section 2]{liang2025low}.

\subsection{Controllable generation with diffusion guidance}
\label{sec:guidance}
In many practical scenarios, instead of sampling unconditionally from $p_0$, one may wish to generate samples associated with a user-specified label $\eta_0$.
Diffusion guidance achieves this by augmenting the unconditional score $\nabla \log p_{T - t}(\cdot)$ with an additional term $\gamma \nabla \log p_{T - t}(\eta_0 \mid \cdot)$, where the hyperparameter $\gamma \geq 0$ controls the guidance strength. 
Here, with a slight abuse of notation, we denote by $p_t$ the joint distribution of $(y_t, \eta_0)$, assuming that the Wiener process  $(W_t)_{t \geq 0}$ in \cref{eq:forward_sde} is independent of $(y_0, \eta_0)$. 
From a continuous-time perspective, the guided ODE associated with the DDIM sampler takes the following form:
\begin{align}\label{eqn:ddim_guidance}
\begin{split}
	\dd x_t = & \left(x_t +  \nabla \log p_{{T - t}}(x_t) + \gamma \nabla \log p_{{T - t}}(\eta_0 \mid x_t) \right) \D t \\
    = & \left(x_t +  (1 - \gamma)\nabla \log p_{{T - t}}(x_t) + \gamma \nabla \log p_{{T - t}}( x_t,  \eta_0) \right) \D t,
\end{split}
\end{align}
where $p_{T - t}(\eta_0 \mid x_t) \in [0, 1]$ denotes the probability of label $\eta_0$ given $x_t$ under $p_{T - t}$. 
Throughout this work, unless otherwise stated, all gradients are taken with respect to the $x$-space. 
Similarly, the guided SDE associated with the DDPM sampler admits the following expression: 
\begin{align}\label{eqn:ddpm_guidance}
\begin{split}
	\dd x_t = & \left(x_t +  2\nabla \log p_{{T - t}}(x_t) + 2\gamma \nabla \log p_{{T - t}}(\eta_0 \mid x_t) \right) \D t + \sqrt{2} \D B_t \\
    = & \left(x_t +  2(1 - \gamma)\nabla \log p_{{T - t}}(x_t) + 2\gamma \nabla \log p_{{T - t}}( x_t,  \eta_0) \right) \D t + \sqrt{2} \dd B_t. 
\end{split}
\end{align}
In practice, the additional guidance term must also be learned from data. Depending on the training approach, it can be classified as either classifier guidance (where $p_{T - t}(\eta_0 \mid \cdot)$ is learned) or classifier-free guidance (where a joint score function over both the label and the sample space is learned). 
We refer interested readers to \cite{dhariwal2021diffusion} and \cite{ho2022classifierfree} for further details.

Observe that when $\gamma = 0$, processes \eqref{eqn:ddim_guidance} and \eqref{eqn:ddpm_guidance} reduce to the unconditional sampling processes \eqref{eq:ODE-reverse} and \eqref{eq:SDE-reverse}, respectively. On the other hand, when $\gamma = 1$, they coincide with the reverse processes associated with the conditional distribution $p_0(\cdot \mid \eta_0)$.  
Notably, practitioners often choose $\gamma > 1$ to improve fidelity of the generated samples. With this choice, the goal is not to sample exactly from $p_0(\cdot \mid \eta_0)$, but to generate samples that are more strongly aligned with the target label $\eta_0$.

\section{Results for DDIM}
\label{sec:main_results_DDIM}

In this section, we present our main results for DDIM. 
Throughout this paper, we denote by $\cI$ a finite set of labels and,
without loss of generality, assume that diffusion guidance is always directed toward samples associated with $\eta_0 \in \cI$.

\subsection{Summary of main results}

We begin by stating the assumptions on the target distribution $p_0$ required to establish our main theorem.

\begin{ass}
\label{assumption:target}
    We assume that $p_0 = \sum_{\eta \in \cI} w_{\eta} p_{\eta}$ is a mixture distribution in $\R^d$, where $(w_{\eta})_{\eta \in \mathcal{I}}$ are non-negative mixture weights that satisfy $\sum_{\eta \in \cI} w_{\eta} = 1$, and $p_{\eta}$ is a probability density function supported on a convex and compact set $K_{\eta} \subseteq \R^d$. We further impose the following assumptions: 
    \begin{enumerate}
        \item The set $K_{\eta}$ contains at least one interior point for all $\eta \in \cI$. In addition, $K_{\eta} \cap K_{\eta'} = \emptyset$ for all $\eta, \eta' \in \cI$ with $\eta \neq \eta'$. 
        \item The weight associated with the target component is positive, i.e., $w_{\eta_0} > 0$. 
        \item The density function associated with the target component is bounded from below on its support, i.e., $\inf_{x \in K_{\eta_0}} p_{\eta_0}(x) > 0$. 
    \end{enumerate}
\end{ass}

For simplicity, we first present our findings for the continuous-time guided ODE. 

\begin{thm}
\label{thm:continuous-DDIM}
Let $(x_t)_{0 \le t < T}$ be any sample path that represents a solution to the guided ODE \eqref{eqn:ddim_guidance}, then under Assumption \ref{assumption:target}, for any starting location $x_0 \in \R^d$ and guidance strength $\gamma \geq 1$, we have 
\begin{align*}
    \lim_{t \to T^-} \dist (x_t, K_{\eta_0}) = 0. 
\end{align*}
\end{thm}

\begin{proof}[Proof of Theorem \ref{thm:continuous-DDIM}]
    A proof of Theorem \ref{thm:continuous-DDIM} is given in Appendix \ref{proof:continuous-DDIM}. 
\end{proof}

We explain the intuition behind the theorem in the next section.
The crucial ingredient of our proof is a differential inequality controlling the rate of decay of the distance function along the trajectory, which roughly states that
\begin{equation}\label{eqn:dist_deriv_informal}
    \frac{\dd}{\dd t} \dist(x_t, K_{\eta_0}) \leq - \frac{1}{4\sigma_{T - t}^2} \dist(x_t, K_{\eta_0}) 
\end{equation}
for all $t$ close enough to $T$ as long as $\dist(x_t, K_{\eta_0})$ is bounded away from zero by an arbitrarily small positive constant.
A rigorous statement of this inequality, along with its proof, is provided in Appendix \ref{proof:continuous-DDIM}.

\begin{rem}
    Theorem \ref{thm:continuous-DDIM} establishes a path-wise convergence result that holds for any initial location, providing stronger guarantees than those based on average-case analysis. 
    Our result implies that, for a sufficiently small early-stopping parameter $\delta$, the samples generated by the guided ODE \eqref{eqn:ddim_guidance} remain within the vicinity of the target support.
    Furthermore, inequality \eqref{eqn:dist_deriv_informal} suggests that the distance contracts at rate $\mathcal{O}(- 1/ \sigma^2_{T - t})$ whenever the trajectory is not too close to $K_{\eta_0}$.
\end{rem}

We then extend the continuous-time analysis given in Theorem \ref{thm:continuous-DDIM} to cover the discrete samplers used in practice.
While we focus on the sampling scheme based on the exponential integrator \citep{zhang2023fast}, we expect that our results generalize to other discretization methods as well.

To describe the sampling scheme based on the exponential integrator, we consider a discretization with time steps $0 = t_0 < t_1 < \ldots < t_N \leq T$. 
On each interval $[t_k, t_{k + 1})$ for $k = 0, 1, \ldots, N - 1$, we perform the following update
\begin{align}
\label{eq:discretized-ODE}
    \dd \hat x_t = \big( \hat x_t + \nabla \log p_{T - t_k}(\hat x_{t_k}) + \gamma \nabla \log p_{T - t_k}(\eta_0 \mid \hat x_{t_k}) \big) \D t. 
\end{align}
Update equation \eqref{eq:discretized-ODE} is derived by replacing the non-linear drift term $\nabla \log p_{T - t}(\cdot) + \gamma \nabla \log p_{T - t}(\eta_0 \mid \cdot)$ with its value at $t = t_k$ on the interval $[t_k, t_{k + 1})$.
Note that this update equation can be expressed in the following closed form:
\begin{align}
\label{eq:exp-integrator}
\begin{split}
    & \hat x_{t_{k + 1}} = e^{\Delta_k} \hat x_{t_k} + (e^{\Delta_k} - 1) \cdot \big( \nabla \log p_{T - t_k}(\hat x_{t_k}) + \gamma \nabla \log p_{T - t_k}(\eta_0 \mid \hat x_{t_k}) \big), \\
    &\Delta_k = t_{k + 1} - t_k, \qquad  k = 0, 1, \cdots, N - 1. 
\end{split}
\end{align}
We next impose an assumption on the step sizes. 
Note that our assumption is standard and aligns with the one adopted in \cite{benton2024nearly}.
\begin{ass}
\label{assumption:discretization}
    We assume that there exists $\kappa \in (0, 1)$, such that $\Delta_k \leq \kappa \min\{1, T - t_{k + 1}\}$. 
\end{ass}

\begin{thm}
\label{thm:guided_discretized_ode_qualitative}
    Under Assumptions \ref{assumption:target} and \ref{assumption:discretization}, 
    there exists $\kappa_0 > 0$ that depends only on $(p_0, \gamma)$, such that for all $0 < \kappa \leq \kappa_0$, the following statement is true: 
    for any discretization scheme $\{t_k\}_{k = 0}^{\infty}$ that satisfies $t_0 = 0$, $T > t_{k + 1} > t_k \geq 0$, and $t_k \to T$ as $k \to \infty$, it holds that
    \begin{align*}
        \lim_{k \to \infty} \dist(\hat x_{t_k}, K_{\eta_0}) = 0 
    \end{align*}
    for any initialization $\hat x_{t_0} \in \R^d$ and guidance strength $\gamma \geq 1$.  
    Here, the sequence $\{\hat x_{t_k}\}_{k=0}^\infty$ is produced by \cref{eq:exp-integrator}. 
\end{thm}

\begin{proof}[Proof of Theorem \ref{thm:guided_discretized_ode_qualitative}]
    We prove Theorem \ref{thm:guided_discretized_ode_qualitative} in Appendix \ref{proof:thm:guided_discretized_ode_qualitative}. 
\end{proof}

\begin{rem}
\label{rem:score}
	So far, we have considered a simplified setting in which the score functions and guidance terms are known exactly. 
	Evidently, the support robustness property continues to hold under uniform bounds on the score and guidance term errors. 
	However, we do not yet know how to extend the result to the more realistic setting in which the errors are measured in $L_2$. 
	In this case, it is natural to assume that the errors are under the base measures $(p_t)_{0 < t \leq T}$ from the forward process, as both the score and the guidance terms are trained using data drawn from these distributions.
	However, in the guided diffusion model, the intermediate states in the reverse process do not approximately follow the distributions of the forward process. Instead, their distributions depend on the choice of the guidance parameter $\gamma$. 
	This distribution mismatch represents the primary obstacle to extending our results to account for score and guidance estimation errors.
\end{rem}

\begin{rem}
    Theorem \ref{thm:guided_discretized_ode_qualitative} requires the discretization scheme to be sufficiently fine, as governed by $\kappa_0$.
    Moreover, a closer inspection of our proof shows that the interval of $\kappa_0$ for which the theorem holds may shrink as $\gamma$ increases, 
    indicating that a large guidance strength combined with insufficiently fine discretization can lead to physically distorted samples.
\end{rem}

\begin{rem}
    Although we analyze a commonly used discretization scheme here, we do expect similar results to hold for other discretization schemes.
    One reason is that the one-step displacement under different discretization schemes is $\kappa$ under Assumption \ref{assumption:discretization}. 
    The drift evaluation at $\hat{z}_{t_k}$ is therefore a close approximation of the continuous-time drift, which blows up as $\mathcal{O}(1/\sigma^2_{T-t})$. 
    Thus, the rate of the contraction of the distance function is unaffected.
\end{rem}

Our main result \cref{thm:continuous-DDIM} assumes that the support domains of the mixture components are convex. 
It is worth noting that this convexity requirement may be somewhat restrictive, as practical mixture distributions do not necessarily exhibit convex support.
We make two remarks regarding this point.
First, our numerical experiments indicate that our theoretical findings continue to hold for mixture distributions with non-convex support.
Second, although many high-dimensional datasets in practice may not have convex support, their latent representations in a lower-dimensional subspace often form simpler, and sometimes even linearly separable clusters. 
For instance, this effect is evident in the Principal Component Analysis (PCA) plots of human genomics data \cite{elhaik2022principal}.
On this note, we remark that Latent Diffusion Models introduced in \cite{rombach2022highresolution} 
have been shown to yield highly competitive performance on various image-related tasks compared to diffusion models trained directly on pixel space while being more computationally efficient.

\subsection{Intuition: reformulating DDIM with guidance}
\label{sec:reformulate_DDIM}

In this section, we present a change-of-variables reformulation of the process \eqref{eqn:ddim_guidance} that simplifies the analysis and offers intuition for the proof of Theorem \ref{thm:continuous-DDIM}. 
Observe that for $0 < t \leq T$, we can express the density functions $p_t(\cdot)$ and $p_t(\eta_0 \mid \cdot)$ using Gaussian convolutions:  
\begin{align*}
    & p_t(x) = \int \sum_{\eta \in \cI} w_{\eta} p_{\eta}(x_0) \phi_t(x - \lambda_t x_0) \dd x_0, \\
    & p_t(\eta_0 \mid x) = \frac{w_{\eta_0} \int p_{\eta_0}(x_0) \phi_t(x - \lambda_t x_0) \dd x_0 }{\sum_{\eta \in \cI} w_{\eta} \int p_{\eta}(x_0) \phi_t(x - \lambda_t x_0) \dd x_0}, \\
    & \phi_t(z) = \frac{1}{(\sqrt{2\pi} \sigma_t)^d} \exp\Big( -\frac{1}{2\sigma_t^2} \|z\|^2 \Big). 
\end{align*}
Taking the gradients of $p_t(x)$ and $p_t(\eta_0 \mid x)$ with respect to $x$, we obtain by Stein's formula \citep{stein1981estimation} the following expressions for the score function and the guidance term: 
\begin{align*}
    & \nabla \log p_t(x) = \frac{\sum_{\eta \in \cI} w_{\eta} \int p_{\eta}(x_0) \phi_t(x - \lambda_t x_0) (-\sigma_t^{-2} x + \lambda_t \sigma_t^{-2} x_0) \dd x_0}{\sum_{\eta \in \cI} w_{\eta} \int p_{\eta}(x_0) \phi_t(x - \lambda_t x_0) \dd x_0}, \\
    & \nabla \log p_t(\eta_0 \mid x) 
    = \frac{\int p_{\eta_0}(x_0) \phi_t(x - \lambda_t x_0) (-\sigma_t^{-2} x + \lambda_t \sigma_t^{-2} x_0) \dd x_0 }{\int p_{\eta_0}(x_0) \phi_t(x - \lambda_t x_0) \dd x_0 } \\
    & \qquad\qquad\qquad\qquad - \frac{\sum_{\eta \in \cI} w_{\eta} \int p_{\eta}(x_0) \phi_t(x - \lambda_t x_0) (-\sigma_t^{-2} x + \lambda_t \sigma_t^{-2} x_0) \dd x_0}{\sum_{\eta \in \cI} w_{\eta} \int p_{\eta}(x_0) \phi_t(x - \lambda_t x_0) \dd x_0}. 
\end{align*}
For notational simplicity, for $\eta \in \cI$ and $0 < t \leq T$, we define 
\begin{align}
\label{eq:bar-eta-bar-m}
\begin{split}
    & \bar{\zeta}_{\eta, t}(x) = \frac{w_{\eta} \int p_{\eta}(x_0) \phi_t(x - \lambda_t x_0) \dd x_0 }{\sum_{\eta \in \cI} w_{\eta} \int p_{\eta}(x_0) \phi_t(x - \lambda_t x_0) \dd x_0} = p_t(Y = \eta \mid X_t = x), \\
    & \bar{m}_{\eta, t}(x) = \frac{\int x_0 p_{\eta}(x_0) \phi_t(x - \lambda_t x_0)  \dd x_0}{\int p_{\eta}(x_0) \phi_t(x - \lambda_t x_0) \dd x_0} = \mathbb{E}\big[ X_0 \mid X_t = x, \, Y = \eta \big], 
\end{split}
\end{align}
where $p_t(Y = \eta) = w_{\eta}$, $X_0 \mid Y = \eta \sim p_{\eta}$, $X_t = \lambda_t X_0 + \sigma_t Z$, and $Z \sim \cN(0, \id_d)$ is independent of $(X_0, Y)$. 
Observe that for all $x \in \R^d$, it holds that $\sum_{\eta \in \cI} \bar \zeta_{\eta, t}(x) = 1$, $\nabla \log p_t(x) = -\sigma_t^{-2} x + \lambda_t \sigma_t^{-2} \sum_{\eta \in \cI} \bar \zeta_{\eta, t}(x) \bar m_{\eta, t}(x)$, and
\begin{equation*}
\nabla \log p_t(\eta_0 \mid x) = \lambda_t \sigma_t^{-2} \bar m_{\eta_0, t}(x) - \lambda_t \sigma_t^{-2} \sum_{\eta \in \cI} \bar\zeta_{\eta, t}(x)  \bar m_{\eta, t}(x).
\end{equation*}
Plugging the above expressions into the guided ODE \eqref{eqn:ddim_guidance}, we conclude that 
\begin{align}
\begin{split}
  &\dd x_t =  \Big\{ \frac{\sigma_{T - t}^2 - 1}{\sigma_{T - t}^2} x_t + \frac{\gamma \lambda_{T - t}}{\sigma_{T - t}^2} \bar{m}_{\eta_0, T  - t} (x_t) - \frac{ (\gamma - 1) \lambda_{T - t}}{\sigma_{T - t}^2} \sum_{\eta \in \cI}  \bar\zeta_{\eta, T - t}(x_t) \bar{m}_{\eta, T - t}(x_t) \Big\} \dd t \label{eq:ODE-reformulation-step1}\\
    &=  \Big\{  x_t + 
    \frac{\gamma \lambda_{T - t}}{\sigma_{T - t}^2} \Big(  \bar{m}_{\eta_0, T  - t} (x_t) -  \frac{x_t}{\lambda_{T - t}} \Big)  - \frac{ (\gamma - 1) \lambda_{T - t}}{\sigma_{T - t}^2} \sum_{\eta \in \cI} \bar{\zeta}_{\eta, T - t}(x_t) \Big( \bar{m}_{\eta, T - t}(x_t) - \frac{x_t}{\lambda_{T - t}}  \Big) \Big\} \dd t. 
\end{split}
\end{align}
Define $z_t = \lambda_{T - t}^{-1} x_t$ for $0 \leq t < T$. By \cref{eq:ODE-reformulation-step1}, we see that the resulting process $(z_t)_{0 \leq t < T}$ is the solution to the following ODE: 
\begin{equation}\label{eqn:ode_time_scaled}
    \dd z_t = \Big\{ \frac{\gamma}{\sigma_{T - t}^2} \left( m_{\eta_0, T - t}(z_t) - z_t \right) - \frac{\gamma - 1}{\sigma_{T - t}^2} \sum_{\eta \in \cI} \zeta_{\eta, T - t}(z_t) \left( m_{\eta, T - t}(z_t) - z_t \right) \Big\} \dd t, 
\end{equation}
where for $\eta \in \cI$ and $0 \leq t < T$, 
\begin{align}
\label{eq:zeta-and-m}
\begin{split}
    & \zeta_{\eta, t}(z) = \frac{w_{\eta} \int p_{\eta}(x_0) \phi_t(\lambda_t(x_0 - z)) \dd x_0}{\sum_{\eta \in \cI} w_{\eta} \int p_{\eta}(x_0) \phi_t(\lambda_t(x_0 - z)) \dd x_0} = p_t(Y = \eta \mid X_t = \lambda_t z), \\
    & m_{\eta, t} (z) = \frac{\int x_0 p_{\eta}(x_0) \phi_t(\lambda_t(x_0 - z)) \dd x_0}{\int p_{\eta}(x_0) \phi_t(\lambda_t(x_0 - z)) \dd x_0} = \E\big[ X_0 \mid X_t = \lambda_t z,\, Y = \eta \big]. 
\end{split}
\end{align}
By definition, $\sum_{\eta \in \cI} \zeta_{\eta, t}(z) = 1$ for all $0 \leq t < T$ and $z \in \R^d$. 
Note that the vector-valued function $m_{\eta, t}(z)$ can be interpreted as the mean of a tilted version of $p_{\eta}$, with the tilt determined by the input vector $z \in \R^d$.
When $K_{\eta}$ is convex, we know that $m_{\eta, t}(z) \in K_{\eta}$ for all $z \in \R^d$.

For notational convenience, for $\eta \in \cI$ and $0 < t \leq T$, we define
\begin{align}
\label{eq:F-def}
    F_{\eta, t}(z) = m_{\eta, t}(z) - z. 
\end{align}
With the above definition, we can re-write ODE \eqref{eqn:ode_time_scaled} as 
\begin{align}\label{eqn:ode_general}
    \dd z_t = 
    \Big\{ 
        \frac{\gamma}{\sigma_{T - t}^2} F_{\eta_0, T - t}(z_t) 
        - \frac{\gamma - 1}{\sigma_{T - t}^2} 
            \sum_{\eta \in \cI} \zeta_{\eta, T - t}(z_t) F_{\eta, T - t}(z_t) 
    \Big\} \dd t.
\end{align}
Intuitively, one may interpret the reformulation of \eqref{eqn:ddim_guidance} that we derive here as a force-balance equation: each term $F_{\eta, t}(z_t)$ acts as a ``force'' that pulls $z_t$ towards a point $m_{\eta, t}(z_t)$ in $K_{\eta}$.
When $\gamma = 1$, the only acting force is $F_{\eta_0, T - t}(z_t)$, and one should therefore expect that the trajectory of $z_t$ to move toward $K_{\eta_0}$ as $t$ increases. 
When $\gamma > 1$, the trajectory is subject to additional forces that repel it from other regions, 
and it is not {a priori} clear that the trajectory will still move toward $K_{\eta_0}$. 
Nevertheless, our analysis shows that the force directed toward $K_{\eta_0}$ dominates the other forces, thereby ensuring that the trajectory of $z_t$ ultimately converges to $K_{\eta_0}$.
See Figure \ref{fig:guidance} for an illustration of ODE \eqref{eqn:ode_general}.

\begin{figure}[ht]
	\centering
	\includegraphics[width=0.6\textwidth]{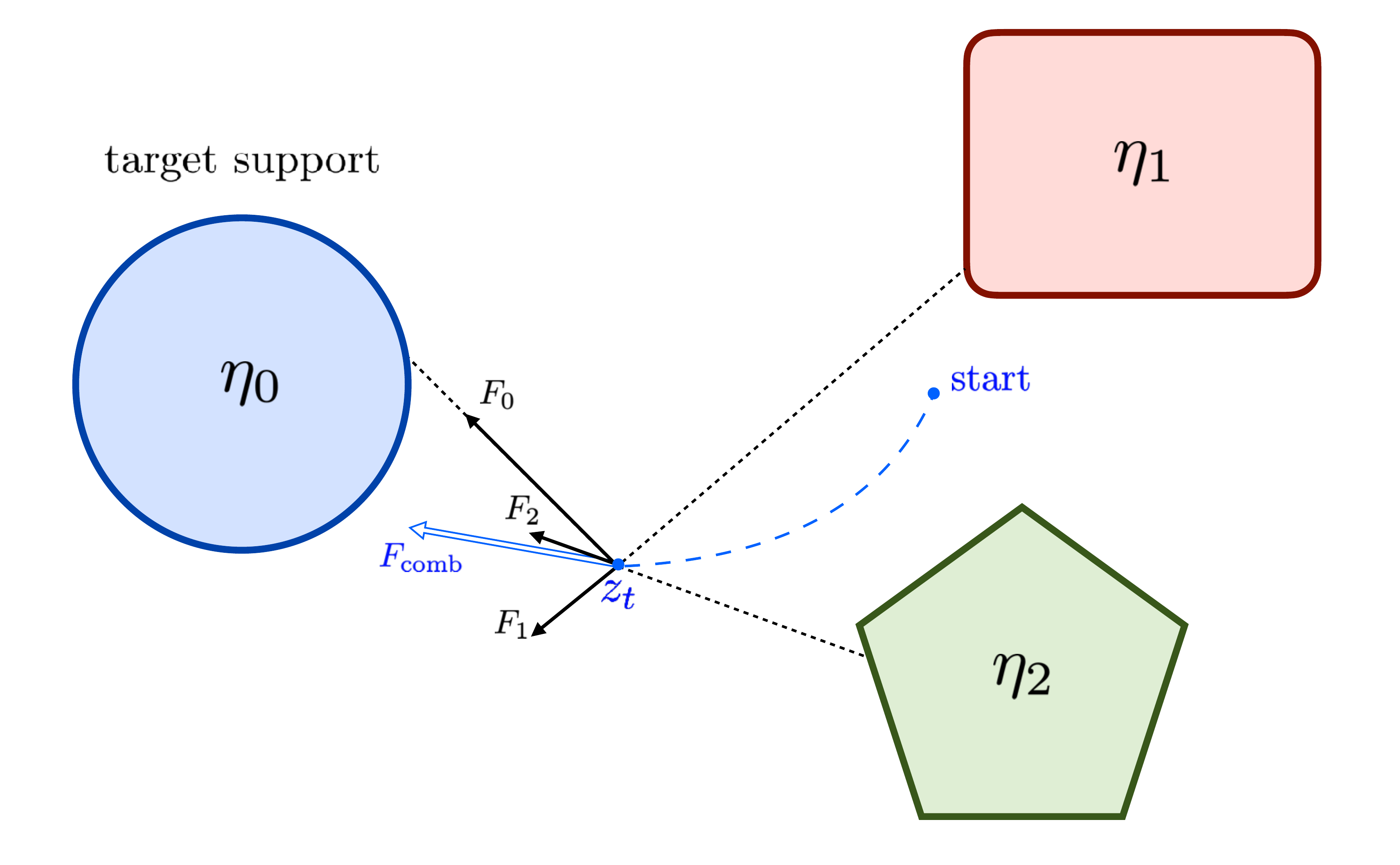}
	\caption{Illustration of ODE \eqref{eqn:ode_general} when $\cI = \{\eta_0, \eta_1, \eta_2\}$. Here, guidance is towards the blue circle with label $\eta_0$, $F_0 =  \frac{\gamma - (\gamma - 1) \zeta_{\eta_0, T - t}(z_t)}{\sigma_{T - t}^2} F_{\eta_0, T - t}(z_t) $, $F_1 = - \frac{\gamma - 1}{\sigma_{T - t}^2} \zeta_{\eta_1, T - t}(z_t) F_{\eta_1, T - t}(z_t) $ and $F_2 =  -\frac{\gamma - 1}{\sigma_{T - t}^2} \zeta_{\eta_2, T - t}(z_t) F_{\eta_2, T - t}(z_t) $. The combined force is given by $F_{\rm comb} = F_0 + F_1 + F_2$. }
	\label{fig:guidance}
\end{figure} 

\section{Results for DDPM}
\label{sec:main_results_DDPM}

In this section, we extend our deterministic results in Section \ref{sec:main_results_DDIM} to the stochastic setting, and establish the support robustness property for DDPM.

\subsection{Summary of main results} 

We present the main results for the DDPM sampler in this section, covering both the continuous-time formulation and a discrete-time scheme based on an exponential integrator.
We start with the continuous-time result. 

\begin{thm}\label{thm:guided_sde_qualitative}
  Let $(x_t)_{0 \le t < T}$ be any sample path that represents a solution to the guided SDE \eqref{eqn:ddpm_guidance}, then under Assumption \ref{assumption:target}, for any starting location $x_0 \in \R^d$ and guidance strength $\gamma \geq 1$, almost surely we have 
  \begin{equation*}
        \lim_{t \to T^-} \dist(x_t, K_{\eta_0}) = 0.
  \end{equation*}
\end{thm}

\begin{proof}[Proof of Theorem \ref{thm:guided_sde_qualitative}]
	We prove Theorem \ref{thm:guided_sde_qualitative} in Appendix \ref{proof:thm:guided_sde_qualitative}. 
\end{proof}

Similar to the update rule defined in \cref{eq:exp-integrator}, we apply an exponential integrator scheme to DDPM with guidance. 
Specifically, with time steps $0 = t_0 < t_1 < \ldots < t_N < T$, 
we perform the following updates on the  interval $[t_k, t_{k + 1})$ for $k = 0, 1, \ldots, N - 1$: 
\begin{align}
\label{eq:discretized-SDE}
    \dd \hat x_t 
    = \big( \hat x_t + 2 \nabla \log p_{T - t_k}(\hat x_{t_k}) + 2 \gamma \nabla \log p_{T - t_k}(\eta_0 \mid \hat x_{t_k}) \big) \D t + \sqrt{2} \dd B_t.   
\end{align}
We next show that, almost surely, the sample path defined by process \eqref{eq:discretized-SDE} enters the target support $K_{\eta_0}$.
Note that the update rule defined in \cref{eq:discretized-SDE} admits the following explicit form:
\begin{align}
\label{eq:15}
\begin{split}
    & \hat x_{t_{k + 1}} = e^{\Delta_k} \hat x_{t_k} + 2 (e^{\Delta_k} - 1) \big( \nabla \log p_{T - t_k}(\hat x_{t_k}) + \gamma \nabla \log p_{T - t_k}(\eta_0 \mid \hat x_{t_k}) \big) + \sqrt{ e^{2\Delta_k} - 1} \,Z_k, \\
    &\Delta_k = t_{k + 1} - t_k, \qquad  Z_k \sim_{i.i.d.} \cN(0, I_d),  \qquad  k = 0, 1, \cdots, N - 1.
\end{split}
\end{align}
The following theorem guarantees that process \eqref{eq:15} almost surely converges to the target support. 

\begin{thm}
\label{thm:discretized-ddpm}
	Under Assumptions \ref{assumption:target} and \ref{assumption:discretization}, 
    there exists $\kappa_0 > 0$ that depends only on $(p_0, \gamma, T)$, such that for all $0 < \kappa \leq \kappa_0$, the following statement is true: 
    for any discretization scheme $\{t_k\}_{k = 0}^{\infty}$ that satisfies $t_0 = 0$, $T > t_{k + 1} > t_k \geq 0$, and $t_k \to T$ as $k \to \infty$, for any initialization $\hat x_{t_0} \in \R^d$ and guidance strength $\gamma \geq 1$, almost surely we have
    \begin{align*}
        \lim_{k \to \infty} \dist(\hat x_{t_k}, K_{\eta_0}) = 0. 
    \end{align*}
    Here, the sequence $\{\hat x_{t_k}\}_{k=0}^\infty$ is produced by \cref{eq:15}. 
\end{thm} 
\begin{proof}
	We prove Theorem \ref{thm:discretized-ddpm} in Appendix \ref{proof:thm:discretized-ddpm}. 
\end{proof}

\subsection{Intuition: reformulating DDPM with guidance}

As in the derivations of Section \ref{sec:reformulate_DDIM}, SDE \eqref{eqn:ddpm_guidance} can be reformulated as
\begin{align*}
& \dd x_t - \sqrt{2} \dd B_t
=\\ & \Big\{  x_t + 
    \frac{2 \gamma \lambda_{T - t}}{\sigma_{T - t}^2} \Big(  \bar m_{\eta_0, T  - t} (x_t) - \frac{x_t}{\lambda_{T - t}}  \Big) 
  - \frac{ 2(\gamma - 1) \lambda_{T - t}}{\sigma_{T - t}^2} \sum_{\eta \in \cI} \bar\zeta_{\eta, T - t}( x_t) \Big( \bar m_{\eta, T - t}(x_t) - \frac{x_t}{\lambda_{T - t}}   \Big) \Big\} \dd t, 
\end{align*}
where we recall that $\bar m_{\eta, t}$ and $\bar \zeta_{\eta, t}$ are defined in \cref{eq:bar-eta-bar-m}.
Define $z_t = e^{T - t} x_t$ for $0 \leq t < T$. 
With such a change of variables, we have
\begin{align}
\begin{split}
    \dd z_t &= \Big\{\frac{2 \gamma}{\sigma_{T - t}^2} \big(  {m}_{\eta_0, T - t} ({ z_t}) - {z_t} \big) - \frac{ 2(\gamma - 1)}{\sigma_{T - t}^2} \sum_{\eta \in \cI} {\zeta}_{\eta, T - t} (z_t) 
  \big( {m}_{\eta, T - t}( z_t) - z_t \big)   \Big\}\dd t + \sqrt{2} e^{T - t} \dd B_t \label{eqn:guided_sde} \\
  & = \Big\{\frac{2 \gamma}{\sigma_{T - t}^2} F_{\eta_0, T - t}(z_t) - \frac{ 2(\gamma - 1)}{\sigma_{T - t}^2} \sum_{\eta \in \cI} {\zeta}_{\eta, T - t} (z_t) 
  F_{\eta, T - t}(z_t)   \Big\}\dd t + \sqrt{2} e^{T - t} \dd B_t. 
\end{split}
\end{align}
The drift term above is exactly twice that of the ODE in \eqref{eqn:ode_general}. 
Consequently, the trajectory defined in Eq.~\eqref{eqn:guided_sde} experiences a force that pulls it toward the target support, along with forces that push it away from the other supports. In addition, the process is influenced by a random force arising from the Brownian motion. 
When $t$ is sufficiently close to $T$, the random force is dominated by the deterministic force because of the large negative drift in the SDE.
To formalize this intuition, we compare process \eqref{eqn:guided_sde} to a Brownian motion with a large negative drift using the SDE comparison theorem \citep{ikeda2014stochastic}.
Again, as in the proof of Theorem \ref{thm:continuous-DDIM}, 
we argue that as $t \to T$ the deterministic force $2 (\gamma - (\gamma - 1) \zeta_{\eta_0, T - t}(z_t)) F_{\eta_0, T - t}(z_t) / \sigma_{T - t}^2$ dominates, thus ensuring that the trajectory moves towards the support of the target distribution. 
See Figure \ref{fig:guidance_DDPM} for an illustration of this process. 

\begin{figure}[ht!]
	\centering
	\includegraphics[width=0.55\textwidth]{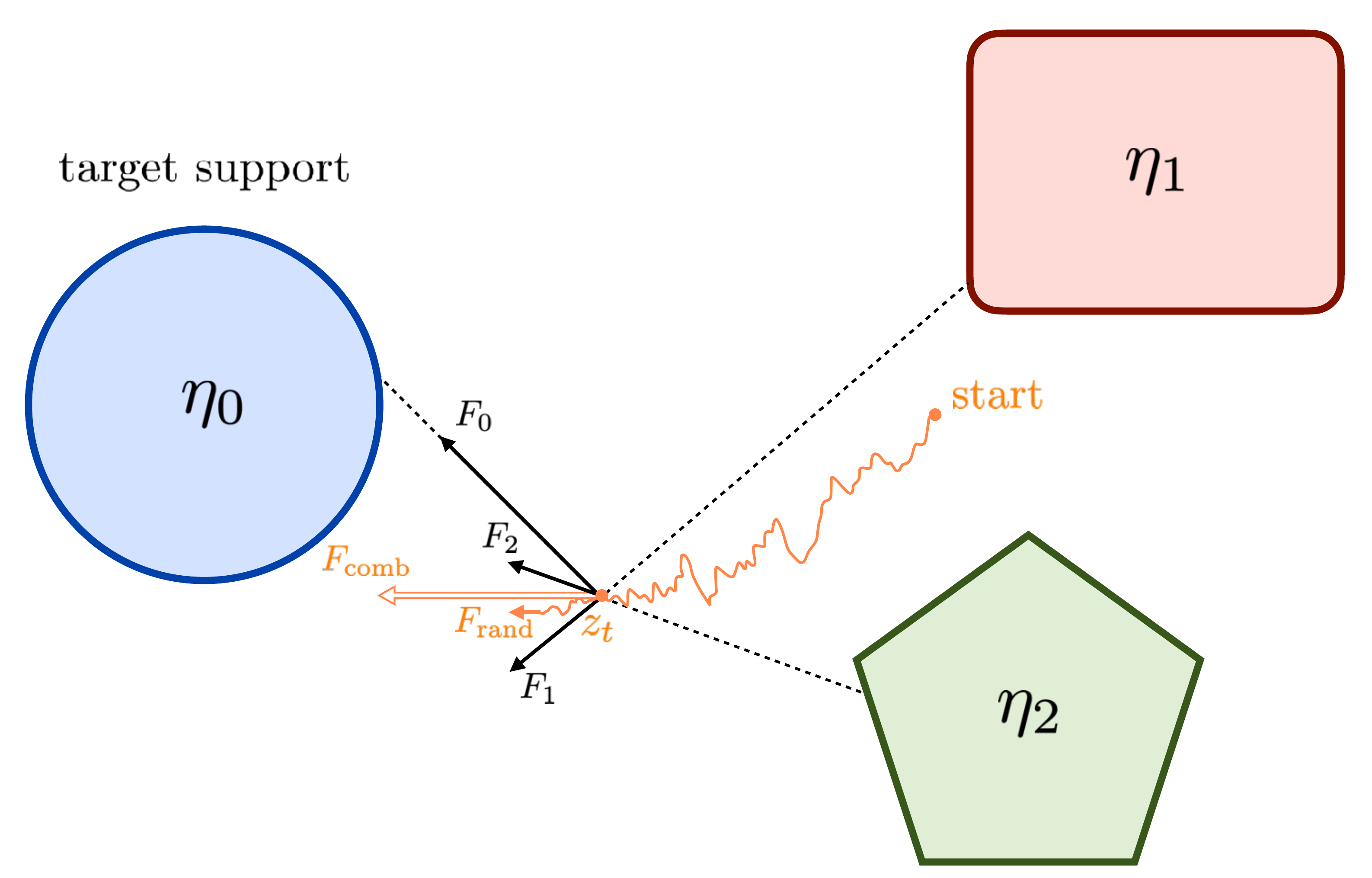}
	\caption{Illustration of SDE \eqref{eqn:guided_sde} for $\cI = \{\eta_0, \eta_1, \eta_2\}$. As before, the guidance is directed toward the blue circle labeled $\eta_0$. The terms $F_0$, $F_1$, and $F_2$ denote deterministic forces associated with the supports $K_{\eta_0}$, $K_{\eta_1}$, and $K_{\eta_2}$, respectively, and $F_{\rm rand}$ represents the stochastic force induced by the Brownian motion. The total force is given by $F_{\rm comb} = F_0 + F_1 + F_2 + F_{\rm rand}$. }
	\label{fig:guidance_DDPM}
\end{figure} 

\section{Numerical experiments}
\label{sec:experiments}

In this section, we illustrate our theoretical results through numerical experiments. 
In the first part, we illustrate sample trajectories of the guided diffusion models, and in the second part we plot the corresponding empirical distributions. 
Both experiments are consistent with our theoretical findings. 
The code used for these experiments is available at
\href{https://github.com/ruijiacao/diffusion_guidance_code}{\textcolor{magenta}{https://github.com/ruijiacao/diffusion\_guidance\_code}}.

\subsection{Sample trajectories}

\subsubsection{Example with convex support}

We consider a mixture distribution $p_0$ over $\R^2$, defined as follows: 
\begin{equation*}
    p_0 = \frac{1}{3} \mathsf{Unif} \Big( \ball \big((-1.5, 0.5),\, 0.7 \big) \Big)
    + \frac{1}{3} \mathsf{Unif} \Big( \ball \big((0.2, 3),\, 1.2 \big) \Big)
    + \frac{1}{3} \mathsf{Unif} \Big( \ball \big((2.5, -0.5),\, 1.5 \big) \Big).
\end{equation*}
Note that $p_0$ is a mixture of distributions, each supported on a convex set. 
We assume that the guidance is directed toward the second component.
For our experiment, we set $T = 5$, $\delta = 0.001$, $\kappa = 0.1$, and perform the diffusion updates using the exponential integrator scheme, which is defined in \cref{eq:exp-integrator} for guided DDIM and in \cref{eq:15} for guided DDPM.
We choose the time steps as follows:
\begin{align*}
	& t_0 = 0, \\
	& t_{k + 1} = t_k + \Delta_k, \\
	& \Delta_k = 
   \begin{cases}
   \kappa  & \text{if } t_k \leq T - 1, \\
   \frac{\kappa}{(1 + \kappa)^{k - M + 1}} & \text{if } t_k > T - 1. 
   \end{cases}
\end{align*}
where $M = \lfloor (T - 1) / \kappa \rfloor + 1$.
We examine different initial positions and guidance strengths. 
The resulting trajectories of the guided DDIM and guided DDPM are shown in Figures~\ref{fig:guided_ddim_trajectories} and \ref{fig:guided_ddpm_trajectories}, respectively. 
These figures illustrate that both methods produce sample trajectories that converge to the target support, confirming our theoretical findings. 
\begin{figure}[H]
    \centering
    \includegraphics[width=1.0\textwidth]{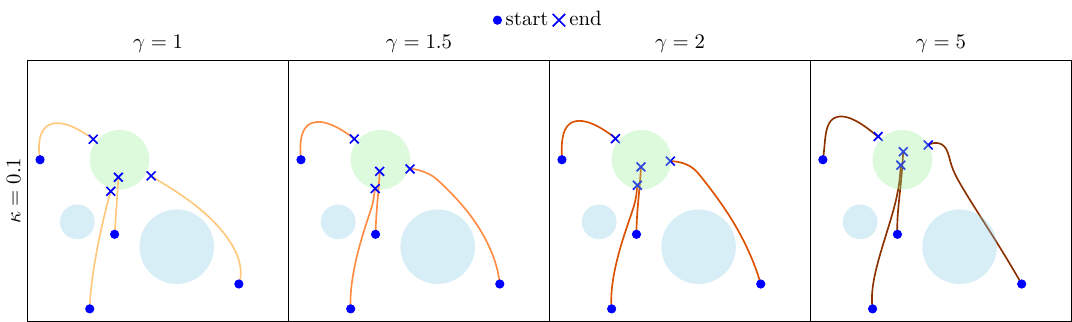}
    \caption{Trajectories of the guided DDIM under varying guidance strengths and initial positions. The target region is shown in green, the initial position is marked by a blue dot, and the endpoint by an ``X''.} 
    \label{fig:guided_ddim_trajectories}
\end{figure}

\begin{figure}[H]
    \centering
    \includegraphics[width=1.0\textwidth]{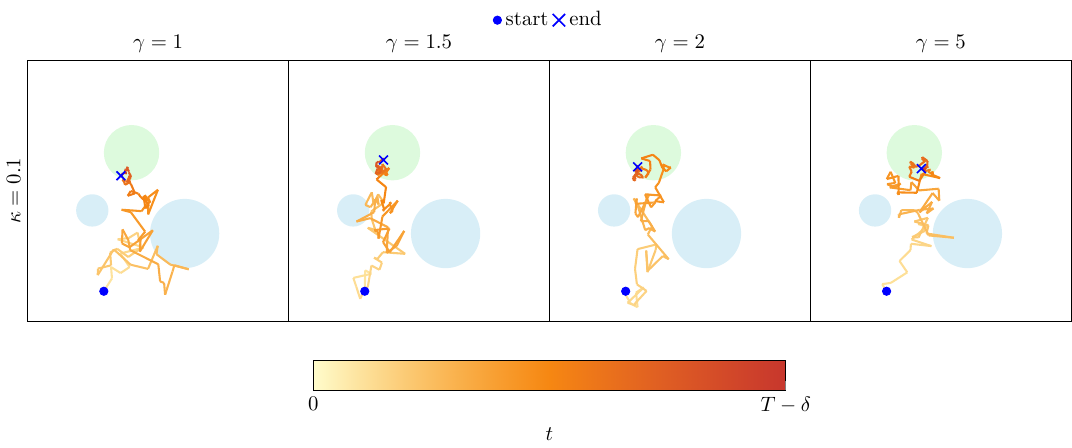}
    \caption{
    Trajectories of the guided DDPM under varying guidance strengths and initial positions. 
    Here, the target region is shown in green, the initial position is marked by a blue dot, and the endpoint by an ``X''.
    Lighter (darker) colors indicate smaller (larger) time indices within the interval $[0, T - \delta]$. 
    }
    \label{fig:guided_ddpm_trajectories}
\end{figure}
\subsubsection{Example with non-convex support}

While our theoretical results assume the convexity of the support of the mixture components,
we also provide an example demonstrating that support robustness can continue to hold even in non-convex settings. 
As before, we consider mixtures of uniform distributions, with the target support highlighted in green. 
Sample trajectories for the guided DDIM and guided DDPM are shown in Figures \ref{fig:guided_ddim_trajectories_nonconvex} and \ref{fig:guided_ddpm_trajectories_nonconvex}, respectively. 
These figures suggest that (under mild regularity assumptions) our theoretical results may hold even without the convexity assumption, which we defer to future work.

\begin{figure}[!ht]
    \centering
    \includegraphics[width=1.0\textwidth]{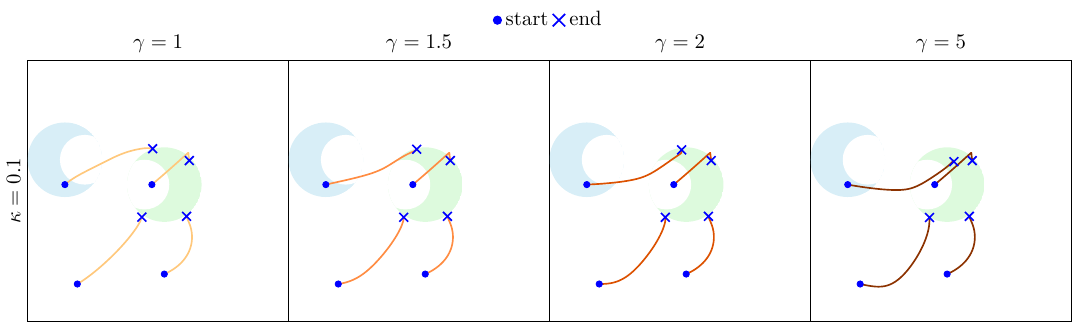}
    \caption{Trajectories of the guided DDIM for varying guidance strengths and initial positions, with the target distribution being a mixture of uniform distributions over non-convex supports. As before, the target region is highlighted in green, the initial position indicated by a blue dot, and the endpoint marked with an ``X.".}
    \label{fig:guided_ddim_trajectories_nonconvex}
\end{figure}

\begin{figure}[!ht]
    \centering
    \includegraphics[width=1.0\textwidth]{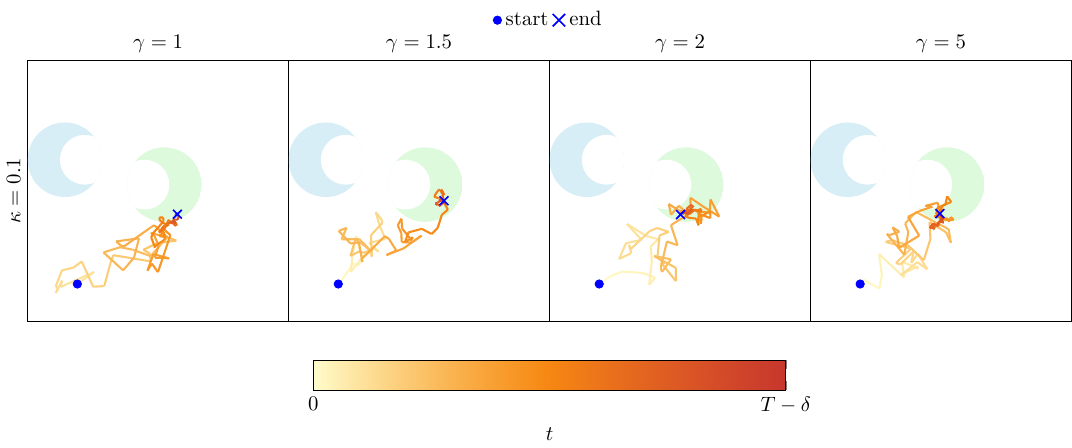}
    \caption{
    Trajectories of the guided DDPM for different guidance strengths and initial positions. As before, lighter (darker) colors correspond to earlier (later) times within the interval $[0, T - \delta]$. }
    \label{fig:guided_ddpm_trajectories_nonconvex}
\end{figure}

\subsection{Visualization of the guided density}

In the second part, we empirically visualize the distributions produced by the guided DDIM and DDPM samplers. 
As before, we set $T = 5$, $\delta = 0.001$, $\kappa = 0.1$, and sample $10^4$ points from $\normal(0, \id_2)$ as the initial positions.
The target distribution is an equally weighted mixture of five uniform distributions over the shapes shown in Figures \ref{fig:ddim_densities} and \ref{fig:ddpm_densities}. 
We then run both guided DDIM and DDPM using exponential integrators and plot the resulting samples in the same figures. 
These plots show that the distributions generated by diffusion guidance remain concentrated on the target support, further confirming our theoretical findings.

\begin{figure}[!ht]
    \centering
    \begin{subfigure}{\textwidth}
        \centering
        \includegraphics[width=\textwidth]{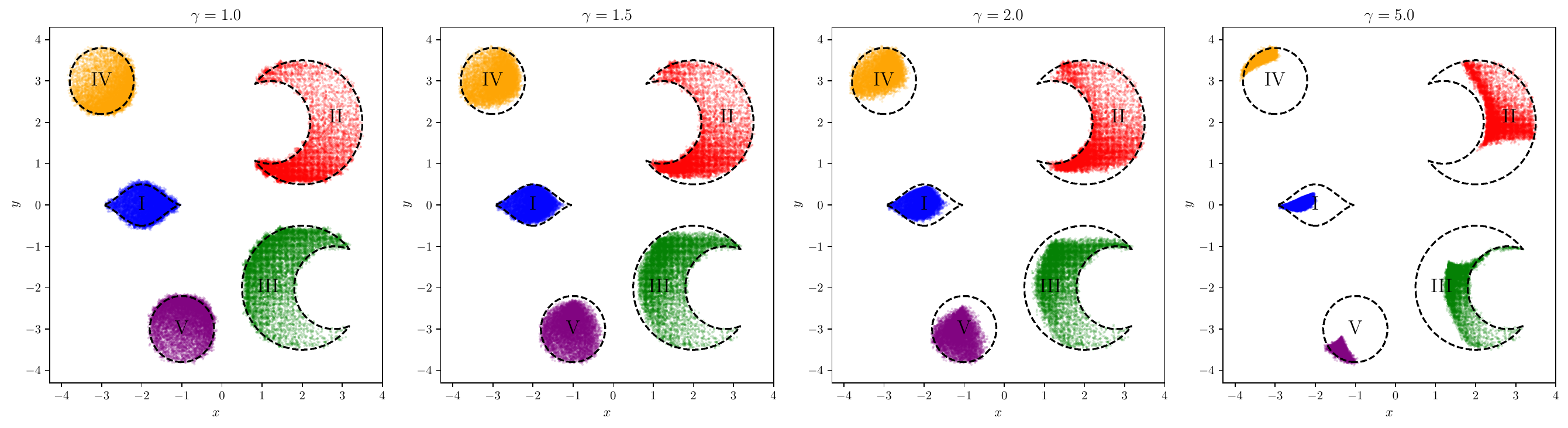}
        \caption{Empirical density obtained by running guided DDIM}
        \label{fig:ddim_densities}
    \end{subfigure}

    \vspace{0.8cm}

    \begin{subfigure}{\textwidth}
        \centering
        \includegraphics[width=\textwidth]{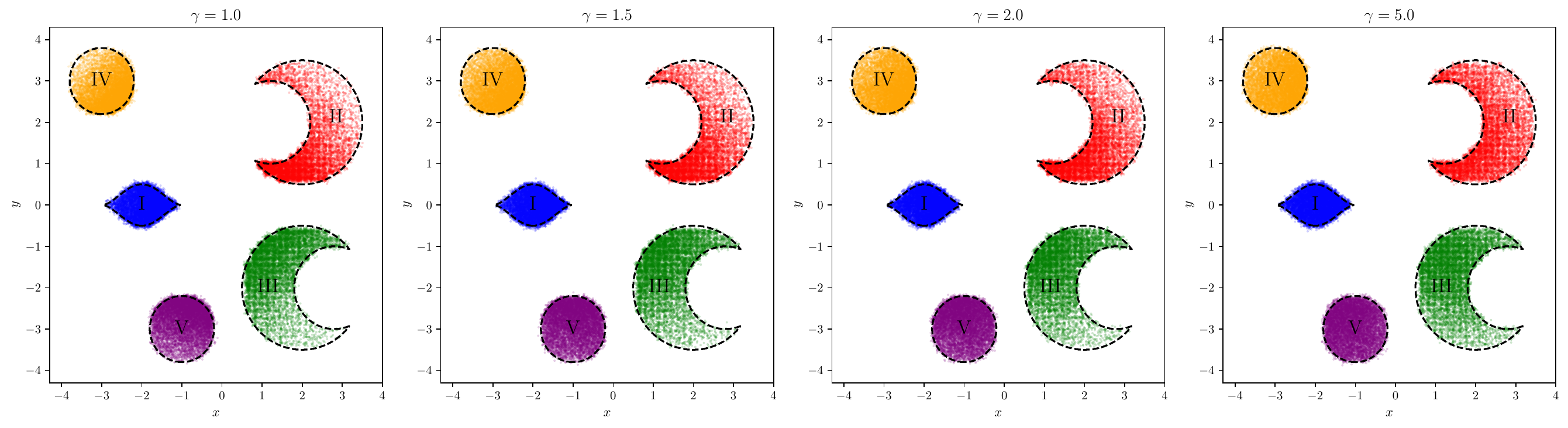}
        \caption{Empirical density obtained by running guided DDPM}
        \label{fig:ddpm_densities}
    \end{subfigure}
    \caption{In Figures \ref{fig:ddim_densities} and \ref{fig:ddpm_densities}, we illustrate the empirical densities produced by the guided DDIM and DDPM samplers under varying levels of guidance. 
    Each colored region represents the empirical distribution obtained by guiding the diffusion model toward its corresponding component. 
 }
\end{figure}

\section{Discussion}
\label{sec:discussion}

In this work, we show that diffusion models with guidance exhibit the support robustness property. 
Our theory applies to both DDIM and DDPM, 
and encompasses both continuous-time processes and discretizations induced by an exponential integrator. 

Several interesting directions remain for future exploration:
\begin{itemize}
	\item \textbf{Support robustness without convexity assumptions. }
	Our numerical experiments in Section \ref{sec:experiments} suggest that the support robustness property extends to mixture distributions with non-convex support. 
	We currently impose the convexity assumption due to technical limitations, 
    removing this restriction and providing a rigorous theoretical justification is an interesting direction for future work. 
	\item \textbf{Accommodating score estimation error. }
	Our current theoretical results assume exact access to the score functions. 
	As noted in Remark \ref{rem:score}, extending these results to estimated scores requires controlling the score estimation error under a shifted distribution, which requires additional technical developments.
	\item \textbf{Further analysis of the guided distribution. }
	This work suggests that diffusion guidance generates samples that remain on the target support. 
	An interesting open problem is to further characterize the distribution produced by diffusion guidance. 
    In particular, while our results indicate that diffusion guidance produces physically plausible samples, they do not fully explain its success with respect to sample-quality metrics.
    A deeper understanding would require a more refined analysis of the guided distribution. 
\end{itemize}

\bibliographystyle{plainnat}
\bibliography{refs}

\newpage

\appendix
\section{Technical lemmas}

In this section, we collect the technical lemmas required for our proof.

\begin{lem}
\label{lemma:volume-lower-bound}
    Let $K \subseteq \R^d$ be a convex and compact set that contains at least one interior point. 
    Then there exist constants  $\rho_K, r_K > 0$ that depends only on $K$, such that for all $x \in \partial K$ and $0 < r \leq r_K$, 
    \begin{align*}
   \vol \big( \cB(x, r) \cap K \big) \geq \rho_K r^d. 
    \end{align*}
\end{lem}

\begin{proof}[Proof of Lemma \ref{lemma:volume-lower-bound}]
    We prove Lemma \ref{lemma:volume-lower-bound} in Appendix \ref{proof:lemma:volume-lower-bound}. 
\end{proof}

\begin{lem}[Projection is to close to a tilted mean]
    \label{le:cm}
    Let $K \subseteq \R^d$ be a convex and compact set that contains at least one interior point. 
    For any $z \notin K$, define 
    \begin{align*}
        & m_{\sigma}(z) = \int_{K} x \exp \Big(- \frac{\norm{x - z}^2}{2 \sigma^2}\Big) \frac{\D x}{p_{\sigma}(z)},
    \end{align*}
    where $p_{\sigma}(z)$ is the normalizing coefficient defined as 
    \begin{equation*}
        p_{\sigma}(z) =  \int_{K} \exp \Big(- \frac{\norm{x - z}^2}{2 \sigma^2}\Big) \D x.
    \end{equation*}
    Then, there exists a constant $C_K > 0$ that depends only on $K$,
    such that for all $\sigma \in (0,1 / 2)$
    \begin{equation*}
        \dist \left(m_{\sigma}(z), \Proj_K(z) \right)
        \le C_K \, \sigma \log \sigma^{-1}.
    \end{equation*}
\end{lem}

\begin{proof}[Proof of Lemma \ref{le:cm}]
    We prove Lemma \ref{le:cm} in Appendix \ref{proof:le:cm}. 
\end{proof}

We then define the Hausdorff distance. 

\begin{defn}[Hausdorff distance]
\label{def:hausdorff}
    Let $X, Y \subseteq \R^d$ be two compact sets.
    Then, their Hausdorff distance, denoted by $\dH(X, Y)$, is given by 
    \begin{equation*}
        \dH(X, Y) 
        = \max \left\{ \sup_{x \in X} \dist(x, Y),\, \sup_{y \in Y} \dist(y, X) \right\}.
    \end{equation*}  
\end{defn}

\begin{rem}
    $\dH(\cdot, \cdot)$ defines a metric on the space of non-empty compact subsets in $\R^d$.
\end{rem}

\begin{rem}
   It is well-known that convex sets can be approximated arbitrarily well by convex sets with algebraic (and hence smooth) surfaces \citep{firey1974approximating}.
   More precisely, for any $\eps > 0$ and any compact convex set $K\subseteq \R^d$, there exists another compact convex set $\tilde{K}$ with smooth boundary, such that $\dH(\tilde{K}, K) \le \eps$.
\end{rem}

\begin{lem}\label{lem:dist_laplacian}
    Let $K \subseteq \R^n$ be a compact, convex set with $C^k$-boundary, where $k \ge 2$. 
    Then, the function $x \mapsto \dist(x, K)$ is $C^k$ on $\R^n \setminus K$. 
    Furthermore, the Laplacian of this function is given by 
    \begin{equation*}
        \Delta \dist(x, K) = \sum_{i=1}^{n-1} \frac{\kappa_i(\Proj_K(x))}{1 + \kappa_i(\Proj_K(x)) \dist(x, K)},
    \end{equation*}
    where $\left\{\kappa_i \right\}_{1\le i \le n - 1}$ denote the principal curvatures of $\partial K$.
    In particular, since $\kappa_i \ge 0$ and $\partial K$ is compact, 
    \begin{equation*}
        \Delta \dist(x, K) \le \mathcal{C}_K =  n  \sup_{y \in \partial K} \max_{1 \le i \le n-1} \kappa_i(y).
    \end{equation*} 
\end{lem}
\begin{proof}[Proof of Lemma \ref{lem:dist_laplacian}]

We prove Lemma \ref{lem:dist_laplacian} in Appendix \ref{proof:lem:dist_laplacian}. 
	
\end{proof}

\subsection{Proof of Lemma \ref{lemma:volume-lower-bound} }
\label{proof:lemma:volume-lower-bound}
The lemma is trivial when $d = 1$. In what follows, we consider $d \geq 2$. 

Let $v_K \in \R^d$ be an interior point of $K$, then there exists $\delta_K > 0$ such that $\ball(v_K, \delta_K) \subseteq K$. 
Define a mapping $F_K \colon \partial K \mapsto \mathbb{R}_{\geq 0}$ as follows: 
\begin{align*}
    F_K(x) = \inf_{0 < r < r_K}\frac{\vol \big( \cB(x, r) \cap K \big)}{r^d},  
\end{align*}
where $r_K = \inf_{x \in \partial K} \norm{v_K - x} > 0$. 
To prove the lemma, it suffices to show $\inf_{x \in \partial K} F_K(x) > 0$.

For all $x \in \partial K$, define 
\begin{align*}
    C_x &= \big\{ \alpha x + (1 - \alpha) z:  z \in \ball(v_K, \delta_K), \, \alpha \in [0, 1] \big\} \\
    H_x &= \left\{ y \in \R^d: \inprod{y - v_K, x - v_K} \ge 0 \right\} \\
    \tilde{C}_x &= C_x \cap H_x.
\end{align*}
By convexity, we have $C_x \subseteq K$ for every $x \in \partial K$. 
Note that for all $r \in (0, r_K)$, it holds that 
\begin{align*}
    \vol \big( C_x \cap \ball(x, r) \big) 
    \geq \vol \big(\tilde{C}_x \cap \ball(x, r) \big)  \geq  \left( \frac{r}{\norm{x - v_K}} \right)^d \cdot \vol \big( \tilde{C}_x \big).
\end{align*}
Note that $\tilde{C}_x$ is a cone with base being a $(d - 1)$-dimensional ball that has radius $\delta_K$, hence 
\begin{align*}
    \vol \left( \tilde{C}_x \right)
    = \frac{\norm{x - v_K} \cdot \omega_{d - 1} \cdot \delta_K^{d - 1}}{d},
\end{align*}
where $\omega_{d  - 1}$ denotes the volume of a $(d - 1)$-dimensional ball.
Therefore, 
\begin{align*}
   \vol \left( C_x \cap \ball(x, r) \right) 
   &\ge \left( \frac{r}{\norm{x - v_K}} \right)^d \cdot \frac{\norm{x - v_K} \cdot \omega_{d - 1} \cdot \delta_K^{d - 1}}{d} \\
    &= \frac{\omega_{d - 1} \cdot \delta_K^{d - 1}}{d \norm{x - v_K}^{d - 1}} \cdot r^d.
\end{align*}
Since $\norm{x - v_K} \le \mathsf{diam}(K)$, we see that for all $r \in (0, r_K)$, 
\begin{align*}
   \vol \left( C_x \cap \ball(x, r) \right) 
    &\ge \frac{\omega_{d - 1} \cdot \delta_K^{d - 1}}{d \cdot \mathsf{diam}(K)^{d - 1}} \cdot r^d.
\end{align*}
The lemma follows by setting $\rho_K = \frac{\omega_{d - 1} \cdot \delta_K^{d - 1}}{d \cdot \mathsf{diam}(K)^{d - 1}}$.

\subsection{Proof of Lemma \ref{le:cm}}
\label{proof:le:cm}

Let $p = \Proj_K(z) \in \R^d$ and $\delta = \dist(z, K) = \norm{z - p}$. 
By convexity of $K$, the hyperplane 
\begin{equation*}
    H = \left\{x \in \R^d: \inprod{x - p,\, \nu_p} = 0\right\}, \qquad
       \nu_p =  \frac{p - z}{\norm{p - z}}
\end{equation*}
is a supporting hyperplane for $K$ at $p \in \partial K$. 
Note that $K \subseteq \left\{x: \left\langle x - p, \nu_p \right\rangle \ge 0\right\}$.
For $x \in \R^d$, define 
\begin{equation*}
    d_H(x) = \inprod{x - p, \nu_p}.
\end{equation*}
By the supporting hyperplane property, $d_H(x) \ge 0$ for all $x \in K$. 
We then partition $K$ according to its distance to $H$. 
Specifically, for $r > 0$, define
\begin{align*}
    K_{1, r} = \{x \in K: d_H(x) \le r\}, \qquad  K_{2, r} = \{x \in K: d_H(x) > r\}.
\end{align*}
We can then upper bound $\norm{m_{\sigma}(z) - p}$ using the following two terms: 
\begin{align*}
    \norm{m_{\sigma}(z) - p} \leq \underbrace{\int_{K_{1, r}} \norm{x - p} \exp \Big(- \frac{\norm{x - z}^2}{2 \sigma^2}\Big) \frac{\D x}{p_{\sigma}(z)}}_{\mathrm{I}} + \underbrace{\int_{K_{2, r}} \norm{x - p} \exp \Big(- \frac{\norm{x - z}^2}{2 \sigma^2}\Big) \frac{\D x}{p_{\sigma}(z)}}_{\mathrm{II}}. 
\end{align*}
We then control terms $\mathrm{I}$ and $\mathrm{II}$ separately below. 

\subsubsection*{Bounding term I}
For notational simplicity, define $w_{\sigma}(x) = \exp(- \norm{x - z}^2/(2\sigma^2))$.
We can then upper bound term $\mathrm{I}$ as 
\begin{equation*}
    \mathrm{I} = \frac{\int_{K_{1, r}} \norm{x - p}\, w_{\sigma}(x)\,\D x}{\int_K w_{\sigma}(x)\,\D x}
       \le
       \frac{\int_{K_{1, r}} \norm{x - p}\, w_{\sigma}(x)\,\D x}{\int_{K_{1, r}} w_{\sigma}(x)\,\D x}
       = \mathbb{E}_{X \sim \mu_{1, r}}\big[\norm{X - p}\big],
\end{equation*}
where $\mu_{1, r}(\dd x) \propto \ind_{K_{1, r}} w_{\sigma}(x) \D x$.
We next prove that
\begin{equation}
\label{eq:stepI-target}
    \mathbb{E}_{X \sim \mu_{1, r}}\big[\norm{X - p}\big] \le r + \sqrt{d}\,\sigma.
\end{equation}
Without loss of generality, via translation and rotation, 
we will assume that $p = 0$ and $\nu_p = e_d$ (the $d$-th standard basis in $\R^d$).
In this case, for $x = (x_1,x_2, \cdots, x_d)^{\top}$, it holds that $d_H(x) = x_{d}$ and $K_{1, r} \subseteq \{x \in \R^d:\, x_{d} \in [0,r]\}$.
Therefore, 
\begin{equation}
\label{eq:13.5}
     \mathbb{E}_{X \sim \mu_{1, r}}\big[\norm{X - p}\big]
       \le  \mathbb{E}_{X \sim \mu_{1, r}}\big[|X_d|\big] +  \mathbb{E}_{X \sim \mu_{1, r}}\big[\norm{X_{1:(d - 1)}}\big]
       \le r + \mathbb{E}_{X \sim \mu_{1, r}}\big[\norm{X_{1:(d - 1)}}\big].
\end{equation}
We then upper bound $\mathbb{E}_{X \sim \mu_{1, r}}\big[\norm{X_{1:(d - 1)}}\big]$.
For $u \in \mathbb{S}^{d - 1}$, we define $\theta(u) = \sup\{\theta \geq 0: \, \theta u \in K_{1, r}\}$. 
Rewriting this expectation in spherical coordinates, we obtain 
\begin{align}
\label{eq:1:d-1-upper}
    \mathbb{E}_{X \sim \mu_{1, r}}\big[\norm{X_{1:(d - 1)}}\big]
    &\leq  \frac{\int_{\mathbb{S}^{d - 1}} \int_0^{\theta(u)} r^{d} e^{-(r + \alpha(u) \delta)^2 / (2\sigma^2)} \D r \D u}{\int_{\mathbb{S}^{d - 1}} \int_0^{\theta(u)} r^{d - 1} e^{-(r + \alpha(u) \delta)^2 / (2\sigma^2)} \D r \D u} \\
    &\leq \sup_{u \in \mathbb{S}^{d - 1}} \frac{\int_0^{\theta(u)} r^{d} e^{-(r + \alpha(u) \delta)^2 / (2\sigma^2)} \D r}{\int_0^{\theta(u)} r^{d - 1} e^{-(r + \alpha(u) \delta)^2 / (2\sigma^2)} \D r}, 
\end{align}
where $\alpha(u) = \norm{u_{1:(d - 1)}}$, and recall $\delta = \dist(z, H)$. 
We then upper bound the right-hand side of \cref{eq:1:d-1-upper}. 
To this end, for $w \geq 0$ and $u \in \mathbb{S}^{d - 1}$, we define 
\begin{align*}
    f_u(w) = \frac{\int_0^{\theta(u)} r^{d} e^{-(r + w)^2 / (2\sigma^2)} \D r}{\int_0^{\theta(u)} r^{d - 1} e^{-(r + w)^2 / (2\sigma^2)} \D r}. 
\end{align*}
Let $R_{u, w}$ be a random variable with density function $p_{u, w}(r) \propto \ind_{r \in [0, \theta(u)]} r^{d - 1} e^{-(r + w)^2 / 2}$. 
Then $f_u(w) = \E[R_{u, w}]$. 
Taking the derivative of $f_u(w)$ with respect to $w$, we get 
\begin{align*}
    f_u'(w) = - \frac{1}{\sigma^2} \mathsf{Var}[R_{u, w}] \leq 0. 
\end{align*}
Therefore, $f_u(w) \leq f_u(0)$ for all $u \in \mathbb{S}^{d - 1}$ and $w \geq 0$. 
Substituting this upper bound into \cref{eq:1:d-1-upper}, we obtain 
\begin{align}
\label{eq:mu1r-upper}
    \mathbb{E}_{X \sim \mu_{1, r}}\big[\norm{X_{1:(d - 1)}}\big] \leq \sup_{u \in \mathbb{S}^{d - 1}}  \frac{\int_0^{\theta(u)} r^{d} e^{-r^2 / (2\sigma^2)} \D r}{\int_0^{\theta(u)} r^{d - 1} e^{-r^2 / (2\sigma^2)} \D r}.
\end{align}
We claim that the function 
\begin{equation}
     y \mapsto \frac{\int_0^{y} r^{d} e^{-r^2 / (2\sigma^2)} \D r}{\int_0^{y} r^{d - 1} e^{-r^2 / (2\sigma^2)} \D r}
\end{equation}
is increasing on $[0, \infty)$.
Indeed, differentiating the function with respect to $y$, we get
\begin{align*}
    &\frac{\left(y^d e^{-y^2/ (2\sigma^2)}\right) \cdot \left( \int_0^{y} r^{d - 1} e^{-r^2 / (2\sigma^2)} \D r\right) - \left(\int_0^y r^d e^{-r^2/(2\sigma^2)} \D r \right) \cdot \left(y^{d - 1} e^{-y^2 / (2\sigma^2)}\right)}{\left(\int_0^{y} r^{d - 1} e^{-r^2 / (2\sigma^2)} \D r \right)^2} \\
    &= \frac{e^{-y^2 / (2\sigma^2)}}{\left(\int_0^{y} r^{d - 1} e^{-r^2 / (2\sigma^2)} \D r\right)^2} \cdot \int_0^y r^{d - 1} e^{-r^2 / (2\sigma^2)} (y - r) \D r \geq 0.
\end{align*}
Therefore, 
\begin{equation*}
   \E_{X \sim \mu_{1, r}}\big[\norm{X_{1:(d - 1)}}\big] 
   \le \frac{\int_0^{\infty} r^{d} e^{-r^2 / (2\sigma^2)} \D r}{\int_0^{\infty} r^{d - 1} e^{-r^2 / (2\sigma^2)} \D r}
   = \frac{\sqrt{2} \sigma \Gamma((d + 1) / 2)}{\Gamma(d / 2)} \leq \sigma \sqrt{d},
\end{equation*}
where $\Gamma(\cdot)$ denotes the Gamma function. 
Combining the above upper bound with \cref{eq:13.5,eq:mu1r-upper} completes the proof of \cref{eq:stepI-target}.

\subsubsection*{Bounding term II}

Note that 
\begin{align}
\label{eq:upper-bound-m}
    \mathrm{II} = \int_{K_{2, r}} \norm{x - p} \exp \Big(- \frac{\norm{x - z}^2}{2 \sigma^2}\Big) \frac{\D x}{p_{\sigma}(z)} \leq \frac{\mbox{diam}(K)}{1 + \underbrace{\frac{\int_{K_{1, r}} w_{\sigma}(x) \dd x}{\int_{K_{2, r}} w_{\sigma}(x) \dd x }}_{_{\diamondsuit}}}, 
\end{align}
where $\mbox{diam}(K) = \sup\{\|x - y\|: \, x, y \in K\}$. 
We then lower bound term ${\diamondsuit}$. 
Observe that $K \cap \cB(p, r / 2) \subseteq  K_{1, r}$. 
For $x \in K \cap \cB(p, r / 2)$, it holds that
\begin{equation*}
    \norm{x - z} \le \norm{x - p} + \norm{z - p} \le \frac{r}{2} + \delta.
\end{equation*}
Therefore,
\begin{align}
\begin{split}
    \int_{K_{1, r}} \exp  \Big(- \frac{\norm{x - z}^2}{2 \sigma^2}\Big) \D x
    &\ge 
    \int_{K_{1, r / 2} \cap \ball(p, r / 2)} \exp   \Big(- \frac{\norm{x - z}^2}{2 \sigma^2}\Big) \D x  \\
    &\ge \vol \left( K \cap \ball(p, r / 2) \right) \cdot \exp \Big( - \frac{(r / 2 + \delta)^2 }{2 \sigma^2} \Big)  \\
    &\ge \rho_K \left(\frac{r}{2}\right)^d \exp \Big( - \frac{(r / 2 + \delta)^2 }{2 \sigma^2} \Big), \label{eq:diamond1}
\end{split}
\end{align}
where the last inequality follows from  Lemma~\ref{lemma:volume-lower-bound}.
On the other hand, $d_H(z) = - \norm{z - p} = - \delta < 0$. 
Hence, $x \in K$ and $z$ fall on different sides of $H$, and the segment connecting them intersects $H$ at a point $h \in \mathbb{R}^d$.
Using this, we see that for $x \in K_{2, r}$, 
\begin{align*}
    \norm{x - z} = \norm{x - h} + \norm{z - h} \geq d_H(x) + \norm{z - p} \geq r + \delta, 
\end{align*}
which further implies that 
\begin{align}
    \int_{K_{2, r}} \exp  \Big(- \frac{\norm{x - z}^2}{2 \sigma^2}\Big) \D x \leq \vol(K) \exp \Big( - \frac{(r + \delta)^2}{2 \sigma^2} \Big). \label{eq:diamond2}
\end{align}
Combining \cref{eq:diamond1,eq:diamond2}, we obtain the following lower bound for $\diamondsuit$: 
\begin{align}
\label{eq:upper-bound-diamond}
\begin{split}
    \diamondsuit
    &\ge \frac{\rho_K (r/2)^d}{\vol(K)}
    \exp \Big( \,\frac{(r + \delta)^2 - (r / 2 + \delta)^2}{2 \sigma^2} \, \Big) \\
    &= \frac{\rho_K}{2^d\vol(K)}
    \cdot r^d
    \cdot  
    \exp \left( \frac{3r^2/4  + r\delta}{2 \sigma^2} \right).
    \end{split}
   \end{align}

\subsubsection*{Proof of the lemma}

We then separately discuss two cases depending on the value of $\delta$.
\begin{itemize}
    \item If $\delta > \sigma$, then we let $r = (2d    + 2) \sigma \log \sigma^{-1}$. Plugging this choice and \cref{eq:upper-bound-diamond} into \cref{eq:upper-bound-m}, and combining this with the upper bound from Step I, we obtain that 
\begin{align*}
    \norm{m_{\sigma}(z) - \Proj_K(z)} \leq &\, (2d + 2) \sigma \log \sigma^{-1} + \sqrt{d} \sigma + C_1 r^{-d} \exp \Big( - \frac{r \delta + 3r^2/4 }{2 \sigma^2} \Big) \\
    \leq & \, (2d + 2) \sigma \log \sigma^{-1} + \sqrt{d} \sigma + C_1 r^{-d} \exp \Big( - \frac{r \delta  }{2 \sigma^2} \Big) \\
    \leq & C_2\, \sigma \log \sigma^{-1},
\end{align*}
where $C_1, C_2$ are positive constants that depend only on $K$. 
\item If $0 < \delta \leq \sigma$.
Let $r = 2\sqrt{d + 1}\, \sigma \sqrt{\log \sigma^{-1}}$. 
Then we have 
\begin{align*}
    \norm{m_{\sigma}(z) - \Proj_K(z)} \leq & 2\sqrt{d + 1}\, \sigma \sqrt{\log \sigma^{-1}} + \sqrt{d} \sigma + C_3 r^{-d}  \exp \Big( - \frac{ 3 r^2}{8 \sigma^2} \Big) \\
    \leq & C_3 \sigma \log \sigma^{-1}, 
\end{align*}
where $C_3, C_4$ are positive constants that depend only on $K$. 

\end{itemize}

The proof is done. 

\subsection{Proof of Lemma \ref{lem:dist_laplacian}}
\label{proof:lem:dist_laplacian}
The following proof is adapted from \cite[Theorem 14.16]{gilbarg1977elliptic}.
For notational convenience, we write $d_K(x) = \dist(x, K)$ and denote the inward-pointing unit normal vector field on $\partial K$ by $\nu_{\partial K}$.
Since $\nabla d_K(x) = - \nu_{\partial K}(\Proj_K(x))$ for all $x \in \R^d \setminus K$, 
    it suffices to show $\nabla d_K$ is $C^{k-1}$ on $\R^d \setminus K$.

As $K$ is convex, for each point $x \not\in K$, $y = \Proj_K(x) \in \partial K$ satisfies $\norm{y - x} = d_K(x)$.
    Hence, the points $x$ and $y$ satisfy the following equation
    \begin{equation*}
    x = y - \nu_{\partial K} (y) d_K(x),
    \end{equation*}
For a fixed point $x_0 \in \R^d \setminus K$, 
let $y_0 = \Proj_K(x_0)$, and choose a principal coordinate system at $y_0$, 
where $\nu_{\partial K}(y_0) = e_n$ and $y_0 = 0_d$ under this coordinate system, and the principal curvatures at $y_0$ are given by $\kappa_1, \dots, \kappa_{n-1}$.
We will work in this principal coordinate system throughout the remainder of the proof.

We define a mapping $g = (g_1, \dots, g_n)$ from $\mathcal{U} = \Pi(\mathcal{N}(y_0) \cap T(y_0)) \times \R$ into $\R^n$ by
    \begin{equation*}
        g \colon (y', t) \mapsto \Phi(y') + t \cdot \nu_{\partial K} \circ \Phi(y') , \quad \Phi(y') = (y', \phi(y')),
    \end{equation*}
where $T(y_0) = \{x \in \R^d: \langle x, e_n\rangle = 0\}$ is the tangent space at $y_0$,  $\mathcal{N}(y_0)$ is a neighborhood of $y_0$
    such that $\partial K \cap \mathcal{N}(y_0)$ can be represented as the graph of a $C^k$ convex function $\phi \colon \Pi( T(y_0) \cap \mathcal{N}(y_0)) \to \R$, and $\Pi: \R^n \to \R^{n - 1}$ denotes the projection onto the first $n - 1$ coordinates.

Since $\nu_{\partial K} \in C^{k - 1}$ and $\Phi \in C^k$, we see that $g \in C^{k-1}(\mathcal{U})$.
For $y' \in \Pi(\mathcal{N}(y_0) \cap T(y_0))$ and $x \notin K$, the total derivative matrix of $g$ at $(y', - d_K(x))$ (with respect to the principal coordinate system) is given by
\begin{equation}\label{eqn:map_principal_directions}
 \nabla g (y', -d_K(x)) = 
 \diag
\begin{pmatrix} 
1 + d_K(x) \kappa_1,\, 1 + d_K(x) \kappa_2,\, \dots,\,  1 + d_K(x) \kappa_{n-1},\, 1
\end{pmatrix}.
\end{equation}
See \citep[Section 14.16]{gilbarg1977elliptic} for more details.
Then the Jacobian of $g$ at $(0_{d - 1}, -d_K(x_0))$ is given by
\begin{equation*}
\det \left( \nabla g (0_{d - 1}, -d_K(x_0)) \right) 
= \prod_{i=1}^{n-1} (1 + d_K(x_0) \kappa_i) \neq 0
\end{equation*}
 since $\kappa_i \ge 0$.
 Note that $g(0_{d - 1}, -d_K(x_0)) = x_0$.
Then, the inverse mapping theorem implies that for some neighborhood $\mathcal{A}(x_0)$ of $x_0$, the mapping $g$ is a $C^{k-1}$-diffeomorphism between $\mathcal{U}$ and $\mathcal{A}(x_0)$.
Therefore, the mapping 
\begin{equation*}
    \Proj_{K}(x) = \big( \,\Pi \circ g^{-1}(x), \, \phi(\Pi \circ g^{-1}(x))\, \big)
\end{equation*}
is $C^{k-1}$ on $\mathcal{A}(x_0)$.
Finally, since $\nabla d_K(x) = - \nu_{\partial K} \circ \Proj_{K}(x)$, 
we see that $\nabla d_K$ is $C^{k-1}$ on $\mathcal{A}(x_0)$. 
Hence, $d_K \in C^k(\mathcal{A}(x_0))$.
Since $x_0$ is arbitrary, we conclude that $d_K \in C^k(\R^d \setminus K)$.

Next, we compute $\Delta d_K$.
We claim that the Hessian matrix $\nabla^2 d_K(x_0)$ with respect to the principal coordinate system at $y_0$ is given by 
\begin{equation*}
  \nabla^2 d_K(x_0)
  = \diag 
    \begin{pmatrix}
    \frac{\kappa_1}{1 + \kappa_1 d_K(x_0)},\, \frac{\kappa_2}{1 + \kappa_2 d_K(x_0)},\, \dots,\, \frac{\kappa_{n-1}}{1 + \kappa_{n-1} d_K(x_0)},\, 0
    \end{pmatrix}.
\end{equation*}
For any $i, j \in [n - 1]$, inverting the matrix given in \cref{eqn:map_principal_directions}, we have
\begin{align*}
    \partial_{ij} d_K(x_0)
    &= - \partial_i \left( \nu^j_{\partial K} \circ g^{-1}(x_0) \right)  \\
    &= - \langle \nabla \nu_{\partial K}^j(y_0),  \partial_i g^{-1}(x_0) \rangle  \\
    &= \frac{\kappa_i}{1 + \kappa_i d_K(x_0)} \cdot \ind_{i = j}.
\end{align*}
Furthermore, for any $i \in [n]$, we have 
\begin{align*}
    \partial_{n, i} d_K(x_0)
    = - \partial_n \left( \nu^i_{\partial K} \circ g^{-1}(x_0) \right) = 0,
\end{align*}  
since $\nu_{\partial K}$ is constant in the $e_n = \nu_{\partial K}(y_0)$ direction.
Therefore, 
\begin{align*}
    \Delta d_K(x_0) 
    &= \sum_{i=1}^{n-1} \frac{\kappa_i}{1 + \kappa_i d_K(x_0)}
\end{align*}
as desired.

\section{Proofs for guided DDIM}

\subsection{Proof of Theorem \ref{thm:continuous-DDIM}}
\label{proof:continuous-DDIM}

To prove the theorem, we first establish the following lemma, which controls the derivative of $\dist(z_t, K_{\eta_0})$ for $t$ close to $T$. 

\begin{lem}\label{le:dist_deriv_general}
    Fix  $\eps \in (0, 1)$. 
    Under the assumptions of Theorem~\ref{thm:continuous-DDIM}, there exists $0 < \tau_\varepsilon < T$ that depends solely on $(\eps, p_0, \gamma)$, such that for all $z_t$ that satisfies $\dist(z_t, K_{\eta_0}) > \eps$ and $t \in [\tau_\varepsilon, T)$, it holds that 
    \begin{align*}
        \frac{\dd}{\dd t} \dist(z_t, K_{\eta_0}) \leq - \frac{1}{4\sigma_{T - t}^2} \dist(z_t, K_{\eta_0}). 
    \end{align*}
\end{lem}

\begin{proof}[Proof of Lemma \ref{le:dist_deriv_general}]
    The proof is given in Appendix \ref{proof:le:dist_deriv_general}. 
\end{proof}

To prove the theorem, it suffices to show for any fixed $\eps > 0$, 
\begin{equation*}
    \limsup_{t \to T} \dist(z_t, K_{\eta_0}) \le \eps.
\end{equation*}
Let $\tau_\eps$ be the time lower bound in Lemma~\ref{le:dist_deriv_general}.
First, we show that there exists $\tau_\eps < \tau' < T$, such that
\begin{equation*}
    \dist(z_{\tau'}, K_{\eta_0})  \le \eps.
\end{equation*}
Suppose that $\dist(z_{\tau'}, K_{\eta_0}) > \eps$ for all 
$\tau_\eps < \tau' < T$.
Then, applying Gr\"{o}nwall's inequality and Lemma~\ref{le:dist_deriv_general}, 
we see that for all $\tau_\eps < \tau < T$, 
\begin{align*}
    \dist(z_{\tau}, K_{\eta_0})
    \le \dist(z_{\tau_\eps}, K_{\eta_0}) \exp \left( - \int_{\tau_\eps}^{\tau} \frac{1}{4 \sigma_{T - \tau}^2} \D \tau  \right),
\end{align*}
implying that $\dist(z_{\tau}, K_{\eta_0}) \to 0$ as $\tau \to T$. 
A contradiction. 
Therefore, there exists $\tau_\eps < \tau' < T$, such that $\dist(z_{\tau'}, K_{\eta_0}) \le \eps$.
Now, since the map $t \mapsto \dist(z_t, K_{\eta_0})$ is decreasing whenever $\dist(z_t, K_{\eta_0}) \ge \eps$ and $\tau_\eps < t < t$, 
we see that $\dist(z_t, K_{\eta_0}) \le \eps$ for all $\tau' \le t < T$.
The proof is completed by mapping the problem back to the $x$-space and noting that $\varepsilon$ is arbitrary.

\subsection{Proof of Theorem \ref{thm:guided_discretized_ode_qualitative}}
\label{proof:thm:guided_discretized_ode_qualitative}

Recall that $(\hat x_t)_{0 \leq t < T}$ is defined in \cref{eq:discretized-ODE}. 
Define $\hat z_t = \lambda_{T - t_k}^{-1} \hat x_t$ for $t \in [t_k, t_{k + 1})$, then on the same interval, 
\begin{align}
\label{eq:26}
    \dd \hat z_t = & \, \hat z_t \, \dd t +  \left\{ \frac{\gamma}{\sigma_{T - t_k}^2} \left( m_{\eta_0, T - t_k}(\hat z_{t_k}) - \hat z_{t_k} \right) - \frac{\gamma - 1}{\sigma_{T - t_k}^2} \sum_{\eta \in \cI} \zeta_{\eta, T - t_k}(\hat z_{t_k}) \left( m_{\eta, T - t_k}(\hat z_{t_k}) - \hat z_{t_k} \right) \right\} \dd t. 
\end{align}
The above equation implies that for all $t \in [t_k, t_{k + 1})$,  
\begin{align}
\label{eq:zt-ztk}
\begin{split}
    & \hat z_t - \hat z_{t_k} \\
    & = (e^{t - t_k} - 1) \Big( \hat z_{t_k} + \frac{\gamma}{\sigma_{T - t_k}^2} \left( m_{\eta_0, T - t_k}(\hat z_{t_k}) - \hat z_{t_k} \right) - \frac{\gamma - 1}{\sigma_{T - t_k}^2} \sum_{\eta \in \cI} \zeta_{\eta, T - t_k}(\hat z_{t_k}) \left( m_{\eta, T - t_k}(\hat z_{t_k}) - \hat z_{t_k} \right) \Big). 
\end{split}
\end{align}
To prove the theorem, we first establish the following lemma. 
\begin{lem}
\label{lemma:discretized_ddim} 
Define $d_0 = \inf_{y \in K_{\eta},\, y' \in K_{\eta'}, \, \eta \neq \eta'} \norm{y - y'}$.
Under the assumptions of Theorem \ref{thm:guided_discretized_ode_qualitative}, with a sufficiently small $\kappa_0 > 0$ that depends only on $(p_0, \gamma)$, there exists $0 < \tau_0 < T$ that also depends only on $(p_0, \gamma)$, such that for all $\hat z_{t_k}$ that satisfies $\dist(\hat z_{t_k}, K_{\eta_0}) \geq  d_0 / 3$ and $t_k \in [\tau_0, T)$, for all $t \in [t_k, t_{k + 1})$ we have
\begin{align*}
     \frac{\dd }{\dd t} \norm{\hat z_t - \Proj_{K_{\eta_0}}(\hat z_{t_k})}^2 \leq - \frac{2}{5\sigma_{T - t_k}^2} \dist(\hat z_{t_k}, K_{\eta_0})^2.
\end{align*}

\end{lem}
\begin{proof}[Proof of Lemma \ref{lemma:discretized_ddim}]
    We prove Lemma \ref{lemma:discretized_ddim} in Appendix \ref{proof:lemma:discretized_ddim}. 
\end{proof}

Lemma \ref{lemma:discretized_ddim} implies that $\hat z_t$ will eventually enter a region closer to $K_{\eta_0}$ than to any other component. We then show that $\hat z_t$ subsequently approaches $K_{\eta_0}$.

\begin{lem}
\label{lemma:discretized_ddim2}

Under the assumptions of Theorem \ref{thm:guided_discretized_ode_qualitative},
with a sufficiently small $\kappa_0 > 0$ that depends only on $(p_0, \gamma)$,
for any $\varepsilon \in (0, \min\{1, d_0 / 3\})$, there exists $0 < \tau_{\varepsilon} < T$ that depends solely on $(p_0, \varepsilon,\gamma)$, such that for all $\hat z_{t_k}$ that satisfies $\dist(\hat z_{t_k}, K_{\eta_0}) \geq  \varepsilon$ and $t_k \in [\tau_\varepsilon, T)$, for all $t \in [t_k, t_{k + 1})$ we have
\begin{align*}
    \frac{\dd }{\dd t} \norm{\hat z_t - \Proj_{K_{\eta_0}}(\hat z_{t_k})}^2 \leq - \frac{2}{5\sigma_{T - t_k}^2} \dist(\hat z_{t_k}, K_{\eta_0})^2. 
\end{align*}
\end{lem}
\begin{proof}[Proof of Lemma \ref{lemma:discretized_ddim2}]
    We prove Lemma \ref{lemma:discretized_ddim2} in Appendix \ref{proof:lemma:discretized_ddim2}. 
\end{proof}

\begin{lem}
\label{lemma:discretized_ddim3}
	Under the assumptions of Theorem \ref{thm:guided_discretized_ode_qualitative}, with a sufficiently small $\kappa_0 > 0$ that depends only on $(p_0, \gamma)$, for any $\delta \in (0, 1)$, there exist $\varepsilon_\delta, \bar\tau_\delta > 0$ that depend only on $(\delta, p_0, \gamma)$, such that:
	\begin{enumerate}
		\item For all $\hat z_{t_k}$ that satisfies $\dist(\hat z_{t_k}, K_{\eta_0}) \leq  \varepsilon_\delta$ and $t_k \in [\bar\tau_\delta, T)$, for all $t \in [t_k, t_{k + 1})$ we have
		\begin{align*}
			\dist (\hat z_t, K_{\eta_0}) \leq \dist (\hat z_{t_k}, K_{\eta_0}) + \delta. 
		\end{align*}
		\item For all $\delta \in (0, 1)$ we have $\varepsilon_\delta < d_0 / 3$. In addition,  $\eps_\delta \to 0$ as $\delta \to 0$. 
	\end{enumerate}
\end{lem}
\begin{proof}[Proof of Lemma \ref{lemma:discretized_ddim3}]
	We prove Lemma \ref{lemma:discretized_ddim3} in Appendix \ref{proof:discretized_ddim3}. 
\end{proof}

We now prove the theorem using the above lemmas.
Choose $\kappa_0$ sufficiently small so that the conditions of Lemmas \ref{lemma:discretized_ddim}--\ref{lemma:discretized_ddim3} are satisfied.
For any $\delta > 0$, 
define $\tilde \tau_\delta = \max \{\tau_0, \tau_{\eps_\delta}, \bar \tau_\delta \}$.
Choose $k_\delta \in \mathbb{N}+$ such that $T - t_{k_\delta} > \tilde \tau_\delta$.
By Lemma \ref{lemma:discretized_ddim2}, there exists $k' \ge k_\delta$ such that $\dist(\hat z_{t_{k'}}, K_{\eta_0}) < \eps_\delta$.
Applying Lemma \ref{lemma:discretized_ddim3}, we conclude that for any $k \ge k'$, 
\begin{equation*}
   \dist(\hat z_{t_{k'}}, K_{\eta_0}) \le \eps_\delta + \delta.
\end{equation*}
The statememt of the theorem then follows by letting $\delta \to 0$.

\subsection{Proof of Lemma \ref{le:dist_deriv_general}}
\label{proof:le:dist_deriv_general}

Since the set $K_{\eta_0}$ is convex, the distance function $\dist(\cdot, K_{\eta_0})$ is continuously differentiable on $\R^d \setminus K_{\eta_0}$ \citep[Theorem 4.8 Part(5)]{federer1959curvature}. 
Therefore, by the chain rule and the ODE reformulation \eqref{eqn:ode_general}, we have
\begin{align}
\begin{split}
& \frac{\dd}{\dd t} \dist(z_t, K_{\eta_0}) = \left\langle 
              \frac{z_t - \Proj_{K_{\eta_0}}(z_t)}{\norm{z_t - \Proj_{K_{\eta_0}}(z_t)}}, \,
              \frac{\dd}{\dd t}{z}_t\right\rangle \label{eqn:dist_deriv} \\
         &= \frac{1}{\sigma_{T - t}^2}\left\langle \frac{z_t - \mathrm{Proj}_{K_{\eta_0}}(z_t)}{\norm{z_t - \mathrm{Proj}_{K_{\eta_0}}(z_t)}}, \,\gamma F_{\eta_0, T - t}(z_t) - (\gamma - 1) \sum_{\eta \in \cI} \zeta_{\eta, T - t}(z_t) F_{\eta, T - t}(z_t)\right\rangle,  
\end{split}
\end{align}
where we recall that $F_{\eta, T - t}$ is defined in \cref{eq:F-def}. 
For $z \in \R^d$, $\eta \in \cI$ and $0 < t \leq T$, define
\begin{align*}
    \Delta_{\eta, t}(z) = \Proj_{K_{\eta}}(z) - m_{\eta, t}(z)
\end{align*}
as the difference between the projection $\Proj_{K_{\eta}}(z)$ and the tilted mean $m_{\eta, t}(z)$ defined in \cref{eq:zeta-and-m}.
By the Cauchy-Schwarz inequality, we see that 
\begin{align}
\begin{split}
        & \text{Equation~\eqref{eqn:dist_deriv}} \\
        &\le \frac{\gamma - (\gamma - 1)\zeta_{\eta_0, T - t}(z_t)}{\sigma_{T - t}^2}\cdot \Big[ - \dist(z_t, K_{\eta_0})
        + \norm{\Delta_{\eta_0, T - t}(z_t)} \Big]  \label{eqn:dist_derive_II}\\
        &\qquad + \frac{\gamma - 1}{\sigma_{T - t}^2} 
        \sum_{\eta \neq \eta_0} \zeta_{\eta, T - t}(z_t) \cdot 
        \Big[\dist(z_t, K_{\eta}) + \norm{\Delta_{\eta, T - t}(z_t) }\Big] \\
        &= \frac{\gamma - 1}{\sigma_{T - t}^2} 
        \sum_{\eta \neq \eta_0} \zeta_{\eta, T - t}(z_t) \cdot 
        \Big[\dist(z_t, K_{\eta}) - \dist(z_t, K_{\eta_0}) \Big] - \frac{1}{\sigma_{T - t}^2} \dist(z_t, K_{\eta_0}) \\
        &\quad  
        + \frac{1}{\sigma_{T - t}^2}\norm{\Delta_{\eta_0, T - t}(z_t)}
        + \frac{\gamma - 1}{\sigma_{T - t}^2} 
        \sum_{\eta \neq \eta_0} \zeta_{\eta, T - t}(z_t) \Big(\norm{\Delta_{\eta, T - t}(z_t)} + \norm{\Delta_{\eta_0, T - t}(z_t)} \Big). 
\end{split}
\end{align}
For each $\eta \neq \eta_0$, we analyze \cref{eqn:dist_derive_II} by dividing the index set $\cI$ into three subsets.
Specifically, for $0 \leq t < T$, we define the following three time-dependent sets: 
\begin{align}
\label{eq:three-sets}
\begin{split}
    \mathcal{I}_t^{>} 
    &= 
        \left\{\eta \in \mathcal{I}: \dist(z_t, K_{\eta}) - \dist(z_t, K_{\eta_0}) \ge 2 \cdot \Xi_{T - t} \right\}, \\
        \mathcal{I}_t^{\approx}
    &= 
        \left\{\eta \in \mathcal{I}: 
        0 \le \dist(z_t, K_{\eta}) - \dist(z_t, K_{\eta_0}) < 2 \cdot \Xi_{T - t} \right\}, \\
        \mathcal{I}_t^{<}
    &= 
        \left\{\eta \in \mathcal{I}: \dist(z_t, K_{\eta}) - \dist(z_t, K_{\eta_0}) < 0 \right\},
\end{split}
\end{align}
where $\Xi_{T - t} = \sqrt{\tilde{\sigma}_{T - t}}$, and $\tilde \sigma_{T - t}^2 = \lambda_{T - t}^{-2}\sigma_{T - t}^2 $. 
Note that the above three sets do not overlap, $\cI = \mathcal{I}_t^{>} \sqcup \mathcal{I}_t^{\approx} \sqcup \mathcal{I}_t^{<}$, and $\eta_0 \in \mathcal{I}_t^{\approx}$. 
Using the above notations, we have
\begin{align}
\begin{split}
        & \mbox{The last line of Eq. \eqref{eqn:dist_derive_II}} \\
        &\le \underbrace{\frac{\gamma - 1}{\sigma_{T - t}^2} 
        \sum_{\eta \in \mathcal{I}_{t}^{>}} \zeta_{\eta, T - t}(z_t) 
        \Big[\dist(z_t, K_{\eta}) - \dist(z_t, K_{\eta_0}) \Big]}_{\rm I} \\
        &\quad + \underbrace{\frac{\gamma - 1}{\sigma_{T - t}^2} 
        \sum_{\eta \in \mathcal{I}_{t}^{\approx}} \zeta_{\eta, T - t}(z_t)
        \Big[\dist(z_t, K_{\eta}) - \dist(z_t, K_{\eta_0}) \Big] }_{\rm II} - \frac{1}{\sigma_{T - t}^2} \dist(z_t, K_{\eta_0}) \\
        &\quad 
        + \frac{1}{\sigma_{T - t}^2}\norm{\Delta_{\eta_0, T - t}(z_t)}
        + \frac{\gamma - 1}{\sigma_{T - t}^2} 
        \sum_{\eta \neq \eta_0} \zeta_{\eta, T - t}(z_t) \Big(\norm{\Delta_{\eta, T - t}(z_t)} + \norm{\Delta_{\eta_0, T - t}(z_t)} \Big).
        \label{eqn:dist_deriv_III}
\end{split}
\end{align}
To bound term $\rm I$, observe that 
    \begin{align*}
        & \dist(z_t, K_{\eta}) - \dist(z_t, K_{\eta_0})
        = \inf_{y \in K_{\eta}} \norm{z_t - y}
        - \inf_{y \in K_{0}} \norm{z_t - y}
        \le \sup_{y \in K_{\eta}, y^{\prime} \in K_{\eta_0}}
        \norm{y - y^{\prime}}
        \le D_0, \\
        & D_0 = \sup_{y, y' \in \cup_{\eta \in \cI} K_{\eta}} \norm{y - y'}. 
\end{align*}
We then upper bound $\sum_{\eta \in \mathcal{I}_t^{>}} \zeta_{\eta, T - t}(z_t)$. Note that 
\begin{align}
\label{eq:sum-zeta-upper1}
\begin{split}
      &  \sum_{\eta \in \mathcal{I}_t^{>}} \zeta_{\eta, T - t}(z_t) =  \, \frac{\sum_{\eta \in \mathcal{I}_t^{>}}  w_{\eta} \int p_{\eta}(x) e^{- \norm{x - z_t}^2 / (2\tilde{\sigma}_{T - t}^2) } \dd x}{\sum_{\eta \in \cI} w_{\eta} \int p_{\eta}(x) e^{- \norm{x - z_t}^2 / (2\tilde{\sigma}_{T - t}^2) } \dd x} \\
        &\leq  \,
        \Bigg({1 +  \frac{w_{\eta_0} \int_{K_{\eta_0}} p_{\eta_0}(x) e^{- \norm{x - z_t}^2 / (2\tilde{\sigma}_{T - t}^2) } \D x}{ \sum\limits_{\eta \in \mathcal{I}_t^{>}} w_{\eta} \int_{K_{\eta}} p_{\eta}(x) e^{- \norm{x - z_t}^2 / (2\tilde{\sigma}_{T - t}^2) } \D x}}\Bigg)^{-1}.
\end{split}
\end{align}
By the second point of Assumption \ref{assumption:target}, we know that
\begin{align}
\label{eq:sum-zeta-upper2}
& \frac{w_{\eta_0} \int_{K_{\eta_0}} p_{\eta_0}(x) e^{- \norm{x - z_t}^2 / (2\tilde{\sigma}_{T - t}^2) } \D x}{ \sum\limits_{\eta \in \mathcal{I}_t^{>}} w_{\eta} \int_{K_{\eta}} p_{\eta}(x) e^{- \norm{x - z_t}^2 / (2\tilde{\sigma}_{T - t}^2) } \D x} \ge w_{\eta_0} \cdot \frac{\int_{K_{\eta_0}} p_{\eta_0}(x) e^{- \norm{x - z_t}^2 / (2\tilde{\sigma}_{T - t}^2) } \D x}{ \sum\limits_{\eta \in \mathcal{I}_t^{>}}  \int_{K_{\eta}} p_{\eta}(x) e^{- \norm{x - z_t}^2 / (2\tilde{\sigma}_{T - t}^2) } \dd x}.
\end{align}
By Lemma~\ref{lemma:volume-lower-bound}, we know that there exists a constant $\rho_{{\eta_0}} > 0$ that depends only on $K_{\eta_0}$, such that
\begin{align}
\label{eq:ratio-C0-lower}
\begin{split}
       & \int_{K_{\eta_0}} p_{\eta_0}(x) e^{- \norm{x - z_t}^2 / (2 \tilde{\sigma}_{T - t}^2 )} \,  \D x \\
       & \geq  \inf_{x \in \R^d} p_{\eta_0}(x) \cdot  \int_{K_{\eta_0} \cap\,  \cB(\Proj_{K_{\eta_0}}(z_t), \, \Xi_{T - t})
       } e^{- \norm{x - z_t}^2 / (2\tilde{\sigma}_{T - t}^2) } \,  \D x \\
       & \geq \inf_{x \in \R^d} p_{\eta_0}(x) \cdot  \vol \Big(K_{\eta_0} \cap\,  \cB\big(\Proj_{K_{\eta_0}}(z_t), \, \Xi_{T - t} \big) \Big) \cdot  e^{-(\dist(z_t, K_{\eta_0}) + \Xi_{T - t})^2 / (2\tilde{\sigma}_{T - t}^2)} \\
       & \geq  \, \inf_{x \in \R^d} p_{\eta_0}(x) \cdot  \rho_{\eta_0}\, \Xi_{T - t}^{d} \cdot  e^{-(\dist(z_t, K_{\eta_0}) + \Xi_{T - t})^2 / (2\tilde{\sigma}_{T - t}^2)}.
\end{split}
\end{align}
On the other hand, for all $\eta \in \mathcal{I}_t^{>}$, it holds that
\begin{align}
\label{eq:ratio-C0-upper}
    \int_{K_\eta} p_{\eta}(x)\,  e^{- \norm{x - z_t}^2 / (2\tilde{\sigma}_{T - t}^2) } \D x \le e^{-\dist(z_t, K_{\eta})^2 / (2\tilde{\sigma}_{T - t}^2)} \leq  e^{-(\dist(z_t, K_{\eta_0}) + 2 \Xi_{T - t})^2 / (2\tilde{\sigma}_{T - t}^2)}. 
\end{align}
Combining \cref{eq:ratio-C0-lower,eq:ratio-C0-upper}, we conclude that 
\begin{align}
\label{eq:sum-zeta-upper3}
\begin{split}
    & \frac{\int_{K_{\eta_0}} p_{\eta_0}(x) \, e^{- \norm{x - z_t}^2 / (2\tilde{\sigma}_{T - t}^2) } \D x}
    { \sum\limits_{\eta \in \mathcal{I}_t^{>}}  \int_{K_{\eta}} p_{\eta}(x) \, e^{- \norm{x - z_t}^2 / (2\tilde{\sigma}_{T - t}^2) } \dd x} \\
    &\ge \frac{\inf_{x \in \R^d} p_{\eta_0}(x) \cdot \rho_{\eta_0} \Xi_{T - t}^{d}}{\abs{\mathcal{I}}} \cdot
        \exp \Big( \frac{(\dist(z_t, K_{\eta_0}) + 2 \Xi_{T - t})^2 - (\dist(z_t, K_{\eta_0}) + \Xi_{T - t})^2}{2 \tilde{\sigma}_{T - t}^2 } \Big).
\end{split}
\end{align}
Putting together \cref{eq:sum-zeta-upper1,eq:sum-zeta-upper2,eq:sum-zeta-upper3},
we see that
    \begin{align}
    \label{sum:zeta-upper}
    \begin{split}
        & \sum_{\eta \in \mathcal{I}_t^{>}} \zeta_{\eta, T - t}(z_t) \\
        &\le \frac{\abs{\cI}}{w_{\eta_0} \inf_{x \in \R^d} p_{\eta_0}(x) \cdot \rho_{\eta_0} \Xi_{T - t}^{d}} \cdot
        \exp \Big( -\frac{(\dist(z_t, K_{\eta_0}) + 2 \Xi_{T - t})^2 - (\dist(z_t, K_{\eta_0}) + \Xi_{T - t})^2}{2 \tilde{\sigma}_{T - t}^2 } \Big) \\
        &\le \frac{\abs{\cI}}{w_{\eta_0} \inf_{x \in \R^d} p_{\eta_0}(x) \cdot \rho_{\eta_0} \Xi_{T - t}^{d}} \cdot
        \exp \left( -\frac{3 \Xi_{T - t}^2}{2 \tilde{\sigma}_{T - t}^2 }\right).
    \end{split}
    \end{align}
In summary, we have
    \begin{align}\label{eqn:bound_I}
    \begin{split}
         & \frac{\gamma - 1}{\sigma_{T - t}^2} 
        {\sum_{\eta \in \mathcal{I}_{t}^{>}} \zeta_{\eta, T - t}(z_t) 
        \Big[\dist(z_t, K_{\eta}) - \dist(z_t, K_{\eta_0}) \Big]} \\
        & \le \frac{D_0(\gamma - 1) \abs{\cI}}{w_{\eta_0} \inf_{x \in \R^d} p_{\eta_0}(x) \cdot \rho_{\eta_0} \Xi_{T - t}^{d}\sigma^2_{T - t}}  \cdot
        \exp \left( -\frac{3 \Xi_{T - t}^2}{2 \tilde{\sigma}_{T - t}^2 }\right), 
    \end{split}
    \end{align}
which upper bounds term $\rm I$. 
To bound term $\rm II$, simply observe that $\sum_{\eta \in \mathcal{I}_{t}^{\approx}}\zeta_{\eta, T - t}(z_t) \le 1$, and
\begin{equation}\label{eqn:bound_II}
        \frac{\gamma - 1}{\sigma_{T - t}^2} 
        \sum_{\eta \in \mathcal{I}_{t}^{\approx}} \zeta_{\eta, T - t}(z_t)
        \Big[\dist(z_t, K_{\eta}) - \dist(z_t, K_{\eta_0}) \Big]
        \le  \frac{2(\gamma - 1) \Xi_{T - t}}{\sigma_{T - t}^2}.
\end{equation}
By Lemma~\ref{le:cm}, we know that 
\begin{equation}
\label{eq:Delta-upper-bound}
        \norm{\Delta_{\eta, T - t}(z_t)} \le C_{K_{\eta}}\Xi_{T - t}^2 \log (\Xi_{T - t}^{-2})
\end{equation}
for some positive constant $C_{K_\eta}$ that depends only on $K_{\eta}$. 
    Thus, combining the upper bounds in \cref{eqn:dist_deriv_III,eqn:bound_I,eqn:bound_II,eq:Delta-upper-bound}, 
    we see that 
\begin{align}
\label{eqn:dist_deriv_final_bound}
\begin{split}
    & \frac{\dd}{\dd t} \dist(z_t, K_{\eta_0}) \leq \frac{D_0(\gamma - 1) \abs{\cI}}{w_{\eta_0} \inf_{x \in \R^d} p_{\eta_0}(x) \cdot \rho_{\eta_0} \Xi_{T - t}^{d}\sigma^2_{T - t}}  \cdot
        \exp \left( -\frac{3 \Xi_{T - t}^2}{2 \tilde{\sigma}_{T - t}^2 }\right) + \frac{2(\gamma - 1) \Xi_{T - t}}{\sigma_{T - t}^2} \\
        & \qquad \qquad \qquad \qquad  - \frac{1}{\sigma_{T - t}^2} \dist(z_t, K_{\eta_0}) + \frac{2\gamma \sup_{\eta \in \cI} C_{K_{\eta}}\Xi_{T - t}^2 \log (\Xi_{T - t}^{-2})}{\sigma_{T - t}^2}. 
\end{split}
\end{align}
We then can choose $\tau_\varepsilon$ sufficiently close to $T$, such that for all $\tau_\varepsilon \leq t \leq T$, if $\dist(z_t, K_{\eta_0}) \geq \epsilon$, then 
\begin{align*}
    \frac{\dd}{\dd t} \dist(z_t, K_{\eta_0}) \leq - \frac{1}{4\sigma_{T - t}^2} \dist(z_t, K_{\eta_0}).
\end{align*}
The proof is done.

\subsection{Proof of Lemma \ref{lemma:discretized_ddim}}
\label{proof:lemma:discretized_ddim}

For $t \in [t_k, t_{k + 1})$, differentiating $\norm{\hat z_t - \Proj_{K_{\eta_0}}(\hat z_{t_k})}$ with respect to $t$ yields
\begin{align}
\label{eq:d-zt-projk}
\begin{split}
    & \frac{\dd }{\dd t} \norm{\hat z_t - \Proj_{K_{\eta_0}}(\hat z_{t_k})} = \left\langle 
              \frac{\hat z_t -  \Proj_{K_{\eta_0}}(\hat z_{t_k}) }{ \norm{\hat z_t - \Proj_{K_{\eta_0}}(\hat z_{t_k})}}, \,
              \frac{\dd}{\dd t}\hat {z}_t\right\rangle \\
    & = \left\langle \frac{\hat z_t -  \Proj_{K_{\eta_0}}(\hat z_{t_k})}{ \norm{\hat z_t - \Proj_{K_{\eta_0}}(\hat z_{t_k})}}, \, \hat z_t +  \frac{\gamma}{\sigma_{T - t_k}^2} F_{\eta_0, T - t_k}(\hat z_{t_k}) - \frac{\gamma - 1}{\sigma_{T - t_k}^2} \sum_{\eta \in \cI} \zeta_{\eta, T - t_k}(\hat z_{t_k}) F_{\eta, T - t_k}(\hat z_{t_k})\right\rangle \\
    & = e^{t - t_k} \left\langle \frac{\hat z_{t} -  \Proj_{K_{\eta_0}}(\hat z_{t_k})}{ \norm{\hat z_t - \Proj_{K_{\eta_0}}(\hat z_{t_k})}}, \, \hat z_{t_k} + \frac{\gamma}{\sigma_{T - t_k}^2} F_{\eta_0, T - t_k}(\hat z_{t_k}) - \frac{\gamma - 1}{\sigma_{T - t_k}^2} \sum_{\eta \in \cI} \zeta_{\eta, T - t_k}(\hat z_{t_k}) F_{\eta, T - t_k}(\hat z_{t_k})\right\rangle, 
\end{split}
\end{align}
where the second to last equality is by \cref{eq:26}, and the last equality is by \cref{eq:zt-ztk}. 
Decomposing $\hat z_t - \Proj_{K_{\eta_0}}(\hat z_{t_k})$ as the sum of $\hat z_t - \hat z_{t_k}$ and $\hat z_{t_k} - \Proj_{K_{\eta_0}}(\hat z_{t_k})$, we obtain that
\begin{align}
\label{eq:dist-zt-ztk-proj}
    \frac{\dd }{\dd t} \norm{\hat z_t - \Proj_{K_{\eta_0}}(\hat z_{t_k})} = \frac{e^{t - t_k}}{\norm{\hat z_t -\Proj_{K_{\eta_0}}(\hat z_{t_k})}} \cdot \big( \mathsf{F}_1(t) +  \mathsf{F}_2(t)\big), 
\end{align}
where 
\begin{align}
\label{eq:F1-F2}
\begin{split}
   & \mathsf{F}_1(t) = \left\langle \hat z_t - \hat z_{t_k}, \, \hat z_{t_k} + \frac{\gamma}{\sigma_{T - t_k}^2} F_{\eta_0, T - t_k}(\hat z_{t_k}) - \frac{\gamma - 1}{\sigma_{T - t_k}^2} \sum_{\eta \in \cI} \zeta_{\eta, T - t_k}(\hat z_{t_k}) F_{\eta, T - t_k}(\hat z_{t_k}) \right\rangle, \\
   & \mathsf{F}_2(t) = \left\langle \hat z_{t_k} -  \Proj_{K_{\eta_0}}(\hat z_{t_k}),\, \hat z_{t_k} + \frac{\gamma}{\sigma_{T - t_k}^2} F_{\eta_0, T - t_k}(\hat z_{t_k}) - \frac{\gamma - 1}{\sigma_{T - t_k}^2} \sum_{\eta \in \cI} \zeta_{\eta, T - t_k}(\hat z_{t_k}) F_{\eta, T - t_k}(\hat z_{t_k}) \right\rangle. 
\end{split}
\end{align}
Define $D_0 = \sup_{y, y' \in \cup_{\eta \in \cI} K_{\eta}}\norm{y - y'}$ and $R_0 = \sup_{y \in \cup_{\eta \in \cI} K_{\eta}} \norm{y}$. 
By \cref{eq:zt-ztk}, we know that 
\begin{align}
\label{eq:F1-upper}
\begin{split}
    \mathsf{F}_1(t) =& (e^{t - t_k} - 1) \cdot \norm{\hat z_{t_k} + \frac{\gamma}{\sigma_{T - t_k}^2} F_{\eta_0, T - t_k}(\hat z_{t_k}) - \frac{\gamma - 1}{\sigma_{T - t_k}^2} \sum_{\eta \in \cI} \zeta_{\eta, T - t_k}(\hat z_{t_k}) F_{\eta, T - t_k}(\hat z_{t_k})}^2 \\
    \leq & \, (e^{t - t_k} - 1) \cdot \Big( \norm{\hat z_{t_k}} + \frac{\gamma - 1}{\sigma_{T - t_k}^2} \sum_{\eta \in \cI} \zeta_{\eta, T - t_k} (\hat z_{t_k}) \norm{m_{\eta_0, T - t_k} (\hat z_{t_k}) - m_{\eta, T - t_k}(\hat z_{t_k})} \\
    &  + \frac{1}{\sigma_{T - t_k}^2} \dist(\hat z_{t_k}, K_{\eta_0}) + \frac{1}{\sigma_{T - t_k}^2} \norm{\Proj_{K_{\eta_0}}(\hat z_{t_k}) - m_{\eta_0, T - t_k} (\hat z_{t_k})} \Big)^2 \\
    \leq & \, (e^{t - t_k} - 1) \Big(\frac{\sigma_{T - t_k}^2 + 1}{\sigma_{T - t_k}^2} \dist(\hat z_{t_k}, K_{\eta_0}) + R_0 + \frac{\gamma D_0}{\sigma_{T - t_k}^2}  \Big)^2 \\
    \leq & \frac{3(e^{t - t_k} - 1)(\sigma_{T - t_k}^2 + 1)^2\, \dist(\hat z_{t_k}, K_{\eta_0})^2}{\sigma_{T - t_k}^4} + 3 (e^{t - t_k} - 1) R_0^2 + \frac{3(e^{t - t_k} - 1)\gamma^2 D_0^2}{\sigma_{T - t_k}^4}.
\end{split}
\end{align}
By Assumption \ref{assumption:discretization}, we see that for all $0 < t - t_k \leq 1$ we have  $\sigma_{T - t_k}^{-2}(e^{t - t_k} - 1) \lesssim \kappa$, where ``$\lesssim$" hides a positive numerical constant. 

As for $\mathsf{F}_2(t)$, its upper bound is similar to the proof of Lemma \ref{le:dist_deriv_general}. Specifically, by the upper bound in \cref{eqn:dist_deriv_final_bound}, we have (recall $\Xi_{T - t} = \sqrt{\tilde{\sigma}_{T - t}}$ and $\tilde \sigma_{T - t}^2 = \lambda_{T - t}^{-2}\sigma_{T - t}^2 $) 
\begin{align*}
	\begin{split}
    & \frac{1}{\sigma_{T - t_k}^2}\left\langle \frac{\hat z_{t_k} - \mathrm{Proj}_{K_{\eta_0}}(\hat z_{t_k})}{\norm{\hat z_{t_k} - \mathrm{Proj}_{K_{\eta_0}}(\hat z_{t_k})}}, \,\gamma F_{\eta_0, T - t}(\hat z_{t_k}) - (\gamma - 1) \sum_{\eta \in \cI} \zeta_{\eta, T - t}(\hat z_{t_k}) F_{\eta, T - t}(\hat z_{t_k})\right\rangle \\
    & \leq \frac{D_0(\gamma - 1) \abs{\cI}}{w_{\eta_0} \inf_{x \in \R^d} p_{\eta_0}(x) \cdot \rho_{\eta_0} \Xi_{T - {t_k}}^{d}\sigma^2_{T - {t_k}}}  \cdot
        \exp \left( -\frac{3 \Xi_{T - {t_k}}^2}{2 \tilde{\sigma}_{T - {t_k}}^2 }\right) + \frac{2(\gamma - 1) \Xi_{T - {t_k}}}{\sigma_{T - {t_k}}^2}   - \frac{1}{\sigma_{T - {t_k}}^2} \dist(\hat z_{t_k}, K_{\eta_0}) \\
        & \qquad  + \frac{2\gamma \sup_{\eta \in \cI} C_{K_{\eta}}\Xi_{T - {t_k}}^2 \log (\Xi_{T - {t_k}}^{-2})}{\sigma_{T - {t_k}}^2}. 
\end{split}
\end{align*}
The above upper bound implies the following bound for $\mathsf{F}_2(t)$:
\begin{align}
\label{eq:F2-three-lines}
\begin{split}
	\mathsf{F}_2(t) \leq & \frac{D_0(\gamma - 1) \abs{\cI} \dist(\hat z_{t_k}, K_{\eta_0}) }{w_{\eta_0} \inf_{x \in \R^d} p_{\eta_0}(x) \cdot \rho_{\eta_0} \Xi_{T - {t_k}}^{d}\sigma^2_{T - {t_k}}}  \cdot
        \exp \left( -\frac{3 \Xi_{T - {t_k}}^2}{2 \tilde{\sigma}_{T - {t_k}}^2 }\right) + \frac{2(\gamma - 1) \Xi_{T - {t_k}}\dist(\hat z_{t_k}, K_{\eta_0}) }{\sigma_{T - {t_k}}^2}    \\
        & \qquad - \frac{1}{\sigma_{T - {t_k}}^2} \dist(\hat z_{t_k}, K_{\eta_0})^2 + \frac{2\gamma \sup_{\eta \in \cI} C_{K_{\eta}}\Xi_{T - {t_k}}^2 \log (\Xi_{T - {t_k}}^{-2}) \dist(\hat z_{t_k}, K_{\eta_0}) }{\sigma_{T - {t_k}}^2} \\
        & \qquad + \dist(\hat z_{t_k}, K_{\eta_0}) \big(\dist(\hat z_{t_k}, K_{\eta_0}) + R_0 \big). 
\end{split} 
\end{align}
Therefore, when $\dist(\hat z_{t_k}, K_{\eta_0}) \geq d_0 / 3$, there exists $\tau_1 > 0$ that depends only on $(p_0, \gamma)$, such that for all $t_k \in [\tau_1, T)$ and $t \in [t_k, t_{k + 1})$, it holds that
\begin{align}
\label{eq:F2-upper}
    \mathsf{F}_2(t) \leq -\frac{1}{4\sigma_{T - t_k}^2} \dist (\hat z_{t_k}, K_{\eta_0})^2 + \dist (\hat z_{t_k}, K_{\eta_0}) \big( \dist (\hat z_{t_k}, K_{\eta_0}) + R_0\big). 
\end{align}
Combining \cref{eq:dist-zt-ztk-proj,eq:F1-upper,eq:F2-upper}, we conclude that under 
%
%
Assumption \ref{assumption:discretization}, there exist $\tau_0, \kappa_0 > 0$, such that for all $\tau_0 \leq t < T$ and $\kappa \in (0, \kappa_0]$,
\begin{align}
\label{eq:B2-2}
\begin{split}
\frac{1}{2} \frac{\dd }{\dd t} \norm{\hat z_t - \Proj_{K_{\eta_0}}(\hat z_{t_k})}^2 \leq - \frac{1}{5\sigma_{T - t_k}^2} \dist(\hat z_{t_k}, K_{\eta_0})^2,
\end{split}
\end{align}
%
%
completing the proof.

\subsection{Proof of Lemma \ref{lemma:discretized_ddim2}}
\label{proof:lemma:discretized_ddim2}

For any $z \in \R^d$ that satisfies $\dist(z, K_{\eta_0}) \leq d_0 / 3$, it holds that (recall $\tilde \sigma_{T - t}^2 = \lambda_{T - t}^{-2} \sigma_{T - t}^2$)
\begin{align*}
    \sum_{\eta \neq \eta_0} \zeta_{\eta, T - t} (z) = & \, \frac{\sum_{\eta \neq \eta_0}  w_{\eta} \int p_{\eta}(x) e^{- \norm{x - z}^2 / (2\tilde{\sigma}_{T - t}^2) } \dd x}{\sum_{\eta \in \cI} w_{\eta} \int p_{\eta}(x) e^{- \norm{x - z}^2 / (2\tilde{\sigma}_{T - t}^2) } \dd x} \\
    = &  \,
        \Bigg({1 +  \frac{w_{\eta_0} \int_{K_{\eta_0}} p_{\eta_0}(x) e^{- \norm{x - z}^2 / (2\tilde{\sigma}_{T - t}^2) } \D x}{ \sum\limits_{\eta \neq \eta_0} w_{\eta} \int_{K_{\eta}} p_{\eta}(x) e^{- \norm{x - z}^2 / (2\tilde{\sigma}_{T - t}^2) } \D x}}\Bigg)^{-1}
\end{align*}
We next divide the proof into two parts based on the distance to $K_{\eta_0}$.

\subsubsection*{Case I: $\dist(\hat z_{t_k}, K_{\eta_0}) \in[\eps, d_0 / 3)$}

Recall that the three index sets are defined in \cref{eq:three-sets}. 
When $\dist(\hat z_{t_k}, K_{\eta_0}) < d_0 / 3$ and $\tau_{\eps}$ is sufficiently close to $T$ (threshold depending only on $d_0$), we have 
\begin{equation*}
   \dist(z_{t_k}, K_{\eta}) - \dist(\hat{z}_{t_k}, K_{\eta_0}) \ge \frac{2d_0}{3}
\end{equation*}
for all $\eta \neq \eta_0$.
Therefore, via a similar calculation leading to \cref{sum:zeta-upper}, we have
\begin{align}
\label{eq:sum-zeta-upper-2}
    \begin{split}
    \sum_{\eta \neq \eta_0} \zeta_{\eta, T - t} (z) 
    &\le \frac{\abs{\cI}}{w_{\eta_0} \inf_{x \in \R^d} p_{\eta_0}(x) \cdot \rho_{\eta_0} \Xi_{T - t}^{d}} \cdot
        \exp \Big( -\frac{(\dist(z_t, K_{\eta_0}) + 2/3 \cdot d_0 )^2 - (\dist(z_t, K_{\eta_0}) + \Xi_{T - t})^2}{2 \tilde{\sigma}_{T - t}^2 } \Big) \\
    &\le 
    \frac{\abs{\cI}}{w_{\eta_0} \inf_{x \in \R^d} p_{\eta_0}(x) \cdot \rho_{\eta_0} (d_0 / 6)^{d}} \cdot
        \exp \left( -\frac{ d_0^2}{24  \tilde{\sigma}_{T - t}^2 }\right). 
    \end{split}
\end{align}
For all $t \in [t_{k}, t_{k + 1})$, recall that 
\begin{align*}
    \frac{\dd }{\dd t} \norm{\hat z_t - \Proj_{K_{\eta_0}}(\hat z_{t_k})} = \frac{e^{t - t_k}}{\norm{\hat z_t -\Proj_{K_{\eta_0}}(\hat z_{t_k})}} \cdot \big( \mathsf{F}_1(t) +  \mathsf{F}_2(t)\big), 
\end{align*}
where $\mathsf{F}_1(t)$ and $\mathsf{F}_2(t)$ are defined in \cref{eq:F1-F2}. 
By Lemma \ref{le:cm}, 
\begin{align}
\label{eq:F-upper-dist-Proj-m}
\begin{split}
     \norm{F_{\eta_0, T - t_k}(\hat z_{t_k})} \leq &\, \dist(\hat z_{t_k}, K_{\eta_0}) + \norm{\Proj_{K_{\eta_0}}(\hat z_{t_k}) - m_{\eta_0, T - t_k}(\hat z_{t_k})} \\
     \leq &\, \dist(\hat z_{t_k}, K_{\eta_0}) + C_{K_{\eta_0}} \tilde \sigma_{T - t_k} \log \tilde \sigma_{T - t_k}^{-1}, 
\end{split}
\end{align}
where $C_{K_{\eta_0}} > 0$ is a constant that depends only on $K_{\eta_0}$. 
Recall that $d_0 = \inf_{y \in K_{\eta},\, y' \in K_{\eta'}, \, \eta \neq \eta'} \norm{y - y'}$, $D_0 = \sup_{y, y' \in \cup_{\eta \in \cI} K_{\eta}}\norm{y - y'}$ and $R_0 = \sup_{y \in \cup_{\eta \in \cI} K_{\eta}} \norm{y}$.
By \cref{eq:sum-zeta-upper-2,eq:F-upper-dist-Proj-m} and the triangle inequality, we see that  
\begin{align}
\label{eq:F1(t)}
\begin{split}
    \mathsf{F}_1(t) = &(e^{t - t_k} - 1) \cdot \norm{\hat z_{t_k} + \frac{\gamma}{\sigma_{T - t_k}^2} F_{\eta_0, T - t_k}(\hat z_{t_k}) - \frac{\gamma - 1}{\sigma_{T - t_k}^2} \sum_{\eta \in \cI} \zeta_{\eta, T - t_k}(\hat z_{t_k}) F_{\eta, T - t_k}(\hat z_{t_k})}^2 \\
    &\leq  (e^{t - t_k} - 1) \cdot \Big( \norm{\hat z_{t_k}} + \frac{\gamma\big(\dist(\hat z_{t_k}, K_{\eta_0}) + C_{K_{\eta_0}} \tilde \sigma_{T - t_k} \log \tilde \sigma_{T - t_k}^{-1} \big) }{\sigma_{T -t_k}^2}  \\
    & \qquad +  \frac{(\gamma - 1)\abs{\cI} e^{ -{ d_0^2} /  (24  \tilde{\sigma}_{T - t}^2 )} (\dist(\hat z_{t_k}, K_{\eta_0}) + D_0) }{\sigma_{T - t_k}^2w_{\eta_0} \inf_{x \in \R^d} p_{\eta_0}(x) \cdot \rho_{\eta_0} (d_0 / 6)^{d}} 
          \Big)^2.
\end{split}
\end{align}
As for $\mathsf{F}_2(t)$, by the upper bound in \cref{eq:F2-three-lines}, there exists $\bar \tau_{1, \varepsilon} > 0$ that depends only on $(p_0, \varepsilon,\gamma)$, such that for all $t_k \in (\bar \tau_{1, \varepsilon}, T)$ and $t \in [t_k, t_{k + 1})$,  
\begin{align}
\label{eq:F2(t)}
    \mathsf{F}_2(t) \leq -\frac{1}{4\sigma_{T - t_k}^2} \dist (\hat z_{t_k}, K_{\eta_0})^2 + \dist (\hat z_{t_k}, K_{\eta_0}) \big( \dist (\hat z_{t_k}, K_{\eta_0}) + R_0\big). 
\end{align}
Putting together \cref{eq:F1(t),eq:F2(t)}, we conclude that there exists $\tau_{1, \varepsilon} > 0$ that depends only on $(p_0, \varepsilon,\gamma)$, such that for all $t_k \in (\tau_{1, \eps}, T)$ and $t \in [t_k, t_{k + 1})$, 
\begin{align*}
    \frac{1}{2} \frac{\dd }{\dd t} \norm{\hat z_t - \Proj_{K_{\eta_0}}(\hat z_{t_k})}^2 \leq -\frac{1}{5\sigma_{T - t_k}^2} \dist (\hat z_{t_k}, K_{\eta_0})^2. 
\end{align*}

\subsubsection*{Case II: $\dist(\hat z_{t_k}, K_{\eta_0}) \in[d_0 / 3, \infty)$}

This part follows immediately from Lemma \ref{lemma:discretized_ddim}. 
We complete the proof by putting together results in these two cases.

\subsection{Proof of Lemma \ref{lemma:discretized_ddim3}}
\label{proof:discretized_ddim3}

Recall that $d_0 = \inf_{y \in K_{\eta},\, y' \in K_{\eta'}, \, \eta \neq \eta'} \norm{y - y'}$, $R_0 = \sup_{y \in \cup_{\eta \in \cI} K_{\eta}} \norm{y}$, $D_0 = \sup_{y, y' \in \cup_{\eta \in \cI} K_{\eta}}\norm{y - y'}$ and $\tilde \sigma_{T - t}^2 = \lambda_{T - t}^{-2} \sigma_{T - t}^2$. 
We choose $\eps_{\delta} < d_0 / 3$.
Note that over $[t_k, t_{k + 1})$, by \cref{eq:d-zt-projk} we have 
\begin{align*}
	& \Big|\frac{\dd }{\dd t} \norm{\hat z_t - \Proj_{K_{\eta_0}}(\hat z_{t_k})}\Big| \\
	& = \left|e^{t - t_k} \left\langle \frac{\hat z_{t} -  \Proj_{K_{\eta_0}}(\hat z_{t_k})}{ \norm{\hat z_t - \Proj_{K_{\eta_0}}(\hat z_{t_k})}}, \, \hat z_{t_k} + \frac{\gamma}{\sigma_{T - t_k}^2} F_{\eta_0, T - t_k}(\hat z_{t_k}) - \frac{\gamma - 1}{\sigma_{T - t_k}^2} \sum_{\eta \in \cI} \zeta_{\eta, T - t_k}(\hat z_{t_k}) F_{\eta, T - t_k}(\hat z_{t_k})\right\rangle  \right| \\
	& \leq e^{t - t_k} \Big( \norm{\hat z_{t_k}} + \frac{\gamma}{\sigma_{T - t_k}^2} \norm{F_{\eta_0, T - t_k}(\hat z_{t_k})} + \frac{\gamma - 1}{\sigma_{T - t_k}^2}\sum_{\eta \neq \eta_0} \zeta_{\eta, T - t_k}(\hat z_{t_k}) \norm{F_{\eta, T - t_k}(\hat z_{t_k})}  \Big) \\
	& \leq e^{t - t_k} \Big( \dist(\hat z_{t_k}, K_{\eta_0}) + R_0 + \frac{\gamma\big(\dist(\hat z_{t_k}, K_{\eta_0}) + C_{K_{\eta_0}} \tilde \sigma_{T - t_k} \log \tilde \sigma_{T - t_k}^{-1} \big) }{\sigma_{T -t_k}^2} \\
	&\qquad \qquad  + \frac{\gamma - 1}{\sigma_{T - t_k}^2}\sum_{\eta \neq \eta_0} \zeta_{\eta, T - t_k}(\hat z_{t_k}) \big( \dist(\hat z_{t_k}, K_{\eta_0}) + D_0 \big)  \Big), 
\end{align*}
where the last inequality follows from triangle inequality and \cref{eq:F-upper-dist-Proj-m}. By assumption we know that $\dist(\hat z_{t_k}, K_{\eta_0}) \leq \eps_\delta < d_0 / 3$. 
In addition, similar to the derivation of \cref{eq:sum-zeta-upper-2} we know that 
\begin{align*}
	\sum_{\eta \neq \eta_0} \zeta_{\eta, T - t_k} (\hat z_{t_k}) \leq \frac{\abs{\cI}}{w_{\eta_0} \inf_{x \in \R^d} p_{\eta_0}(x) \cdot \rho_{\eta_0} (d_0 / 6)^{d}} \cdot
        \exp \left( -\frac{ d_0^2}{24  \tilde{\sigma}_{T - t_k}^2 }\right).
\end{align*}
Leveraging these upper bounds, we further conclude that 
\begin{align*}
	 & \Big|\frac{\dd }{\dd t} \norm{\hat z_t - \Proj_{K_{\eta_0}}(\hat z_{t_k})}\Big| \\
	 & \leq e^{t - t_k} \Big( \eps_{\delta} + R_0 + \frac{\gamma\big(\eps_{\delta} + C_{K_{\eta_0}} \tilde \sigma_{T - t_k} \log \tilde \sigma_{T - t_k}^{-1} \big) }{\sigma_{T -t_k}^2} + \frac{(\gamma - 1)\abs{\cI}( \eps_{\delta} + D_0 ) e^{ -{ d_0^2} / (24  \tilde{\sigma}_{T - t_k}^2 )} }{\sigma_{T - t_k}^2w_{\eta_0} \inf_{x \in \R^d} p_{\eta_0}(x) \cdot \rho_{\eta_0} (d_0 / 6)^{d}} 
         \Big). 
\end{align*}
Under Assumption \ref{assumption:discretization}, we conclude that for any $\kappa_0 \in (0, 1)$, there exist $\bar \tau_\delta > 0$ and $\eps_\delta \in (0, \delta)$ that depend only on $(\delta, p_0, \gamma)$, such that whenever $t_k \in [\bar\tau_\delta, T)$, 
\begin{align*}
	& \dist(\hat z_t, K_{\eta_0}) - \dist(\hat z_{t_k}, K_{\eta_0}) \leq \int_{t_k}^t \Big|\frac{\dd }{\dd s} \norm{\hat z_s - \Proj_{K_{\eta_0}}(\hat z_{t_k})}\Big| \dd s \\
	& \leq (e^{t - t_k} - 1) \Big( \eps_{\delta} + R_0 + \frac{\gamma\big(\eps_{\delta} + C_{K_{\eta_0}} \tilde \sigma_{T - t_k} \log \tilde \sigma_{T - t_k}^{-1} \big) }{\sigma_{T -t_k}^2} + \frac{(\gamma - 1)\abs{\cI}( \eps_{\delta} + D_0 ) e^{ -{ d_0^2} / (24  \tilde{\sigma}_{T - t_k}^2 )} }{\sigma_{T - t_k}^2w_{\eta_0} \inf_{x \in \R^d} p_{\eta_0}(x) \cdot \rho_{\eta_0} (d_0 / 6)^{d}} 
         \Big) \\
         & \leq \delta. 
\end{align*}

The proof is done.

\section{Proofs for guided DDPM}

\subsection{Proof of Theorem \ref{thm:guided_sde_qualitative}}
\label{proof:thm:guided_sde_qualitative}

To address technical issues in applying It\^{o}'s lemma, we first establish that the trajectory $\{{z}_t\}_{t = 0}^T$ defined in Eq.~\eqref{eqn:guided_sde} eventually enters an $\epsilon$-neighborhood of a smooth approximation to $K_{\eta_0}$.

\begin{thm}
\label{thm:guided_sde_qualitative_smooth}
  Under the assumptions of Theorem \ref{thm:guided_sde_qualitative}, 
  assume that $\tilde{K}_{\eta_0} \subseteq \R^d$ is a convex set with smooth boundary and $\dH(\tilde{K}_{\eta_0}, K_{\eta_0}) \le \eps$, 
  where $\dH(\cdot, \cdot)$ denotes the Hausdorff distance (see Definition \ref{def:hausdorff}).
  Then, almost surely,
  \begin{equation*}
        \limsup_{t \to T^-} \dist(z_t, \tilde{K}_{\eta_0}) \le (4 \gamma + 6) \varepsilon.
  \end{equation*}
\end{thm}
\begin{proof}[Proof of Theorem \ref{thm:guided_sde_qualitative_smooth}]
	We prove Theorem \ref{thm:guided_sde_qualitative_smooth} in Appendix \ref{proof:thm:guided_sde_qualitative_smooth}. 
\end{proof}

Fix any $0 < \eps < 1$.
By the main theorem in \cite{firey1974approximating}, there exists a smooth, convex, compact domain $\tilde{K}_{\eta_0}$ such that $\dH(\tilde{K}_{\eta_0}, K_{\eta_0}) < \eps$.
Applying \cref{thm:guided_sde_qualitative_smooth}, we then have that
 \begin{align*}
    \limsup_{t\to T} \dist({z}_t, K_{\eta_0})
    \le \limsup_{t \to T} \dist({z}_t, \tilde{K}_{\eta_0}) + \dH(K_{\eta_0}, \tilde{K}_{\eta_0})\le (4 \gamma + 7) \eps.
 \end{align*}
 Since $\eps > 0$ is arbitrary, the conclusion follows.

\subsection{Proof of Theorem \ref{thm:discretized-ddpm}}
\label{proof:thm:discretized-ddpm}

Note that SDE \eqref{eq:discretized-SDE} can be reformulated as follows: for $t \in [t_k, t_{k + 1})$, 
\begin{align*}
    & \dd \hat x_t - \hat x_t \D t -  \sqrt{2}\D B_t = \\
    & \Big\{\,
    \frac{2 \gamma \lambda_{T - t_k}}{\sigma_{T - t_k}^2} \Big(  \bar{m}_{\eta_0, T  - t_k} (\hat x_{t_k}) 
    -  \frac{\hat x_{t_k}}{\lambda_{T - t_k}} \Big)  
    - \frac{ 2(\gamma - 1) \lambda_{T - t_k}}{\sigma_{T - t_k}^2} \sum_{\eta \in \cI} \bar{\zeta}_{\eta, T - t_k}(\hat x_{t_k}) \Big( \bar{m}_{\eta, T - t_k}(\hat x_{t_k}) - \frac{\hat x_{t_k}}{\lambda_{T - t_k}}  \Big) \Big\} \dd t. 
\end{align*}
Define $\hat z_t = \lambda_{T - t_k}^{-1} \hat x_t$ for $t \in [t_k, t_{k + 1})$. 
Then for $t \in [t_k, t_{k + 1})$, we have 
\begin{align}
\label{eqn:change_of_var_guided_sde_discretized}
\begin{split}
    &\dd \hat z_t -  \hat{z}_t \D t - \sqrt{2} e^{T - t_k} \D B_t \\
    &= 
    \Big\{\frac{2 \gamma}{\sigma_{T - t_k}^2} \big( m_{\eta_0, T - t_k} ({\hat z_{t_k}}) - {\hat z_{t_k}} \big) - \frac{ 2(\gamma - 1)}{\sigma_{T - t_k}^2} \sum_{\eta \in \cI} {\zeta}_{\eta, T - t_k} (\hat z_{t_k}) 
  \big( m_{\eta, T - t_k}(\hat z_{t_k}) - \hat z_{t_k} \big)   \Big\}\D t. 
\end{split}
\end{align}
The above equation further implies that for $t \in [t_k, t_{k + 1})$, 
\begin{align}
\label{eq:ddpm-zt-ztk}
\begin{split}
	& \hat z_t - \hat z_{t_k} - \sqrt{2} e^{t - t_k}\int_{t_k}^t e^{T - s} \dd B_s  \\
	&= (e^{t - t_k} - 1) \Big( \hat z_{t_k} + \frac{2\gamma}{\sigma_{T - t_k}^2} \left( m_{\eta_0, T - t_k}(\hat z_{t_k}) - \hat z_{t_k} \right) - \frac{2(\gamma - 1)}{\sigma_{T - t_k}^2} \sum_{\eta \in \cI} \zeta_{\eta, T - t_k}(\hat z_{t_k}) \left( m_{\eta, T - t_k}(\hat z_{t_k}) - \hat z_{t_k} \right) \Big).
\end{split}
\end{align}
For notational simplicity, we define 
\begin{align}
\label{eq:def-vk}
	v_k = \hat z_{t_k} + \frac{2\gamma}{\sigma_{T - t_k}^2} \left( m_{\eta_0, T - t_k}(\hat z_{t_k}) - \hat z_{t_k} \right) - \frac{2(\gamma - 1)}{\sigma_{T - t_k}^2} \sum_{\eta \in \cI} \zeta_{\eta, T - t_k}(\hat z_{t_k}) \left( m_{\eta, T - t_k}(\hat z_{t_k}) - \hat z_{t_k} \right). 
\end{align}
Substituting this notation into \cref{eq:ddpm-zt-ztk,eqn:change_of_var_guided_sde_discretized}, we see that 
\begin{align}
\label{eq:52-53}
\begin{split}
	& \hat z_t - \hat z_{t_k} = \sqrt{2} e^{t - t_k} \int_{t_k}^t e^{T - s} \dd B_s + (e^{t - t_k} - 1) v_k, \\
	& \dd \hat z_t = \sqrt{2} e^{T - t_k} \dd B_t + e^{t - t_k} v_k \dd t  +  \sqrt{2} e^{t - t_k} \int_{t_k}^t e^{T - s} \dd B_s   \dd t. 
\end{split}
\end{align}
Applying It\^o's formula to the differential of $\norm{\hat z_t - \Proj_{K_{\eta_0}}(\hat z_{t_k})}^2$ and using \cref{eq:52-53}, we obtain
\begin{align}
\label{eq:53}
\begin{split}
	& \dd \|\hat z_t - \Proj_{K_{\eta_0}}(\hat z_{t_k})\|^2 = 2\, \Big\langle \hat z_t - \Proj_{K_{\eta_0}}(\hat z_{t_k}), \, \dd \hat z_t \Big \rangle + 2d e^{2T - 2t_k } \dd t  \\
	& = 2\, \Big\langle \hat z_t - \hat z_{t_k}, \, \dd \hat z_t \Big \rangle + 2\, \Big\langle \hat z_{t_k} - \Proj_{K_{\eta_0}}(\hat z_{t_k}), \, \dd \hat z_t \Big \rangle + 2d e^{2T - 2t_k } \dd t \\
	& = 2\, \left\langle \sqrt{2} e^{t - t_k} \int_{t_k}^t e^{T - s} \dd B_s + (e^{t - t_k} - 1) v_k, \,  \sqrt{2} e^{T - t_k} \dd B_t + e^{t - t_k} v_k \dd t  +  \sqrt{2} e^{t - t_k} \int_{t_k}^t e^{T - s} \dd B_s   \dd t \right \rangle   \\
	& \,\,\,\,\,\,\, + 2d e^{2T - 2t_k } \dd t + 
    2\, \left\langle \hat z_{t_k} - \Proj_{K_{\eta_0}}(\hat z_{t_k}), \, \sqrt{2} e^{T - t_k} \dd B_t + e^{t - t_k} v_k \dd t  +  \sqrt{2} e^{t - t_k} \int_{t_k}^t e^{T - s} \dd B_s   \dd t \right \rangle. 
\end{split} 
\end{align}
We next analyze the above terms in different phases depending on the value of $\|\hat z_t - \Proj_{K_{\eta_0}}(\hat z_{t_k})\|$. 
\begin{lem}
\label{lemma:discretized-ddpm-1}
Define $d_0 = \inf_{y \in K_{\eta},\, y' \in K_{\eta'}, \, \eta \neq \eta'} \|y - y'\|$.
Under the assumptions of Theorem \ref{thm:discretized-ddpm}, 
with a sufficiently small $\kappa_0 > 0$ depending only on $(p_0, \gamma, T)$, there exists $0 < \tau_0 < T$ that also depends only on $(p_0, \gamma, T)$, 
such that for all $\hat z_{t_k}$ that satisfies $\dist(\hat z_{t_k}, K_{\eta_0}) \geq  d_0 / 3$ and $t_k \in [\tau_0, T)$, for all $t \in [t_k, t_{k + 1})$ we have
\begin{align*}
	\frac{\dd}{\dd t} \E \big[ \|\hat z_t - \Proj_{K_{\eta_0}}(\hat z_{t_k})\|^2 \mid \hat z_{t_k} \big] \leq - \frac{1}{5 \sigma_{T - t_k}^2} \dist\big( \hat z_{t_k}, K_{\eta_0} \big)^2. 
\end{align*}	
\end{lem}
\begin{proof}[Proof of Lemma \ref{lemma:discretized-ddpm-1}]

We prove Lemma \ref{lemma:discretized-ddpm-1} in Appendix \ref{proof:lemma:discretized-ddpm-1}.  
	
\end{proof}

\cref{lemma:discretized-ddpm-1} shows that $\E [ \|\hat z_t - \Proj_{K_{\eta_0}}(\hat z_{t_k})\|^2 \mid \hat z_{t_k} ]$ decreases on the interval $[t_k, t_{k + 1})$, provided that $\kappa \in (0, \kappa_0]$, $t_k \in [\tau_0, T)$ and $\dist(\hat z_{t_k}, K_{\eta_0}) \geq  d_0 / 3$.
Since $\norm{\hat z_{t_{k + 1}} - \Proj_{K_{\eta_0}}(\hat z_{t_k})} \geq \norm{\hat z_{t_{k + 1}} - \Proj_{K_{\eta_0}}(\hat z_{t_{k + 1}})}$, 
we conclude that as long as $\dist(\hat z_{t_k}, K_{\eta_0}) \geq  d_0 / 3$, 
\begin{align}\label{eqn:expected_dist_upper_bd} 
    \int_{t_{k}}^{t_{k + 1}} \frac{\dd}{\dd t} \E \big[ \norm{\hat z_t - \Proj_{K_{\eta_0}}(\hat z_{t_k})}^2 \mid \hat z_{t_k} \big] \, \dd t
    \le 
     -\frac{t_{k + 1} - t_k}{5 \sigma_{T - t_k}^2} \cdot \dist\big( \hat z_{t_k}, K_{\eta_0} \big)^2 \notag \\
    \implies
	\E \big[ \dist(\hat z_{t_{k + 1}}, K_{\eta_0})^2 \mid \hat z_{t_k} \big] \leq 
    \Big( 1 - \frac{t_{k + 1} - t_k}{5 \sigma_{T - t_k}^2} \Big) \cdot \dist\big( \hat z_{t_k}, K_{\eta_0} \big)^2. 
\end{align}
Define $\tilde \tau = \inf\{k: \dist(\hat z_{t_k}, K_{\eta_0}) < d_0 / 3\}$.
Consider a sequence of random variables $\{Y_k\}_{k \ge 0}$ defined by 
\begin{align*}
	Y_{k + 1} = \left\{  \begin{array}{ll}
		\dist(\hat z_{t_{k + 1}}, K_{\eta_0})^2, & k <  \tilde \tau, \\
		\Big( 1 - \frac{t_{k + 1} - t_{k}}{5 \sigma_{T - t_{k}}^2} \Big) \cdot  Y_{k} & k \geq \tilde \tau. 
	\end{array} \right.
\end{align*}
By the definition of $\left\{ Y_k \right\}_{k \ge 1}$ and \cref{eqn:expected_dist_upper_bd},
for all $k \in \mathbb{N}$, 
\begin{align}
\label{eq:Yk-Yk+1}
	\E\big[Y_{k + 1} \mid Y_k \big] \leq \Big( 1 - \frac{t_{k + 1} - t_k}{5 \sigma_{T - t_k}^2} \Big) \, Y_k. 
\end{align} 
%
Now, \cref{assumption:discretization} implies that $\Delta_k \le \frac{\kappa}{1 + \kappa} (T - t_k)$ and $1/\sigma^2_{T - t_k} \ge 1/[2(T - t_k)]$ for all $k$ large enough. 
Consequently, the series $\sum_{k\ge 1} \Delta_k / \sigma^2_{T - t_k} = \infty$.
Hence, by a variant of the Robbins-Siegmund Theorem \citep[Lemma 10, Chapter 2]{polyak2021introduction}, we see that $Y_k$ converges almost surely to $0$.
As a consequence, almost surely there exists $k_1 \in \mathbb{N}_+$, such that $\dist(\hat z_{t_{k_1}}, K_{\eta_0}) < d_0 / 3$. 

\begin{lem}
\label{lemma:discretized-ddpm-2}
	Under the assumptions of Theorem \ref{thm:discretized-ddpm},
with a sufficiently small $\kappa_0 > 0$ that depends only on $(p_0, \gamma, T)$,
for any $\varepsilon \in (0, \min\{1, d_0 / 3\})$, there exists $0 < \tau_{\varepsilon} < T$ that depends solely on $(p_0, \varepsilon,\gamma, T)$, such that for all $\hat z_{t_k}$ that satisfies $d_0 / 3 > \dist(\hat z_{t_k}, K_{\eta_0}) \geq  \varepsilon$ and $t_k \in [\tau_\varepsilon, T)$, for all $t \in [t_k, t_{k + 1})$ we have
\begin{align*}
	\frac{\dd}{\dd t} \E \big[ \norm{\hat z_t - \Proj_{K_{\eta_0}}(\hat z_{t_k})}^2 \mid \hat z_{t_k} \big] \leq - \frac{1}{5 \sigma_{T - t_k}^2} \dist\big( \hat z_{t_k}, K_{\eta_0} \big)^2. 
\end{align*}
\end{lem}
\begin{proof}[Proof of Lemma \ref{lemma:discretized-ddpm-2}]
	We prove Lemma \ref{lemma:discretized-ddpm-2} in Appendix \ref{proof:lemma:discretized-ddpm-2}. 
\end{proof}


We complete the proof with the next lemma.  
\begin{lem}
\label{lemma:C3}
	Under the assumptions of Theorem \ref{thm:discretized-ddpm},
with a sufficiently small $\kappa_0 > 0$ that depends only on $(p_0, \gamma, T)$, almost surely 
\begin{align*}
	\limsup_{k \to \infty} \,\dist \big( \hat z_{t_k}, K_{\eta_0} \big) \leq \eps. 
\end{align*}
\end{lem}
\begin{proof}[Proof of Lemma \ref{lemma:C3}]
	We prove Lemma \ref{lemma:C3} in Appendix \ref{proof:lemma:C3}. 
\end{proof}

Theorem \ref{thm:discretized-ddpm} follows immediately from Lemma \ref{lemma:C3}.

\subsection{Proof of Theorem \ref{thm:guided_sde_qualitative_smooth}}
\label{proof:thm:guided_sde_qualitative_smooth}

For convenience, we introduce the following notation:
\begin{align*}
   \dd z_t
  & = \underbrace{\Big\{\frac{2 \gamma}{\sigma_{T - t}^2} F_{\eta_0, T - t}( z_t) - \frac{ 2(\gamma - 1)}{\sigma_{T - t}^2} \sum_{\eta \in \cI} {\zeta}_{\eta, T - t} (z_t) 
  F_{\eta, T - t}(z_t)   \Big\}}_{ b(t,  z_t)} \dd t
   + \sqrt{2} e^{T - t} \dd 
   B_t. 
\end{align*}

In the next lemma, we express differences of distances in terms of integrals.

\begin{lem}
\label{lem:ito_formula_dist}
Under the assumptions of Theorem \ref{thm:guided_sde_qualitative_smooth}, let $({z}_t)_{0 \le t < T}$ be the process defined in Eq.~\eqref{eqn:guided_sde}.
    Then for any ${z}_t \not \in \tilde{K}_{\eta_0}$, we have 
    \begin{align*}
    \dd \big(\dist(z_t, \tilde{K}_{\eta_0})\big) = f(t, {z}_t) \D t + e^{2(T - t)}  \D {W}_t,
    \end{align*}
    where 
    \begin{align*}
        f(t, {z}_t) 
        &= \left\langle \nabla \dist({z}_t, \tilde{K}_{\eta_0}),\, b(t, {z}_t) \right\rangle + e^{2(T-t)} \Delta \dist({z}_t, \tilde{K}_{\eta_0}),
    \end{align*}
    and $({W}_t)_{t \geq 0}$ denotes a standard one-dimensional Brownian motion. 
\end{lem}

\begin{proof}[Proof of Lemma \ref{lem:ito_formula_dist}]
	We prove Lemma \ref{lem:ito_formula_dist} in Appendix \ref{proof:lem:ito_formula_dist}. 
\end{proof}

We then show that the distance to $\tilde K_{\eta_0}$ decreases to a certain level. 

\begin{lem}\label{lem:hitting_time}
  Under the assumptions of Theorem \ref{thm:guided_sde_qualitative_smooth}, there exists $t_{\varepsilon} > 0$ that depends only on $(p_0, \varepsilon, \gamma)$, such that for all $t_\varepsilon \leq \tau < T$, if $\dist(z_{\tau}, \tilde K_{\eta_0}) > (2\gamma + 3) \varepsilon$, then 
%
   \begin{equation*}
    \P \left(\exists \, \tau \le s < T \text{ such that } \dist(z_s, \tilde{K}_{\eta_0}) = (2\gamma + 3) \varepsilon \,\Big|\, z_{\tau} \right) = 1.
   \end{equation*}
\end{lem}
\begin{proof}[Proof of Lemma \ref{lem:hitting_time}]
	We prove Lemma \ref{lem:hitting_time} in Appendix \ref{proof:lem:hitting_time}. 
\end{proof}

Next, we show that the distance increases to a certain level only finitely many times. 

\begin{lem}
\label{lem:upcrossings}
   Let $t_\varepsilon$ be as defined in Lemma~\ref{lem:hitting_time}, and $a = (2\gamma + 3) \varepsilon$.
   Define $U_{[t_{\varepsilon}, T)}$ as the number of upcrossings of $\dist(z_t, \tilde{K}_{\eta_0})$ from $a$ to $2a$ over the interval $[t_{\varepsilon}, T)$.
   Then almost surely $U_{[t_{\varepsilon}, T)} < \infty$. 
\end{lem}
\begin{proof}[Proof of Lemma \ref{lem:upcrossings}]
	We prove Lemma \ref{lem:upcrossings} in Appendix \ref{proof:lem:upcrossings}. 
\end{proof}

With the above two lemmas, we now prove Theorem \ref{thm:guided_sde_qualitative_smooth}. 
For concreteness, let $(\Omega, \mathcal{F}, \P)$ denote the probability space on which the process $(\dist(z_t, \tilde{K}_{\eta_0}))_{0 \le t < T}$ is defined. 
For instance, one may take the classical Wiener space equipped with the filtration generated by the coordinate process.
Fix $a = (2\gamma + 3) \varepsilon$, and let $t_{\varepsilon}$ be from Lemma \ref{lem:hitting_time}.
Define the sets
\begin{align*}
   A &= \left\{ \omega \in \Omega: \exists\, s \in [t_\varepsilon, T) \text{ such that } \dist(z_s (\omega), \tilde{K}_{\eta_0}) > 2 a \right\} \\
   G &= \left\{ \omega \in \Omega: \limsup_{t \to T}  \dist(z_t (\omega), \tilde{K}_{\eta_0}) \le 2 a \right\}. 
\end{align*}
We next prove that $\P(G) = 1$.
Note that
\begin{equation*}
   \P(G)
   =  \P \left( G \mid A \right) \cdot \P(A) + \P \left( G \mid A^c \right) \cdot \P(A^c)
   = \P \left( G \mid A \right) \cdot \P(A) + \P(A^c).
\end{equation*}
Hence, it suffices to show $\P \left( G \mid A \right) = 1$. 
Define
\begin{equation*}
   T_\mathsf{Last} 
   = \inf \left\{ m \ge 1: U_{I_n} = 0, \, \forall \, n \ge m \right\},
\end{equation*} 
where $(I_m)_{m \geq 1}$ are the intervals defined in the proof of Lemma~\ref{lem:upcrossings}, and $U_{I_m}$ denotes the number of upcrossings of $\dist(z_t, \tilde{K}_{\eta_0})$ from $a$ to $2a$ within the interval $I_m$.
Recall that $\mathsf{Hit}(a)=\inf \{ t \ge t_\varepsilon: \dist({z}_t, \tilde{K}_{\eta_0}) = a \}$ is defined in the proof of Lemma \ref{lem:upcrossings}. 
Note that
\begin{align*}
    G  &\supseteq
    \left\{ \,\mathsf{Hit}(a) < T, \, \exists \tau \in [\mathsf{Hit}(a),\, T): \dist(z_s, \tilde{K}_{\eta_0}) \le 2a, \, \forall\, s \ge \tau \right\}  \\
    &\supseteq \left\{ T_\mathsf{Last}  < \infty, \, \mathsf{Hit}(a) < T \right\} \\
    &\supseteq \left\{ U_{[\mathsf{Hit}(a), T)} < \infty, \, \mathsf{Hit}(a) < T \right\}.
\end{align*}
By Lemma~\ref{lem:hitting_time} and Lemma~\ref{lem:upcrossings}, conditioning on $A$, we have
\begin{equation*}
    \P \left( \mathsf{Hit}(a) < T, \, U_{[\mathsf{Hit}(a), T)} < \infty \mid A \right) = 1.
\end{equation*} 
 Therefore, $\P \left( G \mid A \right) = 1$ as desired.
 The proof is done.

\subsection{Proof of Lemma \ref{lemma:discretized-ddpm-1}}
\label{proof:lemma:discretized-ddpm-1}

Recall that $d_0 = \inf_{y \in K_{\eta},\, y' \in K_{\eta'}, \, \eta \neq \eta'} \norm{y - y'}$, $D_0 = \sup_{y, y' \in \cup_{\eta \in \cI} K_{\eta}}\norm{y - y'}$ and $R_0 = \sup_{y \in \cup_{\eta \in \cI} K_{\eta}} \norm{y}$.
Similar to the derivation of \cref{eqn:dist_deriv_final_bound}, we see that
\begin{align*}
	& \frac{1}{\sigma_{T - t_k}^2}\left\langle \frac{\hat z_{t_k} - \mathrm{Proj}_{K_{\eta_0}}(\hat z_{t_k})}{\norm{\hat z_{t_k} - \mathrm{Proj}_{K_{\eta_0}}(\hat z_{t_k})}}, \,\gamma F_{\eta_0, T - t_k}(\hat z_{t_k}) - (\gamma - 1) \sum_{\eta \in \cI} \zeta_{\eta, T - t_k}(\hat z_{t_k}) F_{\eta, T - t_k}(\hat z_{t_k})\right\rangle \\
	& \leq \frac{D_0(\gamma - 1) \abs{\cI}}{w_{\eta_0} \inf_{x \in \R^d} p_{\eta_0}(x) \cdot \rho_{\eta_0} \Xi_{T - t_k}^{d}\sigma^2_{T - t_k}}  \cdot
        \exp \left( -\frac{3 \Xi_{T - t_k}^2}{2 \tilde{\sigma}_{T - t_k}^2 }\right) + \frac{2(\gamma - 1) \Xi_{T - t_k}}{\sigma_{T - t_k}^2} \\
        & \,\,\,\,\,\,  - \frac{1}{\sigma_{T - t_k}^2} \dist(\hat z_{t_k}, K_{\eta_0}) + \frac{2\gamma \sup_{\eta \in \cI} C_{K_{\eta}}\Xi_{T - t_k}^2 \log (\Xi_{T - t_k}^{-2})}{\sigma_{T - t_k}^2}.
\end{align*} 
Recall that $v_k$ is defined in \cref{eq:def-vk}.
Using the above upper bound and the fact $\norm{\hat{z}_{t_k}} \le R_0$, we have 
\begin{align*}
	& \Big\langle \hat z_{t_k} - \Proj_{K_{\eta_0}}(\hat z_{t_k}), \, v_k  \Big\rangle \\
	& \leq \dist\big( \hat z_{t_k}, K_{\eta_0} \big) \Big(\dist\big( \hat z_{t_k}, K_{\eta_0} \big) + R_0 \Big) + \frac{2\dist\big( \hat z_{t_k}, K_{\eta_0} \big)D_0(\gamma - 1) \abs{\cI}}{w_{\eta_0} \inf_{x \in \R^d} p_{\eta_0}(x) \cdot \rho_{\eta_0} \Xi_{T - t_k}^{d}\sigma^2_{T - t_k}}  \cdot
        \exp \left( -\frac{3 \Xi_{T - t_k}^2}{2 \tilde{\sigma}_{T - t_k}^2 }\right) \\
        &\,\,\,\,\, + \frac{4\dist\big( \hat z_{t_k}, K_{\eta_0} \big)(\gamma - 1) \Xi_{T - t_k}}{\sigma_{T - t_k}^2}   - \frac{2}{\sigma_{T - t_k}^2} \dist(\hat z_{t_k}, K_{\eta_0})^2 + \frac{4\dist\big( \hat z_{t_k}, K_{\eta_0} \big)\gamma \sup_{\eta \in \cI} C_{K_{\eta}}\Xi_{T - t_k}^2 \log (\Xi_{T - t_k}^{-2})}{\sigma_{T - t_k}^2}.
\end{align*}
Therefore, there exist $\tau_1, \kappa_1 > 0$ that depend only on $(p_0, \gamma)$, such that for all $t_k \in [\tau_1, T)$, $t \in [t_k, t_{k + 1})$, $\kappa \in (0, \kappa_1]$ and $\dist( \hat z_{t_k}, K_{\eta_0} ) \geq d_0 / 3$, we have 
\begin{align}
\label{eq:C1-1}
\begin{split}
	& \Big\langle \hat z_{t_k} - \Proj_{K_{\eta_0}}(\hat z_{t_k}), \, v_k  \Big\rangle \leq  - \frac{1}{4\sigma_{T - t_k}^2} \dist\big( \hat z_{t_k}, K_{\eta_0} \big)^2. 
\end{split}
\end{align}	
%
Similar to the derivation of \cref{eq:F1-upper}, we conclude that 
\begin{align*}
	& (e^{t - t_k} - 1) \|v_k\|^2 \\
	& =  (e^{t - t_k} - 1) \cdot \Big\| \hat z_{t_k} + \frac{2\gamma}{\sigma_{T - t_k}^2} \left( m_{\eta_0, T - t_k}(\hat z_{t_k}) - \hat z_{t_k} \right) - \frac{2(\gamma - 1)}{\sigma_{T - t_k}^2} \sum_{\eta \in \cI} \zeta_{\eta, T - t_k}(\hat z_{t_k}) \left( m_{\eta, T - t_k}(\hat z_{t_k}) - \hat z_{t_k} \right) \Big\|^2  \\
	&\leq \, (e^{t - t_k} - 1) \cdot \Big( \|\hat z_{t_k}\| + \frac{2\gamma - 2}{\sigma_{T - t_k}^2} \sum_{\eta \in \cI} \zeta_{\eta, T - t_k} (\hat z_{t_k}) \big\| m_{\eta_0, T - t_k} (\hat z_{t_k}) - m_{\eta, T - t_k}(\hat z_{t_k}) \big\| \\
    & \quad + \frac{2}{\sigma_{T - t_k}^2} \dist(\hat z_{t_k}, K_{\eta_0}) + \frac{2}{\sigma_{T - t_k}^2} \|\Proj_{K_{\eta_0}}(\hat z_{t_k}) - m_{\eta_0, T - t_k} (\hat z_{t_k})\| \Big)^2 \\
    &\leq  \, (e^{t - t_k} - 1) \Big(\frac{\sigma_{T - t_k}^2 + 2}{\sigma_{T - t_k}^2} \dist(\hat z_{t_k}, K_{\eta_0}) + R_0 + \frac{2\gamma D_0}{\sigma_{T - t_k}^2}  \Big)^2 \\
    &\leq  \frac{3(e^{t - t_k} - 1)(\sigma_{T - t_k}^2 + 2)^2\, \dist(\hat z_{t_k}, K_{\eta_0})^2}{\sigma_{T - t_k}^4} + 3 (e^{t - t_k} - 1) R_0^2 + \frac{12(e^{t - t_k} - 1)\gamma^2 D_0^2}{\sigma_{T - t_k}^4}.
\end{align*}
Again, Assumption \ref{assumption:discretization} implies that $(t - t_k) / \sigma_{T - t_k}^2 \lesssim \kappa$.
Hence, there exist $\tau_2, \kappa_2 > 0$ that depend only on $(p_0, \gamma)$, 
such that for all $t_k \in [\tau_2, T)$, 
$t \in [t_k, t_{k + 1})$, $\kappa \in (0, \kappa_2]$ and $\dist( \hat z_{t_k}, K_{\eta_0} ) \geq d_0 / 3$ 
\begin{align}
\label{eq:C1-2}
	(e^{t - t_k} - 1) \norm{v_k}^2 \leq \frac{1}{20 \sigma_{T - t_k}^2} \dist(\hat z_{t_k}, K_{\eta_0})^2. 
\end{align}
Substituting \cref{eq:C1-1,eq:C1-2} into \cref{eq:53}, we conclude that for all $t_k \in [\tau_1 \vee \tau_2, T)$, $t \in [t_k, t_{k + 1})$, $\kappa \in (0, \kappa_1 \wedge\kappa_2]$, and $\dist( \hat z_{t_k}, K_{\eta_0} ) \geq d_0 / 3$, 
\begin{align*}
	 \dd \E \big[ \|\hat z_t - \Proj_{K_{\eta_0}}(\hat z_{t_k})\|^2 \mid \hat z_{t_k} \big]  \leq - \frac{2}{5 \sigma_{T - t_k}^2} \dist\big( \hat z_{t_k}, K_{\eta_0} \big)^2 \dd t + 2de^{2T + 2t - 4t_k}  \dd t. 
\end{align*}
We can then conclude that there exist $\tau_0, \kappa_0 > 0$ that depend only on $(p_0, \gamma, T)$, such that for all $t_k \in [\tau_0, T)$, $t \in [t_k, t_{k + 1})$, $\kappa \in (0, \kappa_0]$ and $\dist( \hat z_{t_k}, K_{\eta_0} ) \geq d_0 / 3$,
\begin{align*}
	\frac{\dd}{\dd t} \E \big[ \|\hat z_t - \Proj_{K_{\eta_0}}(\hat z_{t_k})\|^2 \mid \hat z_{t_k} \big] \leq - \frac{1}{5 \sigma_{T - t_k}^2} \dist\big( \hat z_{t_k}, K_{\eta_0} \big)^2,
\end{align*}
completing the proof.

\subsection{Proof of Lemma \ref{lemma:discretized-ddpm-2}}
\label{proof:lemma:discretized-ddpm-2}

Recall $\tilde \sigma_{T - t}^2 = \lambda_{T - t}^{-2} \sigma_{T - t}^2$, and $\rho_{\eta_0}$ is the constant associated with $K_{\eta_0}$ from Lemma \ref{lemma:volume-lower-bound}.
By \cref{eq:sum-zeta-upper-2}, for any $z \in \R^d$ that satisfies $\dist(z, K_{\eta_0}) \leq d_0 / 3$, 
\begin{align}
\label{eq:63}
	\sum_{\eta \neq \eta_0} \zeta_{\eta, T - t} (z) \leq \frac{\abs{\cI}}{w_{\eta_0} \inf_{x \in \R^d} p_{\eta_0}(x) \cdot \rho_{\eta_0} \Xi_{T - t}^{d}} \cdot
        \exp \left( -\frac{ d_0^2}{24  \tilde{\sigma}_{T - t}^2 }\right).
\end{align}
Recall that $v_k$ is defined in \cref{eq:def-vk}. 
For all $t \in [t_{k}, t_{k + 1})$, 
\begin{align*}
	 (e^{t - t_k} - 1) \norm{v_k}^2 =  (e^{t - t_k} - 1) \cdot \Big\| \hat z_{t_k} + \frac{2\gamma}{\sigma_{T - t_k}^2} F_{\eta_0, T - t_k}(\hat z_{t_k}) - \frac{2(\gamma - 1)}{\sigma_{T - t_k}^2} \sum_{\eta \in \cI} \zeta_{\eta, T - t_k}(\hat z_{t_k}) F_{\eta, T - t_k}(\hat z_{t_k}) \Big\|^2. 
\end{align*}
By Lemma \ref{le:cm}, for all $\eta \in \cI$, 
\begin{align}
\label{eq:F-upper-dist-Proj-m2}
\begin{split}
     & \norm{F_{\eta, T - t_k}(\hat z_{t_k})} \leq \, \dist(\hat z_{t_k}, K_{\eta}) + \norm{\Proj_{K_{\eta}}(\hat z_{t_k}) - m_{\eta, T - t_k}(\hat z_{t_k})} \\
      &\leq \, \dist(\hat z_{t_k}, K_{\eta}) + C_{K_{\eta}} \tilde \sigma_{T - t_k} \log \tilde \sigma_{T - t_k}^{-1}, 
\end{split}
\end{align}
where $C_{K_{\eta}}$ is the constant from Lemma \ref{le:cm}.  
Combining \cref{eq:63,eq:F-upper-dist-Proj-m2}, we obtain that
\begin{align}
\label{eq:65-67}
\begin{split}
    (e^{t - t_k} - 1) \|v_k\|^2
    &= (e^{t - t_k} - 1) \cdot \Big\| \hat z_{t_k} + \frac{2\gamma}{\sigma_{T - t_k}^2} F_{\eta_0, T - t_k}(\hat z_{t_k}) - \frac{2\gamma - 2}{\sigma_{T - t_k}^2} \sum_{\eta \in \cI} \zeta_{\eta, T - t_k}(\hat z_{t_k}) F_{\eta, T - t_k}(\hat z_{t_k})\Big\|^2 \\
    &\leq  (e^{t - t_k} - 1) \cdot \Big( \|\hat z_{t_k}\| + \frac{2\gamma\big(\dist(\hat z_{t_k}, K_{\eta_0}) + C_{K_{\eta_0}} \tilde \sigma_{T - t_k} \log \tilde \sigma_{T - t_k}^{-1} \big) }{\sigma_{T -t_k}^2}  \\
    & \qquad +  \frac{2(\gamma - 1)\abs{\cI} e^{ -{ d_0^2} /  (24  \tilde{\sigma}_{T - t}^2 )} (\dist(\hat z_{t_k}, K_{\eta_0}) + D_0) }{\sigma_{T - t_k}^2
    {w_{\eta_0} \inf_{x \in \R^d} p_{\eta_0}(x) \cdot \rho_{\eta_0} \Xi_{T - t}^{d}}}
          \Big)^2,
\end{split}
\end{align}
where the last inequality follows from a similar calculation as in \cref{eq:sum-zeta-upper-2}.
%
Similar to the derivation of  \cref{eqn:dist_deriv_final_bound}, we conclude that 
\begin{align}
	& \Big\langle \hat z_{t_k} - \Proj_{K_{\eta_0}}(\hat z_{t_k}), \, v_k  \Big\rangle \nonumber \\
	& \leq \dist\big( \hat z_{t_k}, K_{\eta_0} \big) \Big(\dist\big( \hat z_{t_k}, K_{\eta_0} \big) + R_0 \Big) + \frac{2\dist\big( \hat z_{t_k}, K_{\eta_0} \big)D_0(\gamma - 1) \abs{\cI}}{w_{\eta_0} \inf_{x \in \R^d} p_{\eta_0}(x) \cdot \rho_{\eta_0} \Xi_{T - t_k}^{d}\sigma^2_{T - t_k}}  \cdot
        \exp \left( -\frac{3 \Xi_{T - t_k}^2}{2 \tilde{\sigma}_{T - t_k}^2 }\right) \label{eq:67.5}\\
        &\,\,\,\,\, + \frac{4\dist\big( \hat z_{t_k}, K_{\eta_0} \big)(\gamma - 1) \Xi_{T - t_k}}{\sigma_{T - t_k}^2}   - \frac{2}{\sigma_{T - t_k}^2} \dist(\hat z_{t_k}, K_{\eta_0})^2 + \frac{4\dist\big( \hat z_{t_k}, K_{\eta_0} \big)\gamma \sup_{\eta \in \cI} C_{K_{\eta}}\Xi_{T - t_k}^2 \log (\Xi_{T - t_k}^{-2})}{\sigma_{T - t_k}^2}.\nonumber
\end{align}
Therefore, there exists $\bar \tau_\eps > 0$ that depends only on $(p_0, \gamma, \eps)$, such that for all $t_k \in [\bar \tau_\eps, T)$, $t \in [t_k, t_{k + 1})$, and $\dist( \hat z_{t_k}, K_{\eta_0} ) \geq \eps $, we have
\begin{align}
\label{eq:68}
	\Big\langle \hat z_{t_k} - \Proj_{K_{\eta_0}}(\hat z_{t_k}), \, v_k  \Big\rangle \leq - \frac{1}{2\sigma_{T - t_k}^2} \dist( \hat z_{t_k}, K_{\eta_0} )^2. 
\end{align}
%
Substituting \cref{eq:65-67,eq:68} into \cref{eq:53}, we conclude that there exists $\tau_\eps > 0$ that depend only on $(p_0, \gamma, T, \eps)$, and $\kappa_0 > 0$ that depends only on $(p_0, \gamma, T)$, such that for all $t_k \in [\tau_\eps, T)$, $t \in [t_k, t_{k + 1})$, $\kappa \in (0, \kappa_0]$ and $\dist( \hat z_{t_k}, K_{\eta_0} ) \geq \eps $,
\begin{align*}
	\frac{\dd}{\dd t} \E \big[ \|\hat z_t - \Proj_{K_{\eta_0}}(\hat z_{t_k})\|^2 \mid \hat z_{t_k} \big] \leq - \frac{1}{5 \sigma_{T - t_k}^2} \dist\big( \hat z_{t_k}, K_{\eta_0} \big)^2.
\end{align*}
The proof is done.

\subsection{Proof of Lemma \ref{lemma:C3}}
\label{proof:lemma:C3} 

Taking expectation of both sides of \cref{eq:53} and differentiating, 
we see that 
\begin{align*}
	& \frac{\dd}{\dd t} \E\big[ \norm{\hat z_t - \Proj_{K_{\eta_0}}(\hat z_{t_k})}^2 \mid \hat z_{t_k}\big] \\
    &= 4 \cdot e^{t - t_k} \E \norm{\int_{t_k}^t e^{T - s} \D B_s}^2
    + 2d e^{2(T - t_k)}
    + 2e^{2(t - t_k)}(e^{t - t_k} - 1) \|v_k\|^2 + 2e^{t - t_k} \Big \langle \hat z_{t_k} - \Proj_{K_{\eta_0}}(\hat z_{t_k}), \, v_k \Big \rangle \\
    &= 4 \cdot e^{2(t - t_k)} \cdot \frac{d}{2} \left(e^{2(T - t_k)} - e^{2(T - t_k)}  \right) + 2d e^{2(T - t_k)}
    + 2e^{2(t - t_k)}(e^{t - t_k} - 1) \norm{v_k}^2 \\
    &\qquad + 2e^{t - t_k} \Big \langle \hat z_{t_k} - \Proj_{K_{\eta_0}}(\hat z_{t_k}), \, v_k \Big \rangle  \\
    &=  2d \cdot e^{2 (T - t_k) + 2(t - t_k)}  + 2e^{t - t_k}(e^{t - t_k} - 1) \|v_k\|^2 + 2e^{t - t_k} \Big \langle \hat z_{t_k} - \Proj_{K_{\eta_0}}(\hat z_{t_k}), \, v_k  \Big \rangle.
\end{align*}
By \cref{eq:65-67,eq:67.5}, we see that there exist $\tau', \kappa', C' > 0$ that depend only on $(p_0, \gamma, T)$, such that for $\kappa \in (0, \kappa')$ $t_k \in [\tau', T)$ and $t \in [t_k, t_{k + 1})$, it holds that 
\begin{align*}
	 2d e^{2T + 2t - 4t_k}  + 2e^{t - t_k}(e^{t - t_k} - 1) \norm{v_k}^2 + 2e^{t - t_k} \Big \langle \hat z_{t_k} - \Proj_{K_{\eta_0}}(\hat z_{t_k}), \, v_k  \Big \rangle  \leq  - \frac{\dist(\hat z_{t_k}, K_{\eta_0})^2}{\sigma_{T - t_k}^2} + \frac{C'}{\sigma_{T - t_k}^{3/2}}.  
\end{align*}
%
Since $\dist(\hat z_{t_{k + 1}}, K_{\eta_0}) \leq \norm{\hat z_{t_{k + 1}} -\Proj_{K_{\eta_0}}(\hat z_{t_k})}$, we obtain that  
\begin{align*}
	\E\big[ \dist(\hat z_{t_{k + 1}}, K_{\eta_0})^2 \mid \hat z_{t_k} \big] \leq \Big( 1  - \frac{t_{k + 1} - t_k}{\sigma_{T - t_k}^2}\Big)\,\dist(\hat z_{t_k}, K_{\eta_0})^2 + \frac{C'(t_{k + 1} - t_k)}{\sigma_{T - t_k}^{3/2}}. 
\end{align*}
Note that for all $k' \in \mathbb{N}^+$, 
\begin{align*}
	\sum_{k = k'}^{\infty} \frac{C'(t_{k + 1} - t_k)}{\sigma_{T - t_k}^{3/2}} < \infty. 
\end{align*}
%


Again, applying Lemma 10 in \citep[Chapter 2]{polyak2021introduction}, we see that $\dist(\hat z_{t_{k}}, K_{\eta_0})^2$ converges to a non-negative random variable $L$, almost surely.
We claim that $L \le \eps$ almost surely.
Towards a contradiction, suppose that $\P(L > \eps) > 0$.
Then, with positive probability, there exists $N \ge 1$ such that $\dist(\hat{z}_{t_k}, K_{\eta_0}) > \eps$ for all $k \ge N$. 
Write $\{Z_k\}_{k \ge 0}$ for 
\begin{equation*}
	Z_{k + 1} = 
		\dist(\hat z_{t_{k + 1}}, K_{\eta_0})^2.
\end{equation*}
By Lemmas \ref{lemma:discretized-ddpm-1} and \ref{lemma:discretized-ddpm-2}, we see that $\E[Z_{k + 1} \mid Z_k ] \leq ( 1 - \frac{t_{k + 1} - t_k}{5 \sigma_{T - t_k}^2} ) \, Z_k$. 
By the same reasoning as in the paragraph preceding \cref{lemma:discretized-ddpm-2}, we see that $Z_k \to 0$ almost surely as $k \to \infty$. 
A contradiction.
Therefore, $L \le \eps$ almost surely, completing the proof.

\subsection{Proof of Lemma \ref{lem:ito_formula_dist}}
\label{proof:lem:ito_formula_dist}

Since $\partial \tilde{K}_{\eta_0}$ is smooth, 
Lemma~\ref{lem:dist_laplacian} implies that the distance function $\dist(\cdot, \tilde{K}_{\eta_0})$ is at least $C^2$ on $\R^d \setminus \tilde{K}_{\eta_0}$.
Applying It\^o's formula, we have
   \begin{align*}
    \dd \big(\dist(z_t, \tilde{K}_{\eta_0})\big)
    &= \inprod{\nabla \dist( z_t, \tilde{K}_{\eta_0}), \,\dd  z_t} + \frac{1}{2} 
    \inprod{\nabla^2 \dist({z}_t, \tilde{K}_{\eta_0})\,  \dd  z_t  ,\, \dd  z_t} \\
    &= \left(\inprod{\nabla \dist( z_t, \tilde{K}_{\eta_0}),\, b(t,  z_t)} +  e^{2(T-t)} \Delta \dist({z}_t, \tilde{K}_{\eta_0}) \right) \D t  \\
    &\qquad + \sqrt{2} e^{T - t} \inprod{\nabla \dist( z_t, \tilde{K}_{\eta_0}), \D B_t}. 
   \end{align*}
   Note that the process ${W}_t$ defined by $\dd {W}_t= \langle {\nabla \dist( z_t, \tilde{K}_{\eta_0}), \D B_t} \rangle$ is a one-dimensional Brownian motion. 
   Indeed,  
    It\^o isometry implies the quadratic variation of this process is $[{W}]_t = t$ and $\E [{W}_t] = 0$. 
    Therefore, by L\'evy's characterization of Brownian motion \citep[Theorem 3.6 Chapter IV]{revuz2013continuous} ${W}_t$ coincides in law with a standard one-dimensional Brownian motion.
    Therefore, 
    \begin{align*}
    	\dd \big(\dist(z_t, \tilde{K}_{\eta_0})\big) = f(t, {z}_t) \D t + e^{2(T - t)}  \D {W}_t, 
    \end{align*}
    completing the proof of the lemma.

\subsection{Proof of Lemma \ref{lem:hitting_time}}
\label{proof:lem:hitting_time}

By our assumption that $\dH(\tilde K_{\eta_0}, K_{\eta_0}) \leq \varepsilon$, we have 
\begin{equation}
\label{eq:Proj-Projt}
   \norm{\Proj_{K_{\eta_0}}({z}_t) - \Proj_{\tilde{K}_{\eta_0}}({z}_t)} \le \eps.
\end{equation}
%
Since $\nabla \dist(z_t, \tilde{K}_{\eta_0}) = \frac{z_t - \Proj_{\tilde{K}_{\eta_0}}({z}_t)}{\norm{z_t - \Proj_{\tilde{K}_{\eta_0}}({z}_t)}}$, we have $\norm{\nabla \dist(z_t, \tilde{K}_{\eta_0})} = 1$. Note that for any $0 \leq t < T$, 
\begin{align*}
	&\inprod{\nabla \dist(z_t, \tilde{K}_{\eta_0}),\, b( z_t, t)} \\
	& = \inprod{\nabla \dist(z_t, \tilde{K}_{\eta_0}), \,\frac{2 \gamma}{\sigma_{T - t}^2} (m_{\eta_0, T - t}( z_t) - z_t) - \frac{ 2(\gamma - 1)}{\sigma_{T - t}^2} \sum_{\eta \in \cI} {\zeta}_{\eta, T - t} (z_t) 
  (m_{\eta, T - t}(z_t) - z_t)  } \\
  & = \underbrace{\inprod{\nabla \dist(z_t, \tilde{K}_{\eta_0}), \,\frac{2 \gamma}{\sigma_{T - t}^2} (\Proj_{\tilde{K}_{\eta_0}}({z}_t) - z_t) - \frac{ 2(\gamma - 1)}{\sigma_{T - t}^2} \sum_{\eta \in \cI} {\zeta}_{\eta, T - t} (z_t) 
  (\Proj_{{K}_{\eta}}({z}_t) - z_t)  }}_{\rm I} \\
  & + \underbrace{\inprod{\nabla \dist(z_t, \tilde{K}_{\eta_0}), \,\frac{2 \gamma}{\sigma_{T - t}^2} (m_{\eta_0, T - t}( z_t) - \Proj_{{K}_{\eta_0}}({z}_t)) - \frac{ 2(\gamma - 1)}{\sigma_{T - t}^2} \sum_{\eta \in \cI} {\zeta}_{\eta, T - t} (z_t) 
  (m_{\eta, T - t}( z_t) - \Proj_{{K}_{\eta}}({z}_t))  }}_{\rm II} \\
  & + \underbrace{\inprod{\nabla \dist(z_t, \tilde{K}_{\eta_0}), \,\frac{2 \gamma}{\sigma_{T - t}^2} (\Proj_{{K}_{\eta_0}}({z}_t) - \Proj_{\tilde{K}_{\eta_0}}({z}_t))  }}_{\rm III}
\end{align*}
To upper bound $\rm{III}$, note that\cref{eq:Proj-Projt} and the Cauchy-Schwarz inequality imply that
\begin{align}
\label{equation:III-upper}
	{|\rm{III}|} \leq \frac{2\gamma \varepsilon}{ \sigma_{T - t}^2}. 
\end{align} 
To upper bound $\rm{II}$, Lemma \ref{le:cm} implies that
\begin{align}
\label{equation:II-upper}
	{|\rm{ II}|} \leq \frac{(4 \gamma - 2) \log \sigma_{T - t}^{-1} \sup_{\eta \in \cI} C_{K_{\eta}}}{\sigma_{T - t}}, 
\end{align}
where $C_{K_{\eta}}$ is from Lemma \ref{le:cm}. 
As for term $\rm{I}$, similar to the derivation of \cref{eqn:dist_derive_II} and using the upper bound $\dH(\tilde K_{\eta_0}, K_{\eta_0}) \leq \varepsilon$, we have 
\begin{align*}
	{\rm{I}} \leq & -\frac{2\gamma - 2(\gamma - 1)\zeta_{\eta_0, T - t}(z_t)}{\sigma_{T - t}^2}\cdot   \dist(z_t, \tilde K_{\eta_0}) + \frac{2\gamma - 2}{\sigma_{T - t}^2} \cdot \sum_{\eta \neq \eta_0} \zeta_{\eta, T - t}(z_t) \cdot 
        \dist(z_t, K_{\eta}) \\
        \leq & -\frac{2\gamma - 2(\gamma - 1)\zeta_{\eta_0, T - t}(z_t)}{\sigma_{T - t}^2}\cdot   \dist(z_t, K_{\eta_0}) + \frac{2\gamma - 2}{\sigma_{T - t}^2} \cdot \sum_{\eta \neq \eta_0} \zeta_{\eta, T - t}(z_t) \cdot 
        \dist(z_t, K_{\eta}) + \frac{2\gamma  \varepsilon}{\sigma_{T - t}^2}. 
\end{align*}
%
Following the same argument as in the derivation of \cref{eqn:dist_deriv_final_bound}, we conclude that 
\begin{align}
\label{equation:I-upper}
	{\rm{I}} \leq & \frac{2D_0(\gamma - 1) \abs{\cI}}{w_{\eta_0} \inf_{x \in \R^d} p_{\eta_0}(x) \cdot \rho_{\eta_0} \Xi_{T - t}^{d}\sigma^2_{T - t}}  \cdot
        \exp \left( -\frac{3 \Xi_{T - t}^2}{2 \tilde{\sigma}_{T - t}^2 }\right) + \frac{4(\gamma - 1) \Xi_{T - t}}{\sigma_{T - t}^2}   - \frac{2}{\sigma_{T - t}^2} \dist(z_t, K_{\eta_0}) + \frac{2\gamma  \varepsilon}{\sigma_{T - t}^2} \\
        \leq & \frac{2D_0(\gamma - 1) \abs{\cI}}{w_{\eta_0} \inf_{x \in \R^d} p_{\eta_0}(x) \cdot \rho_{\eta_0} \Xi_{T - t}^{d}\sigma^2_{T - t}}  \cdot
        \exp \left( -\frac{3 \Xi_{T - t}^2}{2 \tilde{\sigma}_{T - t}^2 }\right) + \frac{4(\gamma - 1) \Xi_{T - t}}{\sigma_{T - t}^2}   - \frac{2}{\sigma_{T - t}^2} \dist(z_t, \tilde K_{\eta_0}) + \frac{(2\gamma + 2) \varepsilon}{\sigma_{T - t}^2}, \nonumber
\end{align}
where we recall $D_0 = \sup_{y, y' \in \cup_{\eta \in \cI} K_{\eta}} \norm{y - y'}$, $\Xi_{T - t} = \sqrt{\tilde{\sigma}_{T - t}}$ and $\tilde \sigma_{T - t}^2 = \lambda_{T - t}^{-2}\sigma_{T - t}^2 $. 
The above upper bounds apply to any $t \in [0, T)$. 
Combining \cref{equation:III-upper,equation:II-upper,equation:I-upper}, we see that there exists $t_\varepsilon > T - 1/2$ that depends only on $(p_0, \varepsilon, \gamma)$,
such that for any $t \in [t_{\varepsilon}, T)$ and $\dist({z}_t, \tilde{K}_{\eta_0}) > (2\gamma + 3) \varepsilon$, it holds that  
\begin{align*}
    f(t, z_t) 
    &= \inprod{\nabla \dist(z_t, \tilde{K}_{\eta_0}),\, b(t, z_t)} + e^{2(T - t)} \Delta \dist({z}_t, \tilde{K}_{\eta_0}) \\
    &\le -\frac{1}{\sigma_{T - t}^2} \Big( 2\, \dist(z_t, \tilde K_{\eta_0}) - (4 \gamma + 3) \varepsilon  \Big) + e\,\mathcal{C}_{\tilde K_{\eta_0}} \le -\frac{ \varepsilon}{\sigma_{T - t}^2},
\end{align*}
where we recall that $\mathcal{C}_{\tilde K_{\eta_0}}$ is defined in Lemma \ref{lem:dist_laplacian}. 
Consider the one-dimensional It\^o process $(X_t)_{\tau \leq t < T}$ defined by 
\begin{equation}
\label{eqn:comparison_sde}
\dd X_t = -\frac{ \varepsilon}{\sigma_{T - t}^2} \D t + e^{2(T - t)} \D W_t, \qquad X_{\tau} = \dist(z_{\tau}, \tilde{K}_{\eta_0}).
\end{equation}
By a comparison theorem for SDEs \citep[Chapter VI, Theorem 1.1]{ikeda2014stochastic}, we obtain that almost surely,
\begin{align*}
 & \dist(z_t, \tilde{K}_{\eta_0}) \le X_t, \quad \text{for all } 
    t \in [\tau,\, \tilde T), \\
 & \tilde T = \min\big\{\inf \{ s > \tau: \dist( z_s, \tilde{K}_{\eta_0}) = (2\gamma + 3) \varepsilon \}, \,T\big\}.
\end{align*}
Therefore, it suffices to show that almost surely there exists $s \in [\tau, T)$, such that $X_s = (2\gamma + 3) \varepsilon$.
This is equivalent to proving that almost surely $\tilde T < T$.
Note that 
\begin{align*}
	\P\big(\tilde T = T\big) \leq \P\Big( \inf_{\tau \leq t < T} X_t \geq (2\gamma + 3) \varepsilon \Big) \leq \inf_{s_i \to T^-} \, \P\Big( X_{s_i} \geq (2\gamma + 3) \varepsilon  \Big),  
\end{align*}
where $\{s_i \}_{i = 1}^{\infty}$ is a sequence in $[\tau, T)$ that converges to $T$ as $i \to \infty$. 
In addition, 
\begin{align*}
	\P\Big( X_{s_i} \geq (2\gamma + 3) \varepsilon  \Big) = \P\Big( \dist(z_{\tau}, \tilde{K}_{\eta_0}) - \int_{\tau}^{s_i} \frac{\varepsilon}{\sigma_{T - t}^2} + \int_{\tau}^{s_i} e^{2(T - t)} \dd W_t \geq (2\gamma + 3) \varepsilon \Big), 
\end{align*}
which converges to zero as $i \to \infty$. 
This proves that $\P(\tilde T = T) = 0$, completing the proof.

\subsection{Proof of Lemma \ref{lem:upcrossings}}
\label{proof:lem:upcrossings}

If there does not exist $t \geq t_{\varepsilon}$ such that $\dist({z}_t, \tilde{K}_{\eta_0}) = a $, then $U_{[t_{\varepsilon}, T)} = 0$. 
In what follows, we condition on the event $\{\exists\, t \geq t_{\varepsilon}: \dist({z}_t, \tilde{K}_{\eta_0}) = a\}$.

Define $\mathsf{Hit}(a) = \inf \{ t \ge t_\varepsilon: \dist({z}_t, \tilde{K}_{\eta_0}) = a \}$. 
Define a sequence of intervals $(I_m)_{m \ge 1}$ by 
\begin{align*}
        I_1 &= \left[\mathsf{Hit}(a),\, \mathsf{Hit}(a) + \frac{T - \mathsf{Hit}(a)}{2} \right], \\
        I_m &= \left[I_{m - 1}^r,\, I_{m - 1}^r + \frac{T - I_{m - 1}^r}{2} \right] \quad \text{for } m > 1,
\end{align*}
where we write $I_m = [I_m^{\ell}, I_m^r]$.
By construction, the length of $\abs{I_m}$ is $(T - \mathsf{Hit}(a))/ 2^{m}$.
For each $m \ge 1$, define the event $E_m = \left\{ U_{I_m} \ge 1 \right\}$, where $U_{I_m}$ denotes the number of upcrossings of $\dist(z_t, \tilde{K}_{\eta_0})$ from $a$ to $2a$ over the interval $I_m$.
We next show that $\P ( E_m \text{ i.o.} ) = 0$ (i.o. stands for infinitely often).
Consider the one-dimensional It\^o process $(X_t)_{\tau_m \leq t < I_m^r}$ defined by 
\begin{equation}
\label{eqn:comparison_sde2}
\dd X_t = -\frac{ \varepsilon}{\sigma_{T - t}^2} \D t + e^{2(T - t)} \D W_t, \qquad X_{\tau_m} = a.
\end{equation} 
where $\tau_m$ denotes the first time that $\dist( z_t, \tilde K_{\eta_0}) = a$ within the interval $I_m$ conditioned on the event $E_m$.
Applying the SDE comparison theorem again, we see that
\begin{align*}
        E_m \subseteq & \Big\{  \, \max_{\tau_m \le t \le I_m^r} X_t \ge 2a \Big\} \subseteq \Big\{  \, \max_{\tau_m \le t \le I_m^r} \int_{\tau_m}^t e^{2(T - s)} \D W_s \ge a \Big\} \\
        \subseteq & \Big\{  \max_{I_m^l \le t \le I_m^r} \int_{I_m^l}^t e^{2(T - s)} \D W_s  \ge a / 2\, \Big\} \cup \Big\{  \inf_{I_m^l \le t \le I_m^r} \int_{I_m^l}^t e^{2(T - s)} \D W_s  \le -a / 2\, \Big\}. 
\end{align*}
Note that the quadratic variation of the process $\int_{I_m^l}^t e^{T - s} \D W_s$ is upper bounded by
\begin{equation*}
    \int_{I_m^l}^t e^{2(T - s)} \D s \le \int_{I_m^l}^{I_m^r} e^{2(T - s)} \D s \le e\abs{I_m}.
\end{equation*}
Applying Bernstein's maximal inequality (see \citep[Section 2.3]{mckean2024stochastic} or \citep[Exercise 3.16 Chapter IV]{revuz2013continuous}), we see that
    \begin{align*}
        \P \left( E_m \right) 
        &\le \P 
        \left(\max_{\tau_m \le t \le I_m^r} X_t \ge 2a \right)
 \\
        &= \P \left( \max_{\tau_m \le t \le I_m^r} \Big[ \int_{\tau_m}^t -\frac{ \varepsilon}{\sigma_{T - t}^2} \D t  +  \int_{\tau_m}^t e^{2(T - s)} \D W_s \Big] \ge a \right) \\
        &\le \P \left( \max_{\tau_m \le t \le I_m^r}   \int_{\tau_m}^t e^{2(T - s)} \D W_s  \ge a \right) \\
        & \le \P \Big(\Big\{  \max_{I_m^l \le t \le I_m^r} \int_{I_m^l}^t e^{2(T - s)} \D W_s  \ge a / 2\, \Big\} \cup \Big\{  \inf_{I_m^l \le t \le I_m^r} \int_{I_m^l}^t e^{2(T - s)} \D W_s  \le -a / 2\, \Big\} \Big) \\
        & \le \P \Big(\Big\{  \max_{I_m^l \le t \le I_m^r} \int_{I_m^l}^t e^{2(T - s)} \D W_s  \ge a / 2\, \Big\}  \Big) + \P \Big( \Big\{  \inf_{I_m^l \le t \le I_m^r} \int_{I_m^l}^t e^{2(T - s)} \D W_s  \le -a / 2\, \Big\} \Big) \\
        & \le 2\exp \left( - \frac{ a^2}{4e\abs{I_m}} \right) = 2\exp \left( - \frac{a^2 \cdot 2^m}{4e(T - \mathsf{Hit}(a))} \right),
    \end{align*}
    where the third inequality follows from the fact that $X_t$ has a negative drift.
    Therefore, 
    \begin{equation*}
        \sum_{m \ge 1} \P \left( E_m \right)
        \le \sum_{m \ge 1} 2\exp \left( - \frac{a^2 \cdot 2^m}{4e(T - \mathsf{Hit}(a))} \right) < \infty.
    \end{equation*}
    The first Borel-Cantelli lemma implies that $\P \left( E_m \text{ i.o.} \right) = 0$, completing the proof.

\end{document}